%% file: DRO-MARL-near-optimal.tex
\documentclass[10pt,english]{article}
\usepackage{geometry}
\usepackage[T1]{fontenc}
\usepackage[latin9]{inputenc}
\usepackage{bm}
\usepackage{amsmath}
\usepackage{amssymb}
\usepackage[unicode=true,
 bookmarks=false,
 breaklinks=false,pdfborder={0 0 1},colorlinks=false]
 {hyperref}
\hypersetup{
 colorlinks,citecolor=blue,filecolor=blue,linkcolor=blue,urlcolor=blue}
\usepackage{xcolor,colortbl}
\definecolor{Gray}{gray}{0.85}
\usepackage{enumitem}
\makeatletter
\usepackage{amsthm}
\usepackage{cite}  
\usepackage{comment}
\usepackage[sort,compress]{natbib}
\usepackage{booktabs,mathtools}
\usepackage{graphicx}
\usepackage[linesnumbered,ruled,vlined]{algorithm2e}
\usepackage{algorithmic}
\usepackage{hyperref}
\usepackage{cleveref}

\usepackage{multirow}
\usepackage{dsfont}
\usepackage{color}
\usepackage{float}

\definecolor{yjc}{RGB}{225,0,100}
\definecolor{lxs}{RGB}{138,43,226}
\definecolor{jcg}{RGB}{100,160,0}
\definecolor{adam}{RGB}{0,0,200}
\definecolor{own_pink}{RGB}{217,25,169}
\definecolor{own_blue}{RGB}{0,100,223}

\definecolor{own_pink}{RGB}{217,25,169}
\definecolor{own_blue}{RGB}{0,100,223}

\allowdisplaybreaks

\input{macro.tex}

\title{Breaking the Curse of Multiagency in \\ Robust Multi-Agent Reinforcement Learning}

 \author{
	Laixi Shi\thanks{The first two authors contributed equally.} \thanks{Department of Computing Mathematical Sciences, California Institute of Technology, CA 91125, USA.}\\
    Caltech
    \and
    Jingchu Gai\footnotemark[1] \thanks{School of Mathematical Sciences, Peking University, Beijing, 100871, China.}\\
    PKU
 	\and
 	Eric Mazumdar\footnotemark[2] \\ 	 
	Caltech
	\and
 	Yuejie Chi\thanks{Department of Electrical and Computer Engineering, Carnegie Mellon University, Pittsburgh, PA 15213, USA.} \\ 
	CMU 
	\and
	Adam Wierman\footnotemark[2] \\
	Caltech
 	} 

\date{September 2024; Revised January 2025}
\begin{document}
\sloppy

\theoremstyle{plain} \newtheorem{lemma}{\textbf{Lemma}}
\newtheorem{proposition}{\textbf{Proposition}}
\newtheorem{theorem}{\textbf{Theorem}}
\newtheorem{assumption}{Assumption}
\newtheorem{definition}{Definition}
\newtheorem{corollary}{Corollary}[theorem] 

\theoremstyle{remark}\newtheorem{remark}{\textbf{Remark}}

\maketitle

\input{./sub-files-for-arxiv/abstract}


\allowdisplaybreaks

\setcounter{tocdepth}{2}
\tableofcontents

\input{./sub-files-for-arxiv/intro}

\input{./sub-files-for-arxiv/related-work}
\input{./sub-files-for-arxiv/background}

\input{./sub-files-for-arxiv/general-formulation}

\input{./sub-files-for-arxiv/algorithm}

\input{./sub-files-for-arxiv/results}

\input{sub-files-for-arxiv/conclusion}

\section*{Acknowledgements}

The work of Y. Chi is supported in part by the grants NSF CCF-2106778 and CNS-2148212, and by funds from federal agency and industry partners as specified in the Resilient \& Intelligent NextG Systems (RINGS) program. The work of L. Shi is supported in part by the Resnick Institute and Computing, Data, and Society Postdoctoral Fellowship at California Institute of Technology. The work of E. Mazumdar is supported in part from NSF-2240110. The work of A. Wierman is supported in part from
the NSF through CNS-2146814, CPS-2136197, CNS-2106403, NGSDI-2105648.

%%%%%%%%%%%%%%%%%%%%%%%%%%%%%%%%%%%%%%%%%%%%%%%%%%%%%%%%%%%%

\bibliography{bibfileRL,bibfileDRO,bibfileGame}
\bibliographystyle{apalike}

%\newpage
\appendix

\input{sub-files-for-arxiv/preliminary}
\input{sub-files-for-arxiv/appendix_game}
\input{sub-files-for-arxiv/appendix_sample_complexity}

\end{document}

%% file: macro.tex
\newcommand{\defn}{\coloneqq}

\newcommand{\no}{0} % nominal transition kernel superscript
\newcommand{\ror}{\sigma} % uncertainty set radius
\newcommand{\unb}{\cU} % uncertainty set
\newcommand{\ac}{\mathcal{A}}
\newcommand{\rew}{r}

\newcommand{\rmg}{RMG\xspace}
\newcommand{\rmgs}{RMGs\xspace}

\newcommand{\sset}{$(s,a_i)$}
\newcommand{\RQFTRL}{{\sf Robust-Q-FTRL}\xspace}
\newcommand{\ormg}{fictitious RMGs\xspace}
\newcommand{\nset}{\mathcal{U}^{\ror_i}}
\newcommand{\rmgin}{$\mathcal{RMG}_{\mathsf{in}}$\xspace}

\newcommand{\ba}{\bm{a}}

\newcommand{\cA}{\mathcal{A}}
\newcommand{\cB}{\mathcal{B}}

\newcommand{\cM}{\mathcal{M}}

\newcommand{\cP}{\mathcal{P}}

\newcommand{\cS}{{\mathcal{S}}}

\newcommand{\cU}{\mathcal{U}}

\newcommand{\cX}{\mathcal{X}}
\newcommand{\cY}{\mathcal{Y}}

\newcommand{\mymid}{\,|\,} 

\usepackage{scalerel,stackengine}
\stackMath
\newcommand\reallywidehat[1]{%
\savestack{\tmpbox}{\stretchto{%
  \scaleto{%
    \scalerel*[\widthof{\ensuremath{#1}}]{\kern-.6pt\bigwedge\kern-.6pt}%
    {\rule[-\textheight/2]{1ex}{\textheight}}%WIDTH-LIMITED BIG WEDGE
  }{\textheight}% 
}{0.5ex}}%
\stackon[1pt]{#1}{\tmpbox}%
}
\newcommand\reallywidecheck[1]{%
\savestack{\tmpbox}{\stretchto{%
  \scaleto{%zh
    \scalerel*[\widthof{\ensuremath{#1}}]{\kern-.6pt\bigwedge\kern-.6pt}%
    {\rule[-\textheight/2]{1ex}{\textheight}}%WIDTH-LIMITED BIG WEDGE
  }{\textheight}% 
}{0.5ex}}%
\stackon[1pt]{#1}{\scalebox{-1}{\tmpbox}}%
}

%% file: sub-files-for-arxiv/abstract.tex
\begin{abstract}
Standard multi-agent reinforcement learning (MARL) algorithms are vulnerable to sim-to-real gaps. To address this, distributionally robust Markov games (RMGs) have been proposed to enhance robustness in MARL by optimizing the worst-case performance when game dynamics shift within a prescribed uncertainty set. RMGs remains under-explored, from reasonable problem formulation to the development of sample-efficient algorithms. Two notorious and open challenges are the formulation of the uncertainty set and whether the corresponding RMGs can overcome the curse of multiagency,  where the sample complexity scales exponentially with the number of agents. In this work, we propose a natural class of RMGs inspired by behavioral economics, where each agent's uncertainty set is shaped by both the environment and the integrated behavior of other agents. We first establish the well-posedness of this class of RMGs by proving the existence of game-theoretic solutions such as robust Nash equilibria and coarse correlated equilibria (CCE). Assuming access to a generative model, we then introduce a sample-efficient algorithm for learning the CCE whose sample complexity scales polynomially with all relevant parameters. To the best of our knowledge, this is the first algorithm to break the curse of multiagency for RMGs, regardless of the uncertainty set formulation.

\end{abstract}

%% file: sub-files-for-arxiv/intro.tex
%!TEX root = ./../DRO-MARL-near-optimal.tex
\section{Introduction}

A flurry of problems naturally involve decision-making among multiple players, whether human, artificial intelligence, or both, with strategic objectives. Multi-agent reinforcement learning (MARL) serves as a powerful framework to address these challenges, demonstrating potential in various applications such as social dilemmas \citep{leibo2017multi,baker2020emergent,zhang2024equilibrium}, autonomous driving \citep{lillicrap2015continuous}, robotics \citep{kober2013reinforcement,rusu2017sim}, and games \citep{mnih2015human,vinyals2019grandmaster}. Despite the recent success of standard MARL, its transition from prototypes to reliable production is hindered by robustness concerns due to the complexity and variability of both the real-world environment and human behaviors. Specifically, environmental uncertainty can arise from sim-to-real gaps \citep{tobin2017domain}, unexpected disturbance \citep{pinto2017robust}, system noise, and adversarial attacks \citep{mahmood2018benchmarking}; agents' behaviors are subject to unknown bounded rationality and variability \citep{tversky1974judgment}. The solution learned at training can fail catastrophically when faced with a slightly shifted MARL problem during deployment, resulting in a significant drop in overall outcomes and each agent's individual payoff \citep{balaji2019deepracer,zhang2020robust,zeng2022resilience,yeh2021sustainbench,shi2024sample,slumbers2023game}. \looseness = -1

To address robustness challenges, a promising and flexible framework is  (distributionally) robust Markov games (RMGs) \citep{littman1994markov,shapley1953stochastic}. It is a robust counterpart to the common playground of standard MARL problems --- Markov games (MGs) \citep{zhang2020robustrl,kardecs2011discounted}. In standard MGs, agents consider (competitive) personal objectives and simultaneously interact with each other within a shared unknown environment. The goal is to learn some rationally optimal solution concepts called equilibria, which are joint strategies/policies of agents that all of them stick with rationally with other agents fixed; for instance, Nash equilibria (NE) \citep{nash1951non,shapley1953stochastic},  correlated equilibria (CE), and coarse correlated equilibria (CCE) \citep{aumann1987correlated,moulin1978strategically}. To promote robustness, RMGs differ from standard MGs by defining each agent's payoff (objective) as its worst-case performance when the dynamics of the game shift within a prescribed uncertainty set  centered around a nominal environment.

\subsection{Open questions of robust MARL}

\paragraph{Construction of realistic uncertainty sets.}

The family of RMGs is a rich class of problems because of the flexibility in constructing the uncertainty sets to capture different uncertainty considerations.  The uncertainty sets prevalent in current approaches are constructed under the {\em $(s,\ba)$-rectangularity condition}, yielding each agent's objective as the expectation over the independent risk-aware outcome on each joint action of other agents' strategies. While observations from behavioral economics \citep{friedman2022quantal,sandomirskiy2024narrow,goeree2005regular} reveal that, to handle other players' uncertainty, people often use a risk-aware metric outside of the expected outcome of other players' joint policy, rather than flipping the expectation and the risk metric as that of {\em $(s,\ba)$-rectangularity condition}. To account for realistic human decision-making, we are motivated to develop new classes of RMGs that foster robust solutions for practical MARL problems.
\paragraph{The curse of multiagency.}

Sample efficiency is a crucial challenge for solving MARL due to the limited availability of data relative to the high dimensionality of the problem. In MARL, agents strive to learn through interactions (data) with an unknown environment \citep{silver2016mastering,vinyals2019grandmaster,achiam2023gpt} that is often extremely large-scale, while data acquisition can be prohibitively limited by high costs and stakes. As such, a notable scalability challenge of sample efficiency is the {\em the curse of multiagency} --- the sample complexity requirement scales exponentially with the number of agents (induced by the exponentially growing size of the joint action space). This issue has been recognized and studied in extensive MARL problems \citep{song2021can,rubinstein2017settling}, but remains unresolved for robust MARL. We concentrate on finite-horizon multi-player general-sum Markov games, with a widely-used data collection mechanism --- generative model \citep{kearns1999finite}, where the number of agents is $n$, the episode length is $H$, the size of the state space is $S$, and the size of the $i$-th agent's action space is $A_i$, for $1\leq i\leq n$. 
\begin{itemize}
    \item {\em Breaking the curse of multiagency in standard MARL.} A line of pioneering work \citep{jin2021v,bai2020provable,song2021can,li2023minimax} has recently introduced a new suite of algorithms using adaptive sampling that provably break the curse of multiagency in standard MGs. In particular,
    to find an $\varepsilon$-approximate CCE, \citet{li2023minimax} requires a minimax-optimal sample complexity no more than
    \begin{align}
        \widetilde{O}\left( \frac{H^4S\sum_{i=1}^n A_i}{\varepsilon^2}\right)
    \end{align}
up to logarithmic factors, which depends only on the  sum of individual actions, rather than the number of joint actions. 
 
    \item {\em The persistent curse of multiagency in robust MARL.} The development of provable sample-efficient algorithms for RMGs is largely underexplored, with only a few recent studies \citep{zhang2020robustrl,kardecs2011discounted,ma2023decentralized,blanchet2023double,shi2024sample}. Focusing on a class of RMGs with uncertainty sets satisfying the {\em $(s,\ba)$-rectangularity condition}, existing works all suffer from the curse of multiagency, significantly limiting their scalability. For example, using the total variation (TV) distance as the divergence function, \citet{shi2024sample} relying on non-adaptive sampling, finds an $\varepsilon$-approximate robust CCE with a sample complexity no more than
    \begin{align} \label{eq:shi_prior}
        \widetilde{O}\left( \frac{H^3S \prod_{i=1}^n A_i  }{   \varepsilon^2} \min \Big\{H,  ~\frac{1}{\min_{1\leq i\leq n} \ror_i}\Big\} \right)
    \end{align} 
 up to logarithmic factors, where $\ror_i \in [0,1)$ is the uncertainty level for the $i$-th agent. As a result, the sample size requirement becomes prohibitive when the number of agents is large. Consequently, there is a significant desire to explore paths that could break through the curse of multiagency in RMGs, which is much more involved than its standard counterpart due to complicated non-linearity introduced by planning for worst-case performances. 
\end{itemize}

Given these two challenges of uncertainty set construction and the curse of multiagency, it raises an open question:
\begin{center}
% \vspace{-4mm}
{\em Can we design RMGs with realistic uncertainty sets that come with sample complexity guarantees breaking the curse of multiagency?}
\end{center}

\newcommand{\topsepremove}{\aboverulesep = 0mm \belowrulesep = 0mm} \topsepremove

\begin{table}[t]
	\begin{center}
\resizebox{\textwidth}{!}{
\begin{tabular}{c|c|c|c}
\toprule

	Algorithm & Uncertainty set &Equilibria & Sample  complexity   \tabularnewline
\toprule
\hline 
P$^2$MPO \vphantom{$\frac{1^{7}}{1^{7^{7}}}$} & \multirow{2}{*}{$(s,\ba)$-rectangularity}  & \multirow{2}{*}{robust NE}  & \multirow{2}{*}{$S^4 \left(\prod_{i=1}^n A_i\right)^3 H^4 /\varepsilon^2$}   \tabularnewline
\citep{blanchet2024double} & & &   \tabularnewline
\hline
DR-NVI \vphantom{$\frac{1^{7}}{1^{7^{7}}}$} & \multirow{2}{*}{$(s,\ba)$-rectangularity}  & \multirow{2}{*}{robust NE/CE/CCE} & \multirow{2}{*}{$ \frac{SH^3 \prod_{i=1}^n A_i  }{   \varepsilon^2} \min \Big\{H,  ~\frac{1}{\min_{1\leq i\leq n} \ror_i}\Big\} $
}  \tabularnewline
\citep{shi2024sample} & & &\tabularnewline \hline
\rowcolor{Gray} 
\RQFTRL \vphantom{$\frac{1^{7}}{1^{7^{7}}}$} &  fictitious   &  &  \tabularnewline
\rowcolor{Gray}
	{\bf (this work)} &  $(s,a_i)$-rectangularity &  \multirow{-2}{*}{\cellcolor{Gray} robust CCE } & \multirow{-2}{*}{$\frac{SH^6 \sum_{1\leq i\leq n} A_i  }{   \varepsilon^4} \min \Big\{H,  ~\frac{1}{\min_{1\leq i\leq n} \ror_i}\Big\}$} \tabularnewline
\hline
\toprule
\end{tabular}
}

	\end{center}

	\caption{We show all the existing sample complexity results within the general context of robust Markov games (RMGs) to put our work into perspective, on finding an $\varepsilon$-approximate equilibrium in finite-horizon multi-agent general-sum robust MG, omitting logarithmic factors.  Our result is the only algorithm that breaks the curse of multiagency regardless of the RMG formulations. 
	\label{tab:prior-work} }  
	
\end{table}

\subsection{Contributions}
Inspired by behavioral economics, we propose a new class of RMGs with a {\em fictitious} uncertainty set that explicitly models environmental uncertainties from the perspective of realistic human players, making it suitable for complex real-world scenarios. We begin by verifying the game-theoretic properties of the proposed class of RMGs to ensure the existence of robust variants of well-known standard equilibria notions, robust NE and robust CCE. Next, due to the general intractability of learning NE, we focus on designing algorithms that can provably overcome the curse of multiagency in learning an approximate robust CCE, referring to a joint policy where no agent can improve their benefit by more than  $\varepsilon$ through rational deviations.. Specifically, for sampling mechanisms to explore the unknown environment, we assume access to a generative model that can only draw samples from the nominal environment \citep{shi2024sample}.  The main contributions are summarized as follows. 
\begin{itemize}
    \item We introduce a new class of robust Markov games using fictitious uncertainty sets with {\em others-integrated \sset-rectangularity condition} (see Section~\ref{sec:robust-mgs} for details), which is not only realistic viewpoint observed from behavioral economics, but also a natural adaptation from robust single-agent RL to robust MARL. The uncertainty set for each agent $i$ can be decomposed into independent subsets over each state and its own action tuple $(s,a_i)$, where each subset is a ``ball'' around the expected nominal transition determined by other agents' policies and the nominal transition kernel, a distance function $\rho$, and the radius/uncertainty level $\sigma_i$. 
 We verify several essential facts of this class of RMGs: the existence of the desired equilibrium --- robust NE and robust CCE for this new class of RMGs using game-theoretical tools such as fixed-point theorem; the existence of best-response policies and robust Bellman equations.

    \item We consider the total variation (TV) distance as the distance metric $\rho$ for uncertainty sets due to its popularity in both theory \citep{panaganti2021sample,shi2023curious,blanchet2023double,shi2024sample} and practice \citep{pan2023adjustable,lee2021optidice,szita2003varepsilon}. Focusing on the proposed RMGs with fictitious uncertainty sets, we design \RQFTRL that can provably find $\varepsilon$-approximate robust CCE with high probability, as long as the sample size exceeds 
\begin{align}
       \widetilde{O} \left(\frac{SH^6 \sum_{i=1}^n A_i  }{   \varepsilon^4} \min \Big\{H,  ~\frac{1}{\min_{1\leq i\leq n} \ror_i}\Big\}\right)\label{eq:upper-intro}
   \end{align}
  up to logarithmic factors, where  $\ror_i \in (0,1]$ is the uncertainty level for the $i$-th agent.
To the best of our knowledge, this is the first algorithm to break the curse of multiagency in sample complexity of RMGs regardless of the uncertainty set definition, which can provably find an $\varepsilon$-approximate robust CCE using a sample size that is polynomial to all salient parameters. Table~\ref{tab:prior-work} provides a detailed summary of all existing sample complexity results in robust MARL\footnote{Note that, since we focus on a new class of RMGs, the sample complexity results in this work cannot be directly compared to those in prior studies. However, we provide a summary in Table~\ref{tab:prior-work} of existing sample complexity results for general RMGs, regardless of the uncertainty set formulation, for reference.}, where our results show significantly data efficiency with linear dependency on the size of each agent's action space, which is absent from prior works \citep{blanchet2024double,shi2024sample}. To achieve this, we utilize adaptive sampling and online adversarial learning tools, coupled by a tailored design and analysis for robust MARL due to the nonlinearity of the robust value function, which contrasts with the linear payoff functions in standard MARL with respect to the transition kernel.
\end{itemize}

\paragraph{Notation.} In this paper, we denote $[T] \defn \{1, 2, \dots, T\}$ for any positive integer $T > 0$. We define $\Delta(\cS)$ as the simplex over a set $\cS$. For any policy $\pi$ and function $Q(\cdot)$ defined over a domain $\cB$, the variance of $Q$ under $\pi$ is given by $\mathsf{Var}_{\pi}(Q) \defn \sum_{a \in \cB} \pi(a) [Q(a) - \mathbb{E}_{\pi}[Q]]^2$. We define $x = [x(s, \mathbf{a})]_{(s, \mathbf{a}) \in \mathcal{S} \times \mathcal{A}} \in \mathbb{R}^{SA}$ as any vector that represents values for each state-action pair, and $x = [x(s, a_i)]_{(s, a_i) \in \mathcal{S} \times \mathcal{A}_i} \in \mathbb{R}^{SA_i}$ as any vector representing agent-wise state-action values. Similarly, we denote $x = [x(s)]_{s \in \mathcal{S}}$ as any vector representing values for each state. For $\cX \defn (\cS, \{A_i\}_{i\in[n]}, H, \{\sigma_i\}_{i\in[n]}, \frac{1}{\varepsilon}, \frac{1}{\delta})$,  let $f(\cX) = O(g(\cX))$ denote that there exists a universal constant $C_1 > 0$ such that $f \leq C_1 g$. Furthermore, the notation $\widetilde{O}(\cdot)$ is defined similarly to ${O}(\cdot)$ but hides logarithmic factors.

%% file: sub-files-for-arxiv/related-work.tex
\subsection{Related works}

\paragraph{Breaking curse of multiagency for standard Markov games.}
Breaking the curse of multiagency is a major and prevalent challenge in sequential games.  
In standard multi-agent general-sum MGs, it has been shown that learning a Nash equilibrium requires an exponential sample complexity \citep{song2021can,rubinstein2017settling,bai2020provable}. However, for other types of equilibria, such as CE and CCE, many works have successfully broken the curse of multiagency.  Specifically, for finite-horizon general-sum MGs in the tabular setting with finite state and action spaces, \citet{jin2021v} developed the V-learning algorithm for learning CE and CCE with the sample complexity of $\widetilde{O}(H^{6}S(\max_{i\in[n]}A_i)^2/\epsilon^2)$ and $\widetilde{O}(H^{6}S\max_{i\in[n]}A_i/\epsilon^2)$, respectively; \citet{daskalakis2023complexity} achieved a sample complexity of $ \widetilde{O}(H^{11}S^{3}\max_{i\in[n]}A_i/\epsilon^3)$ for learning a CCE. Beyond tabular settings, \citet{wang2023breaking} and \citet{cui2023breaking} extended these results to linear function approximation, achieving sample complexities of $\widetilde{O}(d^4H^6\left(\max_{i\in[n]}A_i^5\right)/\epsilon^2)$ and $\widetilde{O}(H^{10}d^4\log\left(\max_{i\in[n]}A_i\right)/\epsilon^4)$, respectively, where $d$ is the dimension of the linear features. For Markov potential games, a subclass of MGs, \citet{song2021can} provided a centralized algorithm that learns a NE with a sample complexity of $\widetilde{O}(H^4S^2\max_{i\in[n]}A_i/\epsilon^3)$. 

\paragraph{Finite-sample analysis for distributionally robust Markov games.} 
Robust Markov games under environmental uncertainty are largely underexplored, with only a few provable algorithms \citep{zhang2020robust,kardecs2011discounted,ma2023decentralized,blanchet2023double,shi2024sample}. 
Existing sample complexity analyses all suffer from the daunting curse of multiagency issues, or impose an extremely restricted uncertainty level that can fail to deliver the desired robustness \citep{ma2023decentralized,blanchet2024double,shi2024sample}. Specifically, they all consider a class of RMGs with the {\em $(s,\ba)$-rectangularity condition}, where the uncertainty sets for each agent can be decomposed into independent sets over each $(s,\ba)$ pair. \citet{shi2024sample} considered the generative model with an uncertainty set measured by the TV distance, \citet{blanchet2023double} treated a different sampling mechanism with offline data for both the TV distance and KL divergence. In addition, \citet{ma2023decentralized} required the uncertainty level be much smaller than the accuracy-level and an instance-dependent parameter (i.e., $\ror_i \leq \max\{\frac{\varepsilon}{SH^2}, \frac{p_\text{min}}{H}\}$ for all $i\in[n]$). This can thus fail to maintain the desired robustness, especially when the accuracy requirement is high (i.e., $\varepsilon \rightarrow 0$) or the RMG has small minimal positive transition probabilities (i.e., $p_\text{min} \rightarrow 0$).

\paragraph{Robust MARL.}
Standard MARL algorithms may overfit the training environment and could fail dramatically due to the perturbations and variability of both agents' behaviors and the shared environment, leading to performance drop and large deviation from the equilibrium. To address this, this work considers a robust variant of MARL  adopting the distributionally robust optimization (DRO) framework that has primarily been investigated in supervised learning \citep{rahimian2019distributionally,gao2020finite,bertsimas2018data,duchi2018learning,blanchet2019quantifying} and has attracted a lot of attention in promoting robustness in single-agent RL \citep{nilim2005robust,iyengar2005robust,badrinath2021robust,zhou2021finite,shi2022distributionally,wang2024sample,shi2023curious}. Beyond the RMG framework considered in this work, recent research has advanced the robustness of MARL algorithms from various perspectives, including resilience to uncertainties or attacks on states \citep{han2022solution,zhou2023robustness}, the type of agents \citep{zhang2021robust}, other agents' policies  \citep{li2019robust,kannan2023smart}, offline data poisoning \citep{wu2024data,mcmahan2024roping}, and nonstationary environment \citep{szita2003varepsilon}. A recent review can be found in \citet{vial2022robust}. 
 

%% file: sub-files-for-arxiv/background.tex
%!TEX root = ./../DRO-MARL-near-optimal.tex

\section{Preliminaries}\label{sec:framework-robust-marl}

In this section, we begin with some background on multi-agent general-sum standard Markov games (MGs) in finite-horizon
settings, followed by a general framework of a robust variant of standard MGs ---- distributionally robust Markov games.

\subsection{Standard Markov games}\label{sec:background}
 
A finite-horizon {\em multi-agent general-sum Markov game} (MG) can be  characterized by the tuple 
$$\mathcal{MG}=\big\{ \cS,  \{\cA_i\}_{1 \le i \le n}, P,  \rew, H \big\}.$$ 
This setup features $n$ agents each striving to maximize their individual long-term cumulative rewards within a shared environment. 
At each time step, all agents observe the same state over the state space $\cS= \{1,\cdots,S\}$ within the shared environment. For each agent $i$ ($i\in[n]$), $\mathcal{A}_i=\{1,\cdots, A_i\}$ denotes its action space containing $A_i$ possible actions. The joint action space for all agents (resp.~the subset excluding the $i$-th agent) is defined as $\ac \coloneqq \mathcal{A}_1 \times \cdots \times \mathcal{A}_n$ (resp.~ $\mathcal{A}_{-i} \coloneqq \prod_{j \neq i}\mathcal{A}_j$ for any $i\in[n]$). We use the notation $\bm{a} \in \ac$ (resp. $\bm{a}_{-i} \in \cA_{-i}$) to denote a joint action profile involving all agents (resp. all except the $i$-th agent). In addition, the probability transition kernel $P = \{P_h\}_{1\leq h\leq H}$, with each $P_h: \cS\times \cA \mapsto \Delta(\cS)$, describes the dynamics of the game: $P_h(s'\mymid s,\bm{a})$ is the probability of transitioning from state $s \in \cS$ to state $s'\in \cS$ at time step $h$ when agents choose the joint action profile $\bm{a}\in \ac$.
The reward function of the game is $\rew= \{r_{i,h} \}_{1\leq i\leq n, 1\leq h\leq H}$, with each $r_{i,h}: \cS\times \ac \mapsto [0,1]$ normalized to the unit interval. For any $(i,h,s,\bm{a})\in [n] \times [H] \times \cS\times \cA$, $r_{i,h}(s,\bm{a})$ represents the immediate reward received by the $i$-th agent in state $s$ when the joint action profile $\ba$ is taken. Lastly, $H>0$ represents the horizon length.

\paragraph{Markov policies and value functions.} In this work, we concentrate on Markov policies that the action selection rule depends only on the current state $s$, independent from previous trajectory. Namely, the $i$-th ($i\in[n]$) agent chooses actions according to $ \pi_i = \{\pi_{i,h} :\cS   \mapsto \Delta(\ac_i)  \}_{1\leq h\leq H}$. Here, $\pi_{i,h}(a \mymid s)$ represents the probability of selecting action $a \in \mathcal{A}_i  $ in state $s $ at time step $h$. As such, the joint Markov policy of all agents can be denoted as $\pi= (\pi_1,\ldots, \pi_n): \cS \times [H] \mapsto \Delta(\ac)$, i.e., given any $s\in\cS$ and $h\in[H]$, the joint action profile $\bm{a}\in\mathcal{A}$ of all agents is chosen following the distribution $\pi_h(\cdot \mymid s) = (\pi_{1,h}, \pi_{2,h}\ldots, \pi_{n,h})(\cdot\mymid s) \in \Delta(\ac)$.

To continue, for any given joint policy $\pi$ and transition kernel $P$ of a $\mathcal{MG}$, the $i$-th agent's long-term cumulative reward can be characterized by the value function $V_{i,h}^{\pi, P}: \cS \mapsto \mathbb{R}$ (resp.~Q-function $Q_{i,h}^{\pi, P}: \cS \times \ac \mapsto \mathbb{R}$) as below: for all $(h,s,a)\in [H]\times \cS \times \cA$,
\begin{align}
	 V_{i,h}^{\pi, P}(s) &\coloneqq\mathbb{E}_{\pi,
P}\left[\sum_{t=h}^{H} r_{i,t}\big(s_{t}, \bm{a}_{t}\big)\mid s_{h}=s\right], \quad 
  Q_{i,h}^{\pi, P}(s, \ba)\coloneqq\mathbb{E}_{\pi,P}\left[\sum_{t=h}^{H}  r_{i,t}\big(s_{t}, \bm{a}_{t}\big)\mid s_{h}=s, \ba_h = \ba\right]. \label{eq:value-function-defn}
\end{align}
In this context, the expectation is calculated over the trajectory $\{(s_t,\bm{a}_t)\}_{h\leq t\leq H}$ produced by following the joint policy $\pi$ under the transition kernel $P$.

\subsection{Distributionally robust Markov games}\label{sec:robust-mgs}

A general distributionally robust Markov game (RMG) is represented by the tuple 
$$\mathcal{RMG}  = \big\{ \cS, \{\cA_i\}_{1 \le i \le n},\{\cU_{\rho}^{\ror_i}(P^\no, \cdot)\}_{1 \le i \le n}, \rew, H \big\}.$$ 
Here, $\cS, \{\cA_i\}_{1 \le i \le n}, r, H$ are defined in the same manner as those in standard MGs (see Section~\ref{sec:background}). RMGs differ from standard MGs: for each agent $i$ ($1\leq i\leq n$), the transition kernel is not fixed but can vary within its own prescribed uncertainty set $\cU_{\rho}^{\ror_i}(P^\no,\cdot)$ determined by (possibly the current policy and) a 
{\em nominal} kernel $P^\no: H \times \cS\times \ac \mapsto \Delta(\cS)$ that represents a reference (such as the training environment). 
The shape and the size of the uncertainty set $\big\{ \cU^{\ror_i}_\rho(P^\no,\cdot)\big\}_{i\in[n]}$ are further specified by a divergence function $\rho$ and the uncertainty levels $\{ \ror_i \}_{i\in[n]}$, serving as the ``distance'' metric and the radius respectively.

Various choices of the divergence function have been considered in the literature of robust RL, including but not limited to $f$-divergence (such as total variation, $\chi^2$ divergence, and Kullback-Leibler (KL) divergence) \citep{yang2021towards,zhou2021finite,shi2022distributionally,lu2024distributionally,wang2024sample} and Wasserstein distance \citep{xu2023improved}. Adopting uncertainty sets with different structures leads to distinct RMGs, as they address distinct types of uncertainty and game-theoretical solutions. This paper focuses on variability in environmental dynamics (transition kernels), though uncertainty in agents' reward functions could also be considered similarly but is omitted for brevity.

\paragraph{Robust value functions and best-response policies.}
For any \rmg, each agent seeks to maximize its worst-case performance in the presence of other agents' behaviors despite perturbations in the environment dynamics, as long as the kernel transitions remain within its prescribed uncertainty set. Mathematically, given any joint policy $\pi: \cS\times [H]\mapsto \Delta(\cA)$, the worst-case performance of any agent $i$ is characterized by the {\em robust value function} $V_{i,h}^{\pi,\ror_i}$ and the {\em robust Q-function} $Q_{i,h}^{\pi,\ror_i}$: for all $ (i,h,s,a_i)\in [n] \times [H] \times \cS \times \cA_i $,
\begin{align}
	 V_{i,h}^{\pi,\ror_i}(s)& \coloneqq \inf_{P\in \unb_{\rho}^{\ror_i}(P^{\no},\pi)} V_{i,h}^{\pi,P} (s)  \qquad\mbox{and}\qquad  Q_{i,h}^{\pi,\ror_i}(s, a_i) \coloneqq  \inf_{P\in \unb_{\rho}^{\ror_i}(P^{\no},\pi)} Q_{i,h}^{\pi,P} (s, a_i).
\label{eq:value-function-defn-robust}
\end{align}
Note that different from \eqref{eq:value-function-defn}, here the Q-function for any $i$-th agent is defined only over its own action $a_i \in \cA_i$ rather than the joint action $\ba\in\cA$.

To continue, we denote $\pi_{-i}$ as the policy for all agents except for the $i$-th agent. By optimizing the $i$-th agent's policy $\pi'_i: \cS \times [H] \rightarrow \Delta(\mathcal{A}_i)$ (independent from $\pi_{-i}$), we define the maximum of the robust value function as 
\begin{align}
	\label{eq:defn-optimal-V}
	   V_{i,h}^{\star,\pi_{-i}, \ror_i}(s) & \coloneqq \max_{\pi'_i: \cS \times [H] \mapsto \Delta(\mathcal{A}_i)} V_{i,h}^{\pi'_i \times \pi_{-i}, \ror_i }(s) 
	  = \max_{\pi'_i: \cS \times [H] \mapsto \Delta(\mathcal{A}_i)} \inf_{P\in \unb_{\rho}^{\ror_i}(P^{\no},\pi)} V_{i,h}^{\pi_i' \times \pi_{-i},P} (s)
\end{align}
for all $(i,h,s) \in [n] \times [H] \times  \cS$. 
The policy that  achieves the maximum of the robust value function for all $(i,h,s) \in [n] \times [H] \times  \cS$ is called a {\em robust best-response policy}.

\paragraph{Solution concepts for robust Markov games.}

In view of the conflicting objectives between agents, establishing equilibrium becomes the goal of solving \rmgs. As such, we introduce two kinds of solution concepts --- robust NE and robust CCE ---  robust variants of standard NE and CCE (usually considered in standard MGs) specified to the form of \rmgs.

\begin{itemize}
	\item {\em Robust NE.} A product policy $\pi = \pi_1 \times \pi_2 \times \cdots \times \pi_{n}: \cS\times [H] \mapsto \prod_{i=1}^n \Delta(\cA_i) $ is 
said to be a {\em robust NE}  if 
\begin{equation}
	 V_{i,1}^{\pi, \ror_i}(s)=V_{i,1}^{\star,\pi_{-i}, \ror_i}(s), \quad \forall (s,i)\in\cS\times [n].
	\label{eq:defn-robust-Nash-E}
\end{equation}
Given the strategies of the other agents $\pi_{-i}$, when each agent wants to optimize its worst-case performance when the environment and other agents' policy stay within its own uncertainty set $\unb_{\rho}^{\ror_i}(P^{\no},\pi)$, robust NE means that no player can benefit by unilaterally diverging from its present strategy.

\item {\em Robust CCE.}
A distribution over the joint product policy $\xi \defn \{\xi_h\}_{h\in[H]}: S\times [H] \mapsto \Delta(\prod_{i\in[n]} \Delta(\cA_i))$ is said to be a {\em robust CCE} if it holds that 
\begin{equation}
	\mathbb{E}_{\pi\sim \xi}\left[V_{i,1}^{\pi,\ror_i}(s)\right] \geq \mathbb{E}_{\pi\sim \xi} \left[V_{i,1}^{\star,\pi_{-i}, \ror_i }(s)\right], \quad \forall (i,s)\in[n] \times \cS.
	\label{eq:defn-robust-CCE}
\end{equation}
Considering all agents follow the policy drawn from the distribution $\xi$, i.e., $\pi_h(s) \sim \xi_h(s)$ for all $(s,h) \in \cS\times [H]$, when the distribution of all agents but the $i$-th agent's policy is fixed as the marginal distribution of $\xi$, robust CCE indicates that no agent can benefit from deviating from its current policy.

\end{itemize}

Note that, for standard MGs, CCE is defined as a possibly correlated joint policy $\pi^{\mathsf{CCE}}: \cS\times [H] \mapsto \Delta(\ac)$ \citep{moulin1978strategically,aumann1987correlated} if it holds that
\begin{equation}
    V_{i,1}^{\pi^{\mathsf{CCE}}, P}(s) \geq \max_{\pi'_i: \cS \times [H]\rightarrow \Delta(\mathcal{A}_i)} V_{i,1}^{\pi'_i \times \pi^{\mathsf{CCE}}_{-i}, P}(s), \qquad \forall (s,i)\in\cS\times [n].
	\label{eq:defn-CCE}
\end{equation}
This correlated policy $\pi^{\mathsf{CCE}}$ can also be viewed as a distribution $\xi$ over the product policy space since each joint action $\ba$ can be seen as a deterministic product policy. 
Careful readers may note that the definition \eqref{eq:defn-CCE} of CCE in standard MGs  is in a different form from the one \eqref{eq:defn-robust-CCE} in \rmgs, as the latter does not include the expectation operator $\mathbb{E}_{\pi\sim \xi}[\cdot]$ with respect to the policy distribution ($\xi$) over the value function. We emphasize that the definition with the expectation operator outside of the value (or cost) function with respect to a distribution of product pure strategies in \eqref{eq:defn-robust-CCE} is a natural formulation originating from game theory \citep{moulin2014coarse,moulin1978strategically}. In standard MARL and previous robust MARL studies, the definition in \eqref{eq:defn-CCE} is typically used because \eqref{eq:defn-CCE} and \eqref{eq:defn-robust-CCE} are identical in those situations, as the expectation operator and the corresponding value functions are linear with respect to the joint policy, allowing them to be interchanged \citep{li2023minimax,shi2024sample}.

%% file: sub-files-for-arxiv/general-formulation.tex
%!TEX root = ./../DRO-MARL-near-optimal.tex

\section{Robust Markov Games with Fictitious Uncertainty Sets}
\label{sec:rmg_definition}

Given the definition of general RMGs, a natural question arises: what kinds of uncertainty sets should we consider to achieve the desired robustness in our solutions? To address this, we focus on a class of RMGs characterized by a type of natural yet realistic uncertainty sets inspired from behavioral economics. More discussions of this class of games are provided momentarily.

\subsection{A novel uncertainty set definition in RMGs}
We propose a new class of uncertainty sets, named {\em fictitious} uncertainty sets, which count in the uncertainty induced by both the environment and other agents' behaviors in an integrated manner. Before introducing the uncertainty sets, we provide some auxiliary notations as below. We denote a vector of any transition kernel $P: \cS\times \cA \mapsto \Delta(\cS)$ or $P^{\no}:\cS\times \cA \mapsto \Delta(\cS)$ respectively as
\begin{align}\label{eq:defn-P-sa}
	\forall (s,\ba) \in \cS \times \cA: \quad  &P_{h,s,\ba} \defn P_h(\cdot \mymid s, \ba) \in \mathbb{R}^{1\times S}, \qquad P_{h,s, \ba }^\no \defn P^\no_h(\cdot \mymid s,\ba) \in \mathbb{R}^{1\times S}.
\end{align}
For any (possibly correlated) joint Markov policy (defined in \cref{sec:background})  $\pi: \cS \times [H] \mapsto \Delta(\cA)$,  we define the expected nominal transition kernel conditioned on the situation that the $i$-th agent chooses some action $a_i\in\cA_i$ and other agents play according to the conditional policy (i.e., $\ba_{-i} \sim \pi_h(\cdot \mymid s, a_i)$) given $s\in\cS$ and $a_i$ as below: 
\begin{align}
	\forall (h,s,a_i) \in [H] \times \cS\times \cA_i: \quad P^{\pi_{-i}}_{h,s,a_i} = \mathbb{E}_{\ba\sim \pi_{h}(\cdot \mymid s, a_i)} \left[ P^{0}_{h, s, \ba } \right] = \sum_{\ba_{-i} \in \cA_{-i}} \frac{\pi_h(a_i, \ba_{-i} \mymid s)}{\pi_{i,h}(a_i\mymid s)}  \left[ P^{0}_{h, s, \ba } \right]. \label{eq:uncertainty-correlated-policy-s-ai}
\end{align}

Armed with the above definitions, now we are in a position to define the {\em fictitious} uncertainty sets, which satisfy a {\em others-integrated \sset-rectangularity condition}.

\begin{definition}\label{def:uncertainty-set}
For any joint policy $\pi: \cS\times [H] \mapsto \Delta(\cA)$, divergence function $\rho: \Delta(\cS) \times \Delta(\cS) \mapsto \mathbb{R}^+$ and accessible uncertainty levels $\ror_i \geq 0$ for all $i\in[n]$, the  fictitious uncertainty sets $\big\{\nset_\rho(P^\no,\pi)  \big\}_{i\in[n]}$ satisfy the {\em others-integrated \sset-rectangularity} condition:
\begin{align}
	 \forall i\in[n]:~\cU^{\ror_i}_\rho(P^\no, \pi )  & \defn \otimes \; \cU^{\ror_i}_\rho \left(P^{\pi_{-i}}_{h,s,a_i} \right), \notag \\
  \text{s.t. }  \forall (h,s,a_i) \in [H] \times \cS \times \cA_i:  \quad 
	\cU^{\ror_i}_\rho \left(P^{\pi_{-i}}_{h,s,a_i} \right) &\defn \left\{ P\in \Delta (\cS): \rho \left(P,P^{\pi_{-i}}_{h,s,a_i} \right) \leq \ror_i \right\}, \label{eq:def-of-s-ai-set}
\end{align}
where $\otimes$ represents the Cartesian product.
\end{definition}
In words, conditioned on a fixed joint policy $\pi$, the uncertainty set $\cU^{\ror_i}_\rho(P^\no,\pi)$ for each $i$-th agent can be decomposed into a Cartesian product of subsets over each state and agent-action pair $(s,a_i)$. Each uncertainty subset $\cU^{\ror_i}_\rho(P^{\pi_{-i}}_{h,s,a_i})$ over $(s,a_i)$ is defined as a ``ball'' around a reference --- the expected nominal transition kernel $P^{\pi_{-i}}_{h,s,a_i}$ conditioned on both transition kernel and agents' joint policy $\pi$. 

% --- allowing each agent to account for a broader kind of uncertainty/disturbance
\paragraph{Further discussions of fictitious uncertainty set.} 
Here, we discuss the proposed fictitious uncertainty sets in detail, focusing on their practical implications, properties, and relation to prior work and the broader robust RL literature. Before proceeding, we have to mention that prior works on RMGs focused on a type of uncertainty sets with {\em $(s,\ba)$-rectangularity condition } \citep{ma2023decentralized,blanchet2023double,shi2024sample}. This class of uncertainty sets decouples the uncertainty into independent subsets for each state-joint action pair $(s, \ba)$, accounting for the uncertainty induced by other agents separately and independently, mathematically defined as
\begin{align*}
    \mathcal{U}_\rho^{\sigma_i}(P^0) := \otimes \mathcal{U}^{\sigma_i}(P_{h,s,\mathbf{a}}^0), \quad \text{where} \quad
    \mathcal{U}^{\sigma_i}_{\rho}(P_{h,s,\mathbf{a}}^0) = \left\{ P_{h,s,\mathbf{a}} \in \Delta(\mathcal{S}) : \rho(P_{h,s,\mathbf{a}}, P_{h,s,\mathbf{a}}^0) \leq \sigma_i \right\}.
\end{align*}
\begin{itemize}

\item 
{\em Realistic and predictive of human decisions in comparisons to prior works.} Observed from experimental data of behavioral economics, in many games considering agents' randomness \citep{friedman2022quantal,goeree2005regular,sandomirskiy2024narrow}, people address other players' uncertainty in an integrated manner as a risk metric outside of their expected outcomes (e.g., $\mathbf{Risk}(\mathbb{E}_{a_{-i}\in \pi_{-i}}[V_{i,h}^{\pi,P}(a_i, \ba_{-i})])$), but not in a separate manner as a expectation of the risk metric over outcome of each joint action (namely, $\mathbb{E}_{a_{-i}\in \pi_{-i}}[\mathbf{Risk}(V_{i,h}^{\pi.P}(a_i, \ba_{-i})]$ ). Here, the former one that is realistic corresponds to our fictitious uncertainty set, while the latter one corresponds to the uncertainty sets with {\em $(s,\ba)$-rectangularity condition } \citep{ma2023decentralized,blanchet2023double,shi2024sample} studied in prior works. 
The proposed uncertainty set models are realistic and predictive of human decision-making behaviors from behavioral economics.

\item {\em A natural adaptation from single-agent robust RL.} When agents follow some joint policy $\pi: \cS\times [H] \mapsto \Delta(\cA)$, fixing other agents' policy $\pi_{-i}$, from the perspective of each individual agent $i$,  \rmgs with our others-integrated \sset-rectangularity condition will degrade to a single-agent robust RL problem with the widely used $(s,a_i)$-rectangularity condition in the single-agent literature \citep{iyengar2005robust,zhou2021finite}. Namely, from any agent $i$'s viewpoint, in a \rmg, it deals with an "overall environment" player that can not only manipulate the environmental dynamics but also other players' policy $\pi_{-i}$.

\end{itemize}

\subsection{Properties of RMGs with fictitious uncertainty set}

Throughout the paper, we focus on the class of \rmgs with the above proposed fictitious uncertainty sets, denoted  as 
\rmgin
and abbreviated as \ormg in the remaining of the paper. In this section, we present key facts about \ormg related to best-response policies, equilibria, and the corresponding one-step lookahead robust Bellman equations. The proofs are postponed to  Appendix~\ref{proof:thm:existence-of-ne}.

First, we introduce the following lemma, which verifies the existence of a robust best-response policy that achieves the maximum robust value function (cf.~\eqref{eq:defn-optimal-V}) in any \rmgin. 

\begin{lemma}\label{lemma:existence-best-reponse}
For any $i\in[n]$, given $\pi_{-i}: \cS \times [H] \mapsto \Delta(\cA_i)$, there exists at least one policy $\widetilde{\pi}_i: \cS \times [H] \rightarrow \Delta(\mathcal{A}_i)$ for the $i$-th agent that can simultaneously attain $V_{i,h}^{\widetilde{\pi}_i \times \pi_{-i}, \ror_i}(s)  = V_{i,h}^{\star,\pi_{-i}, \ror_i}(s)$ for all $s\in \cS$ and $h\in[H]$. We refer this policy as the {\em robust best-response policy}.

\end{lemma}

\paragraph{Existence of robust NE and robust CCE.}
Fictitious RMGs can be viewed as hierarchical games with $n+nS\sum_{i=1}^n A_i$ agents. This includes the original $n$ agents and $n$ additional sets of $S\sum_{i=1}^n A_i$ independent adversaries, each determining the worst-case transitions for one agent over a state plus agent-wise-action pair. 
Considering the solution concepts --- robust NE and robust CCE --- introduced in Section~\ref{sec:robust-mgs}, the following theorem  verifies the existence of them for any \ormg using Kakutani's fixed-point theorem \citep{kakutani1941generalization}, focusing on robust NE firstly.
\begin{theorem}[Existence of robust NE]\label{thm:existence-of-ne}
 For any \rmgin $= \big\{ \cS, \{\cA_i\}_{1 \le i \le n}, \{\cU_{\rho}^{\ror_i}(P^\no,  \cdot)\}_{1 \le i \le n}, $ $\rew, H \big\}$ with an uncertainty set defined in Definition~\ref{def:uncertainty-set}, there exists at least one robust NE.
\end{theorem}
Analogous to standard Markov games, since $\{\text{robust NE}\} \subseteq \{\text{robust CCE}\}$, Theorem~\ref{thm:existence-of-ne} indicates the existence of robust CCEs directly. The class of \ormg feature a robust counterpart of the Bellman equation ---  {\em robust Bellman equation}, which is detailed in Appendix~\ref{sec:robust-bellman-equation}.

%% file: sub-files-for-arxiv/algorithm.tex
%!TEX root = ./../DRO-MARL-near-optimal.tex

\section{Sample-Efficient Learning: Algorithm and Theory}

In this section, we focus on designing sample-efficient algorithms for solving \ormg when agents need to collect data by interacting with the unknown shared environment in order to learn the equilibria. To proceed, we shall first specify   the data collection mechanism and the divergence function for the uncertainty set. Then we propose a sample-efficient algorithm \RQFTRL that leverages tailored adaptive sampling strategy to break the curse of multiagency for solving \ormg.

\subsection{Problem setting and goal}
Recall that the uncertainty sets are constructed by specifying a divergence function $\rho$ and the uncertainty level to control its shape and size.
In this work, we focus on using the TV distance as the divergence function $\rho$ for the uncertainty set, following \citet{szita2003varepsilon,lee2021optidice,pan2023adjustable,shi2023curious,shi2024sample}, defined by 
\begin{align}\label{eq:general-infinite-P}
	\forall P, P'\in \Delta(\cS): \quad \rho_{\mathsf{TV}}\left(P, P'\right) \defn  \frac{1}{2}\left\|P-P'\right\|_1.
\end{align}
For convenience, throughout the paper, we abbreviate $\cU^{\ror_i}(\cdot) \defn \cU^{\ror_i}_{\rho_{\mathsf{TV}}}(\cdot)$ when there is no ambiguity.

\paragraph{Data collection mechanism: a generative model.}
We assume the agents interact with the environment through a generative model (simulator) \citep{kearns1999finite}, which is a widely used sampling mechanism in both single-agent RL and MARL \citep{zhang2020model-based,li2022minimax}. Specifically, at any time step $h$, we can collect an arbitrary number of independent samples  from any state and joint action tuple $(s,\ba)\in \cS\times \cA$, generated based on the true {\em nominal} transition kernel $P^{\no}$:
\begin{align}
	 s_{h,s, \ba}^i \overset{i.i.d}{\sim} P^\no_h(\cdot \mymid s, \ba), \qquad i = 1, 2,\ldots
\end{align}

\paragraph{Goal.}
Consider any \ormg \rmgin $ = \big\{ \cS, \{\cA_i\}_{1 \le i \le n},\{\cU^{\ror_i}(P^\no,\cdot)\}_{1 \le i \le n}, \rew, H \big\}$. In practice, learning exact robust equilibria is computationally challenging and may not be necessary, instead in this work, we focus on finding an approximate robust CCE (defined in \eqref{eq:defn-robust-CCE}). Namely, a distribution $\xi \defn \{\xi_h\}_{h\in[H]}: [H] \times \cS\mapsto \Delta(\prod_{i\in[n]} \Delta(\cA_i))$  is said to be an {\em $\varepsilon$-robust CCE}  if 
\begin{equation}
	\mathsf{gap}_{\mathsf{CCE}}(\xi) \defn \max_{s\in \cS, 1 \leq i \leq n} \left\{ \mathbb{E}_{\pi\sim \xi} \left[V_{i,1}^{\star,\pi_{-i}, \ror_i }(s)\right] - \mathbb{E}_{\pi\sim \xi}\left[V_{i,1}^{\pi,\ror_i}(s)\right] \right\} \leq \varepsilon.
	\label{eq:defn-CCE-epsilon}
\end{equation}
Armed with a generative model of the nominal environment, the goal becomes learning a robust CCE using as few samples from the simulator as possible.

\subsection{Algorithm design}

With the sampling mechanism over a generative model in hand, we propose an algorithm called \RQFTRL to learn an {\em $\varepsilon$-robust CCE} in a sample-efficient manner. \RQFTRL draws inspiration from Q-FTRL developed in the standard MG literature \citep{li2022minimax}, but empowers tailored designs for learning in \ormg to achieve a robust equilibrium and to tackle statistical challenges arising from agents' nonlinear worst-case objectives.
Overall, \RQFTRL takes a single pass to learn recursively from the final time step $h=H$ to $h=1$. At each time step $h\in[H]$, an online learning process with  $K$ iterations will be executed. Before introducing the algorithm, we first concentrate on  two essential steps customized for learning in \ormg.

\begin{algorithm}[t]

	\textbf{Initialization:} the reward $\widehat{r}=0 \in \mathbb{R}^{SA_i}$ and the transition model $\widehat{P}=0 \in \mathbb{R}^{SA_i\times S}$. 
	
	\For{$(s,a_i)\in \cS\times \cA_i$}{

		\For{$ t = 1$ \KwTo $N$}{
           
		Sample $\bm{a}^t(s,a_i)= [a_{j}(s,a_i) ]_{1\leq j\leq n}$ constructed by independent actions drawn from policy:%
		\begin{equation}
			a_{j}(s,a_i) \overset{\text{ind.}}{\sim} \pi_{j, h}(\cdot \mymid s) ~~~~ (j\neq i)
			\qquad \text{and} \qquad
			a_{i}(s,a_i) = a_i.
		\end{equation}
		\\
		
		Sample from the generative model:
		\begin{equation}
			r_{i,h}^t(s,a_i) = r_{i,h}(s, \bm{a}^t(s,a_i)), \qquad s^t_{s,a_i} \sim P_h\big(\cdot \mymid s, \bm{a}^t(s,a_i)\big).
		\end{equation}
}
		Set $\widehat{r}(s,a_i)=\frac{1}{N}\sum_{t\in[N]} r_{i,h}^t(s,a_i)$ and $\widehat{P}\big(s'\mymid s,a_i \big)=\frac{1}{N}\sum_{t\in[N]} \mathds{1} \big\{ s^t_{s,a_i} = s' \big\}$. 
	
	}

	\textbf{Return:} empirical model $\big( \widehat{r}, \widehat{P} \big)$.

	\caption{\tt{$N$-sample estimation}$\big(\pi_h = \{\pi_{j, h}\}_{j\in[n]}, i,h\big)$.
    \label{alg:sampling-function}}
\end{algorithm}

\paragraph{Constructing the empirical model via $N$-sample estimation.}
For each time step $h$, we denote $\pi_{i,h}^k$ as the current learning policy of the $i$-th agent before the beginning of the $k$-th iteration for any $k\in[K]$. And we denote the joint product policy as $\pi_h^k = (\pi_{1,h}^k ,\cdots, \pi_{n,h}^k)$. During each iteration $k$, for each agent $i\in[n]$, we require to generate $N$ independent samples from the generative model over each $(s,a_i) \in \cS \times \cA_i$ to obtain an empirical model, detailed in Algorithm~\ref{alg:sampling-function}. It includes an empirical reward function represented by $r_{i,h}^k \in \mathbb{R}^{SA_i}$ and transition kernels denoted by $P_{i,h}^k  \in \mathbb{R}^{SA_i \times S}$. Note that different from standard MGs, we need to generate $N$ samples instead of $1$ sample per iteration to handle the additional statistical challenges induced by the non-linear objective of agents ($N$ will be specified momentarily).

\paragraph{Estimating robust Q-function of the current policy $\pi_h^k$.}  
We denote $\widehat{V}_{i,h} \in \mathbb{R}^{S}$ as the estimation of the $i$-th agent's robust value function at time step $h$.
For any agent $i$, with the empirical reward function $r_{i,h}^k$, empirical kernel $P_{i,h}^k$, and the estimated robust value function $\widehat{V}_{i,h+1}$ at the next step in hand, the robust Q-function $\{q^k_{i,h}\}$ of current policy $\pi_h^k$ can be estimated as:
\begin{align}
  \forall (i,h,s,a_i)\in  [n] \times [H] \times \cS\times A_i: \quad q^k_{i,h}(s, a_i) = r_{i,h}^k(s, a_i) + \inf_{ \cP \in \unb^{\ror_i}(P_{i,h,s,a_i}^k)} \cP \widehat{V}_{i,h+1}. \label{eq:nvi-iteration}
\end{align}
Unlike the linear function w.r.t. $P_{i,h}^k$ in standard MGs, \eqref{eq:nvi-iteration} lacks a closed form and introduces an additional inner optimization problem. Solving \eqref{eq:nvi-iteration} directly is computationally challenging due to the need to optimize over an $S$-dimensional probability simplex, with complexity growing exponentially with the state space size $S$. Fortunately, by applying strong duality, we can solve \eqref{eq:nvi-iteration} equivalently via its dual problem with tractable computation \citep{iyengar2005robust}:
\begin{align}
&q^k_{i,h}(s, a_i) = r_{i,h}^k(s, a_i)  +   \max_{\alpha\in [\min_s \widehat{V}_{i,h+1}(s), \max_s \widehat{V}_{i,h+1}(s)]}  \Big\{ P_{i,h}^k \left[\widehat{V}_{i,h+1}\right]_{\alpha} - \ror_i \left(\alpha - \min_{s'}\left[\widehat{V}_{i,h+1}\right]_{\alpha}(s') \right) \Big\}  ,  \label{eq:nvi-iteration-dual-alg} 
\end{align} 
where $[V]_{\alpha}$ denotes the clipped version of any  vector $V\in\mathbb{R}^S$ determined by some level $\alpha\geq 0$, namely,
\begin{align}
	[V]_{\alpha}(s) \defn \begin{cases} \alpha, & \text{if } V(s) > \alpha, \\
V(s), & \text{otherwise.}
\end{cases} \label{eq:V-alpha-defn}
\end{align}

The above two modules are key components of \RQFTRL, serving for constructing nonlinear robust objectives in the online learning process and ensuring the desired statistical accuracy.

\paragraph{Overall pipeline of \RQFTRL.}  
With these modules in place, we introduce \RQFTRL, which follows a similar online learning procedure as Q-FTRL for standard MGs \citep{li2022minimax}. The complete procedure is summarized in Algorithm~\ref{alg:summary}. We denote $Q_{i,h}^k \in \mathbb{R}^{SA_i}$ as the estimated robust Q-function of the equilibrium for the $i$-th agent at the $k$-th iteration of time step $h$. To begin with, \RQFTRL initialize the robust value function, robust Q-function $\widehat{V}_{i,H+1}(s)=Q_{i, h}^0(s,a_i)=0$, and the policy $\pi_{i,h}^1(a_i\mymid s)=1/A_i$ for all $(i,s) \in [n] \times \cS$.
Then subsequently from the final time step $h=H$ to $h=1$, for each step $h$, a $K$ iterations online learning process will be executed. At each $k$-th iteration, given current policy $\pi_h^k$, as described above, an empirical model ($\{r_{i,h}^k\}_{i\in [n]}$ and $\{ P_{i,h}^k\}_{i \in [n]}$) is constructed by $N$-sample estimation (cf.~\cref{alg:sampling-function}). Then the robust Q-function $\{q_{i,h}^k\}_{i\in[n]}$ of the current policy $\pi_h^k$ is estimated by \eqref{eq:nvi-iteration-dual-alg}. 

Now we are ready to specify the loss objective and proceed the online learning procedure. 
With the current one-step update $ \{q_{i,h}^k \}$, we update the Q-estimate as $Q_{i, h}^k = (1-\alpha_k)Q_{i, h}^{k-1} + \alpha_k q_{i, h}^k$. Here, $\{\alpha_k\}_{k\in[K]}$ is a series of rescaled linear learning rates with some $c_\alpha\geq 24$, 
\begin{align}
   \forall  k\in[K]: \quad \alpha_k=\frac{c_\alpha\log K}{k-1+c_\alpha\log K} \quad  \text{ and } \quad   \alpha_k^n = \begin{cases} \alpha_k \prod_{i=k+1}^n (1-\alpha_i), & \text{if } 0 <k <n \leq K\\
   \alpha_n & \text{if } k = n \end{cases}. \label{eq:online-learning-rate}
\end{align}
Let the Q-estimate be the online learning loss objective at this moment, we apply the Follow-the-Regularized-Leader strategy \citep{shalev2012online,li2022minimax} to update the corresponding policy as below:
\begin{align}
			\pi_{i, h}^{k+1}(a_i\mymid s) = \frac{\exp\big(\eta_{k+1} Q_{i, h}^k(s,a_i)\big)}{\sum_{a'} \exp\big(\eta_{k+1} Q_{i, h}^k(s,a')\big)} \quad \text{ with } \quad \eta_{k+1}=\sqrt{\frac{\log K}{\alpha_k H}},\qquad k=1,2,\ldots
			\label{eq:policy-update-exponential-alg}
   \end{align}
This is a widely used adaptive sampling and learning procedure for MARL problems.

After completing $K$ iterations for time step $h$, we finalize the robust value function estimation by setting it to its confidence upper bound, incorporating carefully designed optimistic bonus terms $\{\beta_{i,h}\}$ as: for all $(i,h,s)\in [n] \times [H] \times \cS$,
\begin{align}
     \beta_{i,h}(s) = c_{\mathsf{b}}\sqrt{\frac{\log^3(\frac{KS\sum_{i=1}^nA_i}{\delta})}{KH}}\sum_{k=1}^K\alpha_k^K\left\{\mathsf{Var}_{\pi_{i,h}^k(\cdot\mid s)}\left(q_{i,h}^k(s,\cdot)\right)+H\right\}, \label{eq:bonus-beta}
\end{align}
where $c_{\mathsf{b}}$ denotes some absolute constant, and $\delta \in(0,1)$ is the high probability threshold.
Finally, after the recursive learning process ends for all time steps $h= H, H-1,\cdots, 1$, we output a distribution of product policy $\widehat{\xi} = \{\widehat{\xi}_h\}_{h\in[H]}$ over all the policies $\{ \pi_h^k = (\pi_{1,h}^k \times \cdots \times \pi_{n,h}^k)\}_{h\in[H],k\in[K]}$ occurs during the process that defined as
\begin{align}
\forall (h,k)\in[H] \times [K]: \quad \xi_h(\pi_h^k) \defn \alpha_k.
\end{align}

\begin{algorithm}[t]
	{\em Input:} learning rates $\{\alpha_k\}$  and $\{\eta_{k+1}\}$, number of iterations $K$ per time step, and number of samples $N$ per iteration.\\
	{\em Initialization:} $\widehat{V}_{i,H+1}(s)=Q_{i, h}^0(s,a_i)=0$ and $\pi_{i,h}^1(a_i\mymid s)=1/A_i$ for all $i\in [n]$ and then all $(h,s,a_i)\in [H] \times \cS\times \cA_i$.\\
    \tcp{start recursive learning process.}
	\For{$ h = H, H-1,\cdots, 1$}{
	\For{$ k = 1,2,\cdots, K$}{
	\For{$ i = 1,2,\cdots, n$}{
		\tcp{construct empirical models and estimate current robust Q-function}
		$\big(r_{i,h}^k,P_{i,h}^k\big)$ $\leftarrow$ 
		$N$-sample estimation $\big(\pi_h^k = \{\pi_{j, h}^k\}_{j\in[n]}, i, h \big)$. (Algorithm~\ref{alg:sampling-function})

		Estimate the robust Q-function $q_{i, h}^k$  of current $\pi_h^k$ according to \eqref{eq:nvi-iteration-dual-alg}.\\
	
		\tcp{Online learning procedure}
	Update the Q-estimate $Q_{i, h}^k = (1-\alpha_k)Q_{i, h}^{k-1} + \alpha_k q_{i, h}^k$ and apply FTRL:
	$\forall (s,a_i)\in \cS\times \cA_i: \quad \pi_{i, h}^{k+1}(a_i\mymid s) = \frac{\exp\big(\eta_{k+1} Q_{i, h}^k(s,a_i)\big)}{\sum_{a'} \exp\big(\eta_{k+1} Q_{i, h}^k(s,a')\big)}.$
			
	}
	}
	\tcp{set the final robust value estimate at time step $h$.}

	\For{$ i = 1,2,\cdots, n $}{
	\label{line:policy-V-output} 

	For all $s\in\cS$: set $\beta_{i,h}(s)$ to be the optimistic bonus term in \eqref{eq:bonus-beta} and 
	\begin{align}
		\widehat{V}_{i,h}(s)&=\min\Big\{ \sum_{k=1}^{K}\alpha_{k}^{K}\big\langle\pi_{i,h}^{k}(\cdot\mymid s),\,q_{i,h}^{k}(s,\cdot)\big\rangle+\beta_{i,h}(s),~H-h+1 \Big\}. \label{eq:line-number-policy-update}
	\end{align}

	}
	}
	{\em Output:} a set of policies $\{ \pi_h^k = (\pi_{1,h}^k \times \cdots \times \pi_{n,h}^k) \}_{k\in[K], h\in[H]}$ and a distribution  $\widehat{\xi} = \{\widehat{\xi}_h\}_{h\in[H]}$ over them. For any time step $h$, $\widehat{\xi}_h$ is the distribution over $\{\pi_h^k\}_{k\in[K]}$ so that $\widehat{\xi}_h(\pi_h^k) = \alpha_k^K$.
	
	\caption{ \RQFTRL }
\label{alg:summary}
\end{algorithm}

%% file: sub-files-for-arxiv/results.tex
%!TEX root = ./../DRO-MARL-near-optimal.tex

\subsection{Theoretical guarantees}

In this section, we provide the theoretical guarantees for the sample complexity of our proposed algorithm \RQFTRL, shown as below:
\begin{theorem}[Upper bound]\label{thm:robust-mg-upper-bound} 
	Using the TV uncertainty set defined in \eqref{eq:general-infinite-P}. Consider any $\delta \in (0,1)$ and any \ormg  \rmgin $= \big\{ \cS, \{\cA_i\}_{1 \le i \le n},\{\cU^{\ror_i}(P^\no,\cdot)\}_{1 \le i \le n}, \rew,  H \big\}$ with $\ror_i \in (0,1]$ for all $i\in[n]$. For any $\varepsilon \leq \sqrt{\min \big\{H,  \frac{1}{\min_{1\leq i\leq n} \ror_i} \big\}}$, Algorithm~\ref{alg:summary} can output an $\varepsilon$-robust CCE $\widehat{\xi}$, i.e.,
     \begin{align*}
\mathsf{gap}_{\mathsf{CCE}}(\widehat{\xi}) \defn \max_{s\in \cS, 1 \leq i \leq n} \left\{ \mathbb{E}_{\pi\sim \widehat{\xi}} \left[V_{i,1}^{\star,\pi_{-i}, \ror_i }(s)\right] - \mathbb{E}_{\pi\sim \widehat{\xi}}\left[V_{i,1}^{\pi,\ror_i}(s)\right] \right\} \leq \varepsilon
	\end{align*}
with probability at least $1-\delta$, as long as 
\begin{align}
N \geq \frac{C_1 H^2}{\epsilon^2} \min\left\{ \frac{1}{\min_{1\leq i\leq n}\sigma_i}, H \right\} , \quad K \geq \frac{C_1 H^3}{\epsilon^2}.
\end{align}
Here $C_1$ is some universal large enough constant.
Namely, it is sufficient if the total number of samples acquired in the learning process obeys 
\begin{align*}
	N_{\mathsf{all}} \defn HKNS\sum_{1\leq i\leq n} A_i  \geq  \frac{ (C_1)^2 H^6  S   \sum_{1\leq i\leq n} A_i   }{  \varepsilon^4}\min \Big\{H,  \frac{1}{\min_{1\leq i\leq n} \ror_i} \Big\}.
\end{align*}
\end{theorem}

Before we jump into more discussions of the above theorem, in addition, we introduce the information-theoretic minimax lower bound for this problem as well.
\paragraph{Lower bound for learning in \ormg.}
Considering the instances of \ormg that the action space for all the agents except the $i$-th agent contains only a single action, i.e., $A_j=1$ for all $j \neq i$. As such, all the agents $j \neq i$ will take a fixed action and the game reduces to a single-agent robust MDP with {\em $(s,a)$-rectangularity condition} \citep{zhou2021finite}. So the goal of finding the robust equilibrium --- robust NE/CCE also degrades to finding the optimal policy of the $i$-th agent. Invoking the results from \citet[Theorem~2]{shi2024sample}, the lower bound for the class of \ormg is achieved directly:
consider any tuple $\big\{ S, \{A_i\}_{1 \le i \le n},\{\sigma_i\}_{1 \le i \le n}, H \big\}$ obeying $\ror_i \in (0, 1 - c_1]$ with $0 <c_1 \leq \frac{1}{4}$ being any small enough positive constant, and $H > 16 \log2$. Let
\begin{align}
 \varepsilon \leq \begin{cases} \frac{c_1}{H}, &\text{if } \sigma_i\leq \frac{c_1}{2H}, \\
    1 & \text{otherwise}
    \end{cases}
\end{align}
 We can construct a set of \ormg  $\mathcal{M} = \{\mathcal{RMG}_{\mathsf{in}}^i\}_{i\in [I]}$, such that for any dataset generated from the nominal environment with in total $N_{\mathsf{all}}$ independent samples over all state-action pairs, we have
	\begin{align}
	\inf_{\widehat{\xi} \in [H] \times \cS \mapsto \Delta( \prod_{i=1}^n \cA_i)} \;\max_{\mathcal{RMG}_{\mathsf{in}}^i\in \cM} \left\{ \mathbb{P}_{\mathcal{RMG}_{\mathsf{in}}^i}\big(  \mathsf{gap}_{\mathsf{CCE}}(\widehat{\xi}) >\varepsilon\big) \right\} &\geq\frac{1}{8},
\end{align}
provided that
\begin{align}\label{eq:lower-bound}
N_{\mathsf{all}} \leq \frac{C_2  S H^3  \max_{1\leq i\leq n} A_i    }{ \varepsilon^2} \min \Big\{H,  \frac{1}{\min_{1 \leq i \leq n} \sigma_i  } \Big\}.
\end{align}
Here, the infimum is taken over all estimators $\widehat{\xi}$, $\mathbb{P}_{\mathcal{RMG}_{\mathsf{in}}^i}$ denotes the probability when the game is $\mathcal{RMG}_{\mathsf{in}}^i$ for all $\mathcal{RMG}_{\mathsf{in}}^i \in \mathcal{M}$, and $C_2$ is some small enough constant.

Armed with both the upper bound (Theorem~\ref{thm:robust-mg-upper-bound}) and lower bound in \eqref{eq:lower-bound}, we are now ready to discuss the implications of our sample complexity results.

\paragraph{Breaking the curse of multiagency in the sample complexity for RMGs.} Theorem \ref{thm:robust-mg-upper-bound} demonstrates that for any \ormg, \RQFTRL  algorithm finds an $\epsilon$-robust CCE when the total number of samples exceeds
\begin{align}
    \widetilde{O}\left(\frac{S H^6 \sum_{1\leq i\leq n}A_i }{\epsilon^4} \min\left\{ H, \frac{1}{\min_{1\leq i\leq n}\sigma_i} \right\}\right). \label{eq:sample_complexity}
\end{align}
To the best of our knowledge, \RQFTRL with the above sample complexity in \eqref{eq:sample_complexity} is the first algorithm for \rmgs breaking the curse of multiagency, regardless of the types of uncertainty sets. Our sample complexity depends linearly on the sum of each agent's actions $\sum_{i=1}^n A_i$ rather than their product $\prod_{i=1}^n A_i$---making the algorithm highly scalable as the number of agents increases. Nonetheless, there still exist gaps between our upper bound and the lower bound---especially in terms o the dependency on  the horizon length $H$ and the accuracy level $\varepsilon$---an interesting direction to investigate in the future.

\paragraph{Comparisons with prior works.}

All prior works focus on learning equilibria for a different kind of robust MGs with  $(s,\ba)$-rectangular uncertainty sets \citep{ma2023decentralized,blanchet2023double,shi2024sample}. However, the state-of-the-art sample complexity $ \widetilde{O} \left(\frac{SH^3 \prod_{i=1}^n A_i  }{   \varepsilon^2} \min \Big\{H,  ~\frac{1}{\min_{1\leq i\leq n} \ror_i}\Big\}\right)$ \citep{shi2024sample} still suffers from the curse of multiagency with an exponential dependency on the number of agents when all agents have equal action spaces, which uses nonadaptive sampling. Our work circumvents the curse of multiagency by the introduction of a new class of \ormg inspired from behavioral economics, together with resorting to a tailored adaptive sampling and online learning procedure, providing a fresh perspective to learning practical-meaningful RMGs.

\paragraph{Technical insights.}

For sample complexity analysis, while previous works have addressed the curse of multiagency in sequential games like standard Markov games (MGs) and Markov potential games, these methods are not directly applicable to \rmgs. Prior approaches assume a linear relationship between the value function and the transition kernel, allowing statistical errors across $K$ iterations to cancel out. However, in \rmgs, the robust value function, due to its distributionally robust requirement, is highly nonlinear and often lacks a closed form, making it impossible to linearly aggregate statistical errors. To tackle the nonlinear challenges in RMGs, we design a variance-style bonus term through non-trivial decomposition and control of auxiliary statistical errors caused by nonlinearity, resulting in a tight upper bound on regret during the online learning process.

%% file: sub-files-for-arxiv/conclusion.tex
%!TEX root = ./../DRO-MARL-near-optimal.tex
\section{Conclusion}

Robustness in MARL presents greater challenges than in single-agent RL due to the strategic interactions between agents in a game-theoretic setting.  This work proposes a new class of RMGs with fictitious uncertainty sets that naturally extends from robust single-agent RL and  addresses more realistic problems considering human features where each agent considers the uncertainty of others in an integrated manner. We then propose \RQFTRL, the first algorithm to break the curse of multiagency in robust Markov games regardless of the uncertainty set definitions, with sample complexity scaling polynomially with all key parameters. This opens up new research directions in MARL, such as uncertainty set selection and construction, equilibrium refinement, and sample-efficient algorithm design. \looseness = -1

%% file: sub-files-for-arxiv/preliminary.tex
%!TEX root = ./../DRO-MARL-near-optimal.tex
\section{Preliminaries}

Denoting the vectors $ x = [x_i]_{1\leq i\leq n} $ and $ y = [y_i]_{1\leq i\leq n} $, we use the notation $ {x}\leq {y} $ (or $ {x}\geq {y} $) to represent that $ x_i \leq y_i $ (or $ x_i \geq y_i $) for every $ 1 \leq i \leq n $.
The Hadamard product of two vectors $ x $ and $ y $ in $ \mathbb{R}^S $ is denoted as $ x \circ y = \big[x(s) \cdot y(s)\big]_{s \in \mathcal{S}} $.
In addition, for any series of vectors $\{x_i\}_{i\in[S}$, $\mathrm{diag}(x_1, x_2,\cdots, x_{S})$ denote a block diagonal matrix by placing each given vector $x_i$  along the diagonal, with zeros filling the off-diagonal blocks. ${0} $ (or $ {1} $) represents the all-zero (or all-one) vector, while $ e_i \in \mathbb{R}^S $ denotes a basis vector of dimension $ S $ with 1 in the $ i $-th position and 0 elsewhere.

\subsection{Additional notations and facts}
\subsubsection{Matrix and vector notations}
Before continuing, we introduce or recall some matrix and vector notations that will be used throughout the paper. In particular, for any joint policy $\pi: \cS\times [H] \mapsto \Delta(\cA)$ and any $(i,h)\in [n] \times [H]$:

\paragraph{Matrices about policies.}
\label{sec:policy_matrix}
We introduce three matrices associated with any joint policy $\pi$, which are defined as block diagonal matrices that adhere to the following properties:
\begin{itemize}[topsep=0pt,leftmargin=17pt]
    \setlength{\itemsep}{0pt}
    \item The matrix \( \Pi_h^\pi \in \mathbb{R}^{S \times S \prod_{i=1}^n A_i} \) is given by \( \mathrm{diag}\left(\pi_h(1)^\top, \pi_h(2)^\top, \ldots, \pi_h^\top(S)\right) \), where \( \pi_h(s) = \left[\pi_h(\ba \mymid s)\right]_{\ba \in \cA} \in \Delta(\cA) \) for each \( s \in \mathcal{S} \) represents the joint policy vectors across all agents.
    
    \item The matrix \( \Pi_h^{\pi_{-i}} \in \mathbb{R}^{S \times S \prod_{j \neq i} A_j} \) can be expressed as \( \mathrm{diag}\left(\pi_{-i,h}(1)^\top, \pi_{-i,h}(2)^\top, \ldots, \pi_{-i,h}^\top(S)\right) \), where \( \pi_{-i,h}(s) = \left[\pi_h(\ba_{-i} \mymid s)\right]_{\ba_{-i} \in \cA_{-i}} \in \Delta(\cA_{-i}) \) for all \( s \in \mathcal{S} \) denotes the joint policy vectors from all agents except agent \( i \).
    
    \item The matrix \( \Pi_h^{\pi_i} \in \mathbb{R}^{S \times S A_i} \) is defined as \( \mathrm{diag}\left(\pi_{i,h}(1)^\top, \pi_{i,h}(2)^\top, \ldots, \pi_{i,h}^\top(S)\right) \), where \( \pi_{i,h}(s) = \left[\pi_{i,h}(a_i \mymid s)\right]_{a_i \in \cA_i} \in \Delta(\cA_i) \) for each \( s \in \mathcal{S} \) represents the policy of the $i$-th agent.
\end{itemize}

\paragraph{Vectors about reward.}  
We recall the definition of \( r_{i,h} \) and introduce the reward vectors \( r_{i,h}^{\pi} \) and \( r_{i,h}^{\pi_{-i}} \) as follows:
\begin{itemize}[topsep=0pt,leftmargin=17pt]
    \setlength{\itemsep}{0pt}
    \item Let \( r_{i,h} = [r_{i,h}(s, \ba)]_{(s, \ba) \in \mathcal{S} \times \cA} \in \mathbb{R}^{S \prod_{i=1}^n A_i} \) represent the reward function for the \( i \)-th player at time step \( h \), where \( \mathcal{S} \) is the state space and \( \cA \) is the action space.
    
    \item \( r_{i,h}^{\pi} \in \mathbb{R}^{S} \): the expected reward of the $i$-th player, associated with any policy $\pi$, where \( r_{i,h}^{\pi}(s) = \mathbb{E}_{\ba \sim \pi_h(s)} [r_{i,h}(s, \ba)] \) for any \( s \in \mathcal{S} \).
    
    \item \( r_{i,h}^{\pi_{-i}} \in \mathbb{R}^{S A_i} \): the reward of the $i$-th player associated with \( \pi_{-i} = \{\pi_{-i,h}\}_{h \in [H]} \). Specifically, \( r_{i,h}^{\pi_{-i}}(s, a_i) = \mathbb{E}_{\ba_{-i} \sim \pi_{-i,h}(s)} [r_{i,h}(s, \ba)] \) for any \( s \in \mathcal{S} \) and \( a_i \in A_i \).
\end{itemize}

\paragraph{Matrices of transition kernel variants.}
\label{sec:matrix_notation}
The following notations are associated with some transition kernel and possibly some joint policy \( \pi \):
\begin{itemize}[topsep=0pt,leftmargin=17pt]
    \setlength{\itemsep}{0pt}
    \item \( P_h^0 \in \mathbb{R}^{S \prod_{i=1}^n A_i \times S} \): the matrix representing the nominal transition kernel at time step \( h \). Specifically, \( P_{h,s, \ba}^0 \in \mathbb{R}^{1 \times S} \) represents the row corresponding to any state-action pair \( (s, \ba)  \in \mathcal{S} \times \cA \).
    \item \( P_h^{\pi_{-i}} \in \mathbb{R}^{S A_i \times S} \): the matrix representing the nominal transition kernel at time step \( h \), associated with the joint policy \( \pi_{-i} \). Each row of \( P_h^{\pi_{-i}}\) corresponds to one state-action pair \( (s, a_i) \in \cS\times \cA_i \), denoted as \( P_{h,s,a_i}^{\pi_{-i}} \in \mathbb{R}^{1 \times S} \). Specifically, 
    $P_{h,s,a_i}^{\pi_{-i}}(s^\prime) = \mathbb{E}_{\ba_{-i} \sim \pi_{-i,h}(s)} [ P_{h,s,\ba}^0(s^\prime) ]$ for all \( s, s^\prime \in \mathcal{S} \) and \( a_i \in \cA_i \).
    
    \item \( P_{i,h}^{k} \in \mathbb{R}^{S A_i \times S} \): the empirical transition kernel matrix output at the $k$-th iteration of time step \( h \) associated with  agent \( i \) (cf.~ line 6 of Algorithm~\ref{alg:summary}). \( P_{i,h,s,a_i}^{k} \in \mathbb{R}^{1 \times S} \) represents the row corresponding to any state-action pair \( (s, a_i)  \in \mathcal{S} \times \cA_i \).
    \item \( \underline{P}_{h}^{\pi} \in \mathbb{R}^{S \times S} \) and \( \underline{P}_{i,h}^{k} \in \mathbb{R}^{S \times S} \): defined as \( \underline{P}_h^{\pi} := \Pi_h^{\pi} P_h^0 \) and \( \underline{P}_{i,h}^{k}  := \Pi_h^{\pi_i} P_{i,h}^{k} \).
\end{itemize}

Next, we introduce some matrix notations for transitions that are associated with not only the nominal transition and policy \( \pi \), but also with some value functions:

\begin{itemize}[topsep=0pt,leftmargin=17pt]
    \setlength{\itemsep}{0pt}
    \item For any joint policy \( \pi \) and any value vector \( V \in \mathbb{R}^S \), we define \( P_{i,h}^{\pi_{-i},V} \in \mathbb{R}^{S A_i \times S} \) as the matrix representing the worst-case transition probability kernel centered around the nominal kernel $P^{\pi_{-i}}_{i,h}$ associated with $V$, for any \( (i,h) \in [n] \times [H] \). The row corresponding to the state-action pair \( (s, a_i) \) in \( P_{i,h}^{\pi_{-i},V} \), denoted as \( P_{i,h,s,a_i}^{\pi_{-i},V} \in \mathbb{R}^S \), is given by:
    \begin{subequations}\label{eq:inf-p-special-marl}
        \begin{align}
            P_{i,h,s,a_{i}}^{\pi_{-i},V} &= \mathrm{argmin}_{\mathcal{P} \in \mathcal{U}^{\sigma_i}(P_{h,s,a_i}^{\pi_{-i}})} \mathcal{P} V.
        \end{align}
        Similarly, we also define the transition matrices $P_{i,h}^{\pi, V} \in \mathbb{R}^{S A_i \times S} $ for specific value vectors as:
        \begin{align}
P_{i,h}^{\pi, V} := P_{i,h}^{\pi_{-i}, V_{i,h+1}^{\pi, \sigma_i}}, \quad \text{where} \quad P_{i,h,s,a_i}^{\pi, V} := P_{i,h,s,a_i}^{\pi_{-i}, V_{i,h+1}^{\pi, \sigma_i}} = \mathrm{argmin}_{\mathcal{P} \in \mathcal{U}^{\sigma_i}(P_{h,s,a_i}^{\pi_{-i}})} \mathcal{P} V_{i,h+1}^{\pi, \sigma_i}. \label{eq:preliminary-V-defn}
\end{align}
        Finally, we define square matrices \( \underline{P}_{i,h}^{\pi,V} \in \mathbb{R}^{S \times S} \) as:
        $
        \underline{P}_{i,h}^{\pi,V}:= \Pi_h^{\pi_i} P_{i,h}^{\pi_{-i}, V}.
        $

    \item By replacing the nominal transition kernel with the empirical transition kernel at some $k$-th iteration, we similarly define \( P_{i,h}^{k,V} \) as the worst-case probability transition kernel within the uncertainty set for agent \( i \), centered around the empirical kernel \( P_{i,h}^{k,V} \). The row corresponding to the state-action pair \( (s, a_i) \) in \( P_{i,h}^{k,V} \) is denoted as \( P_{i,h,s,a_i}^{k,V} \in \mathbb{R}^S \) and is defined as:
    \begin{align}
        P_{i,h,s,a_{i}}^{k,V} = \mathrm{argmin}_{\mathcal{P} \in \mathcal{U}^{\sigma_i}(P_{i,h,s,a_i}^{k})}\mathcal{P} V. \label{eq:preliminary-P-hat-k}
    \end{align}

    \end{subequations}
\end{itemize}

\paragraph{Variance.} 
We now introduce some notations for the variance associated with some specific probability distributions. For any probability vector \( P \in \mathbb{R}^{1 \times S} \) and a vector \( V \in \mathbb{R}^S \), we denote the variance of \( V \) with respect to \( P \) as \( \mathrm{Var}_{P}(V) \), defined as:
\begin{align}\label{eq:defn-variance}
  \mathrm{Var}_{P}(V) := P (V \circ V) - (P V) \circ (P V).
\end{align}
Additionally, for any transition kernel matrix \( P \in \mathbb{R}^{S A_i \times S} \) with each row as a probability distribution vector, and a vector \( V \in \mathbb{R}^S \), we define \( \mathsf{Var}_{P}(V) \in \mathbb{R}^{S A_i} \) as a vector of variances. The \( (s, a_i) \)-th entry of \( \mathsf{Var}_{P}(V) \) is given by:
\begin{align}\label{eq:defn-variance-vector-marl}
	\forall (s,a_i)\in\cS \times \cA_i: \quad \mathsf{Var}_{P}(s,a_i) := \mathsf{Var}_{P_{s,a_i}}(V),
\end{align}
where \( P_{s,a_i} \) denotes the \( (s, a_i) \)-th row of the transition matrix corresponding to state action pair \( (s,a_i) \).
\subsubsection{Additional facts}
\begin{lemma}
\label{lm:weight_variance}
For any transition kernels $P_1,\ldots,P_m\in\mathbb{R}^S$, and any weight $a_1,\ldots,a_m\in[0,1]$ satisfying $\sum_{i=1}^m a_i=1$, one has
\begin{align*}
    \sum_{i=1}^ma_i\sqrt{\mathsf{Var}_{P_i}(V)}\leq \sqrt{\sum_{i=1}^ma_i\mathsf{Var}_{P_i}(V)} \quad \text{and} \quad \sum_{i=1}^m a_i\mathsf{Var}_{P_i}(V)\leq \mathsf{Var}_{\sum_{i=1}^m a_i P_i}(V),
\end{align*}
where $V$ denote any fixed value vector $V\in\mathbb{R}^S$ obeying $0\leq V\leq H$ for all $s\in\mathcal{S}$.
\end{lemma}
\begin{proof}
    
Initially, since $f(x)=\sqrt{x}$ is a concave function, we have
    \begin{align*}
        \sum_{i=1}^ma_i\sqrt{\mathsf{Var}_{P_i}(V)}\leq \sqrt{\sum_{i=1}^ma_i\mathsf{Var}_{P_i}(V)}
    \end{align*}
    by Jensen's inequality. 
Recalling the definition in \eqref{eq:defn-variance} yields that
    \begin{align*}
        \sum_{i=1}^m a_i\mathsf{Var}_{P_i}(V)&=\sum_{i=1}^ma_i\left(\mathbb{E}_{P_i}\left(V\circ V\right)-\left(\mathbb{E}_{P_i}[V]\right)^2 \right)
        \leq\sum_{i=1}^ma_i\mathbb{E}_{P_i}\left(V\circ V\right)-\left(\sum_{i=1}^ma_i\mathbb{E}_{P_i}[V]\right)^2 \notag \\
        & = \mathsf{Var}_{\sum_{i=1}^m a_i P_i}(V).
    \end{align*}
    which holds by the elementary fact that $f(x)=x^2$ is a convex function.

\end{proof}
\subsection{Robust Bellman equations of \ormg}\label{sec:robust-bellman-equation}
Fortunately, the class of \ormg feature a robust counterpart of the Bellman equation ---  {\em robust Bellman equation}.
Specifically, for any joint policy $\pi: \cS\times [H] \mapsto \Delta(\cA)$, the robust value function can be expressed as
\begin{align}
V^{\pi,\ror_i}_{i,h}(s) &= \inf_{ \cU_{\rho}^{\sigma_i}\left(P^\no,\pi\right)} \mathbb{E}\left[\sum_{t=h}^H r_i(s_t,a_t) \mymid s_h =s \right] = \mathbb{E}_{\ba \sim \pi_h(s)}[r_{i,h}(s,\ba)] + \mathbb{E}_{a_i \sim \pi_{i,h}(s)} \bigg[ \inf_{\cU_{\rho}^{\sigma_i}\left(P^{\pi_{-i}}_{h,s,a_i} \right)} P V^{\pi,\ror_i}_{i,h+1} \bigg]. \label{eq:robust-bellman-equation}
\end{align}
It can be verified directly by  definition. 
The robust Bellman equation described above is intrinsically linked to the {\em others-integrated $(s,a_i)$-rectangularity} condition (cf.~\eqref{eq:def-of-s-ai-set}) of  the uncertainty set. This condition leads to a well-posed and computationally-tractable class of \rmgs by allowing the decomposition from an overall uncertainty set to independent subsets across different agents, time steps, and each state and agent-wise action pair $(s,a_i)$.

Note that the specified robust Bellman equation is different for a joint correlated policy and a joint product policy, induced by different expected nominal transition kernels.
In particular, for any joint product policy $\pi: \cS\times [H] \mapsto \prod_{i\in[n]} \Delta(\cA_i)$, the expected nominal transition kernel conditioned on the $i$-th agent's action $a_i\in\cA_i$, current state $s\in\cS$, and the policy $\pi$ can be expressed by 
\begin{align}
     P^{\pi_{-i}}_{h,s,a_i} &= \mathbb{E}_{\ba\sim \pi_{h}(\cdot \mymid s, a_i)} \left[ P^{0}_{h, s, \ba } \right] = \mathbb{E}_{\ba_{-i} \sim \pi_{-i,h}( \cdot \mymid s)} \left[ P^{0}_{h,s, (a_i,\ba_{-i})}\right]  \label{eq:uncertainty-product-policy-s-ai}
\end{align}
for any $(i,h,s,a_i) \in [n] \times [H] \times \cS \times \cA_i$, 
where the last equality holds since the policy $\pi$ is a product policy, and the distribution of $\ba_{-i}$ is independent of $a_i$. It is observed that the expected nominal transition kernel $P^{\pi_{-i}}_{h,s,a_i}$ for a product policy $\pi$ is independent of the $i$-th agent's policy given $(s, a_i)$. This differs from \eqref{eq:uncertainty-correlated-policy-s-ai} for a possibly correlated policy, where \eqref{eq:uncertainty-correlated-policy-s-ai} can generally depend on the $i$-th agent's policy.

\subsection{Preliminary facts of online adversarial learning}
\label{sec:FTRL}

Our proposed Algorithm~\ref{alg:summary} is inspired by online adversarial learning, which plays a crucial role in addressing the challenge of multiagency in multi-agent sequential games \citep{jin2021v,bai2020provable,song2021can,li2023minimax}. In this section, we introduce the fundamentals of online adversarial learning and review key aspects of a widely-used algorithm, the Follow-the-Regularized-Leader (FTRL).

\paragraph{Online learning for weighted average loss.}
\label{sec:online-learning-weighted}
We consider an online learning problem over $K$ steps for some positive integer $K$, commonly studied in adversarial learning \citep{lattimore2020bandit}. The learner is presented with an action set $\cA$ and (possibly different) loss functions \( f_1, \ldots, f_K : \cA \to \mathbb{R}_{\geq 0} \) for each step $k\in[K]$. At each time step \( k \), the learner selects a distribution over the action set, \( \pi_k \in \Delta(\cA) \), and observes the loss function \( f_k(\pi_k) \). The goal of the learner is to minimize the weighted average loss over the $K$ steps, which is defined as
\begin{align}
L_K = \sum_{k=1}^K \alpha_k^K f_k(\pi_k). \label{eq:FTRL-sum-loss}
\end{align}
To evaluate the learner's performance, the regret for the online learning process is defined as:
\begin{align}
R_K = \sum_{k=1}^K \alpha_k^K f_k(\pi_k) - \min_{\pi \in \Delta(\cA)} \sum_{k=1}^K \alpha_k^K f_k(\pi) .
\label{eq:online-learning-goal}
\end{align}

\paragraph{FTRL and its regret bound.}
To solve adversarial learning, a widely-used method for solving the online learning problem described above is the Follow-the-Regularized-Leader (FTRL) algorithm, introduced by \citet{shalev2007primal,shalev2007online}. At time step \( k+1 \), the learner selects a regularized greedy policy by solving:
\begin{align}
    \pi_{k+1} = \mathrm{argmin}_{\pi \in \Delta(\cA)} \left[ \sum_{i=1}^k\alpha_i^k f_i(\pi) + F_k(\pi) \right], \quad k = 1, 2, \ldots,
	\label{eq:FTRL-update-general}
\end{align}
where \( F_k(\pi) \) represents a convex regularization function. The following theorem provides a refined regret bound for the FTRL algorithm when the loss functions $\{f_k\}_{k\in[K]}$ are linear with respect to the policy, and the 

\begin{theorem}[Theorem 3, \citet{li2022minimax}]
\label{thm:FTRL-refined}
For any time step \( k \in [K] \) and the current policy \( \pi_k \), the loss function is defined as \( f_k(\pi)=\left<\pi_k,l_k\right> \), where \( l_k \in \mathbb{R}^{|\cA|} \) represents some loss vector at time step $k$. 

Let the regularization function to be \( F_k(\pi) \defn \sum_{a \in \cA}\frac{1}{\eta_{k+1}} \pi(a) \log(\pi(a)) \), where $\eta_{k+1} > 0$ denotes the learning rate. And the policy \( \pi_{k+1} \) in episode \( k+1 \) updated according to the following FTRL algorithm:
\begin{align}
	\pi_{k+1}= \mathrm{argmin}_{\pi \in \Delta(\cA)} \left\{ \left< \pi, L_k \right> +  F_k(\pi) \right\},
	\label{eq:FTRL-update-entropy}
\end{align}
which indicates
\begin{align}
\forall a\in\cA: \quad \pi_{k+1}= \frac{\exp \big( -\eta_{k+1} L_k(a) \big)}{\sum_{a' \in \cA} \exp \big( -\eta_{k+1} L_k(a') \big)}.
\end{align}
 Suppose \( 0 < \alpha_1 \leq 1 \) and \( \eta_1 = \eta_2(1 - \alpha_1) \), and assume \( 0 < \alpha_k < 1 \) and \( 0 < \eta_{k+1}(1 - \alpha_k) \leq \eta_k \) for all \( k \geq 2 \).
Then, the regret of the FTRL algorithm is bounded by:
\begin{align}
R_K & = \sum_{k=1}^K \alpha_k^K \left< \pi_k, l_k \right> - \min_{a \in \cA} \left[ \sum_{k=1}^K \alpha_k^K l_k(a) \right] \notag \\
&\leq \frac{5}{3} \sum_{k=1}^{K} \alpha_k^K \widehat{\eta}_k \alpha_k \mathsf{Var}_{\pi_k}(l_k)
+ \frac{\log A}{\eta_{K+1}} + 3 \sum_{k=1}^{K} \alpha_k^K \widehat{\eta}_k^2 \alpha_k^2 \| l_k \|_{\infty}^3 \mathbb{I} \left( \widehat{\eta}_k \alpha_k \| l_k \|_{\infty} > \frac{1}{3} \right).
	\label{eq:FTRL-refined}
\end{align}
where
\begin{align}
\widehat{\eta}_k :=
\begin{cases}
\eta_2, & \text{if } k = 1, \\
\frac{\eta_k}{1 - \alpha_k}, & \text{if } k > 1.
\end{cases}
\end{align}

\end{theorem}

\paragraph{Properties of the learning rates.}
In addition, we introduce the following lemma regarding the properties of the learning rate $\{\alpha_k\}_{k\in[K]}$ and $\{\eta_k\}_{k\in[K]}$, defined in \eqref{eq:online-learning-rate} and \eqref{eq:policy-update-exponential-alg}.

\begin{lemma}[{\citet[Lemma 1]{li2023minimax}}]
    \label{lem:weight}
For any $k\geq 1$, the learning rate $\{\alpha_k\}_{k\in[K]}$ satisfy
\begin{subequations}
\begin{align}
\alpha_{1}&=1,\qquad\sum_{i=1}^{k}\alpha_{i}^{k}=1,\qquad\max_{1\leq i\leq k}\alpha_{i}^{k}\leq\frac{2c_\alpha\log K}{k},
    \label{eq:alpha-properties} \\
    1-\alpha_{k} & =1-\frac{c_{\alpha}\log K}{k-1+c_{\alpha}\log K}\geq\begin{cases}
    1-\frac{c_{\alpha}\log K}{1+c_{\alpha}\log K}=\frac{1}{1+c_{\alpha}\log K}\geq\frac{1}{2c_{\alpha}\log K}, & \text{if }k\geq2,\\
    1-\frac{c_{\alpha}\log K}{K/2+c_{\alpha}\log K}=\frac{K}{K+2c_{\alpha}\log K}\geq\frac{1}{2}, & \text{if }k \geq K/2+1.
    \end{cases}  \label{eq:alpha-properties-2}
\end{align}
In addition, if $k \geq c_{\alpha}\log K + 1$ and $c_{\alpha} \geq 24$, then one has
\begin{equation}
    \max_{1\leq i\leq k/2}\alpha_{i}^{k}\leq\frac{1}{K^6}. 
    \label{eq:alpha-properties-2}
\end{equation}
\end{subequations}
\end{lemma}

\begin{lemma}\label{lem:weight-2}
The learning rates of the online learning $\{\eta_k\}_{k\in[K]}$ obey
\begin{align}
\bigg(\frac{\eta_{k}}{\eta_{k+1}}\bigg)^{2} & >(1-\alpha_{k})^{2} \quad \text{and} \quad \eta_{k}\alpha_{k} = \sqrt{\frac{2c_{\alpha}\log^{2}K}{kH}}. \label{eq:eta-k-condition-check}
\end{align}
\end{lemma}

\begin{proof}

By basic calculus, we have
\begin{align}
    \bigg(\frac{\eta_{k}}{\eta_{k+1}}\bigg)^{2} & =\frac{\alpha_{k}}{\alpha_{k-1}} 
         =\frac{k-2+c_{\alpha}\log K}{k-1+c_{\alpha}\log K}
         \geq \frac{k-1}{k-1+c_{\alpha}\log K}
         = 1 - \alpha_k
         >(1-\alpha_{k})^{2}, \notag \\
         \eta_{k}\alpha_{k} &= \sqrt{\frac{\log K}{\alpha_{k-1}H}}\cdot\alpha_{k} \leq\sqrt{\frac{\log K}{\alpha_{k}H}}\cdot\alpha_{k}=\sqrt{\frac{\alpha_{k}\log K}{H}}\leq\sqrt{\frac{2c_{\alpha}\log^{2}K}{kH}}.
        \label{eq:alphak-properties-UB-LB-135}
    \end{align}
    
\end{proof}

%% file: sub-files-for-arxiv/appendix_game.tex
%!TEX root = ./../DRO-MARL-near-optimal.tex

\section{Proof for Section~\ref{sec:rmg_definition}}\label{proof:thm:existence-of-ne}

Before proceeding, we introduce some useful definition and existing facts that are standard in real analysis and game theory literature.

\begin{definition}[Upper semi-continuous]\label{def:semi-continous}
A point-to-set mapping $x \in \cX \mapsto \phi(x) \in \cY$ is upper semi-continuous if $\lim_{n\rightarrow\infty}x^n =x_0, y^n\in \phi(x^n), 
\lim_{n\rightarrow \infty} y^n = y_0$ imply that $y^0\in \phi(x_0)$.
\end{definition}

\begin{theorem}[Kakutani's fixed point Theorem \citep{kakutani1941generalization}]\label{thm:kakutani}
If $X$ is a closed, bounded, and convex set in a Euclidean space, and $\phi$ is a upper semi-continuous correspondence mapping $X$ into the family of all closed convex subsets of $X$, then there exists $x\in X$ so that $x\in\phi(x)$.
\end{theorem}

\subsection{Proof of Theorem~\ref{thm:existence-of-ne}}

\paragraph{Step 1: constructing an auxiliary single-step game.}

Focusing on finite-horizon \ormg \rmgin	$ = \big\{ \cS, \{\cA_i\}_{1 \le i \le n},\{\cU_{\rho}^{\ror_i}(P^\no,\cdot)\}_{1 \le i \le n}, \rew, H \big\}$, we shall verify the theorem by firstly consider a one-step game and then apply the results recursively to the sequential Markov games.

Without loss of generality, we focus on any of the steps $h\in[H]$ and construct an auxiliary one-step game. Towards this, we first introduce a fixed value function $V_{i,h+1} \in\mathbb{R}^S$ with $0 \leq V_{i,h+1} \leq H$ for the $i$-th agent, representing the possible value function obtained at the next time  step $h+1$. Focusing on time step $h$, and any joint product policy $\pi: \cS \mapsto \prod_{i\in[n]} \Delta(\cA_i)$ at this time step, we abuse the notation defined in  \eqref{eq:uncertainty-product-policy-s-ai} to denote the expected nominal transition kernel over each $(s,a_i)$ as:
\begin{align}
	P^{\pi_{-i}}_{h,s,a_i} &= \mathbb{E}_{\pi( \ba_{-i} \mymid s,a_i)} \left[ P^{0}_{h,s, (a_i,\ba_{-i})}\right] = \mathbb{E}_{\pi_{-i}( \ba_{-i} \mymid s)} \left[ P^{0}_{h,s, (a_i,\ba_{-i})}\right].
\end{align}

Armed with this, for any joint product policy $\pi: \cS \mapsto \prod_{i\in[n]} \Delta(\cA_i)$, we can define the payoffs to maximize for the players as below:
\begin{align}
\forall s\in\cS: \quad f_{i,s}(\pi_i(s), \pi_{-i}(s); V_{i,h+1} ) = \mathbb{E}_{\ba \sim \pi(s) }[r_{i,h}(s,\ba)] + \mathbb{E}_{a_i \sim \pi_{i}(s)} \left[ \inf_{\cU^{\sigma_i}\left(P^{\pi_{-i}}_{h,s,a_i} \right)} P V_{i,h+1}\right], \label{eq:lemma-continuous-2-payoff}
\end{align}
which is defined analogous to the robust Bellman equation (cf.~\eqref{eq:robust-bellman-equation}) by replacing a real robust value function vector (associated with some policy) to some fixed vector $V_{i,h+1}$.

Now we are ready to introduce the following useful mapping: for any $\pi:\cS \mapsto \prod_{i\in[n]} \Delta(\cA_i)$,
\begin{align}
\phi(\pi) \defn \left\{ u :\cS \mapsto \prod_{i\in[n]} \Delta(\cA_i) \mymid u_i(s) \in \mathrm{argmax}_{\pi_i'(s)\in\Delta(\cA_i)} \;f_{i,s}(\pi_i'(s), \pi_{-i}(s); V_{i,h+1}), \forall (i,s)\in [n] \times \cS  \right\}. \label{eq:definiton-nash-phi}
\end{align}

\paragraph{Step 2: the existence of NE in the auxiliary game.}
To apply Theorem~\ref{thm:kakutani}, there are three required conditions. First, let $X = \{ \pi: \cS \mapsto \prod_{i\in[n]} \Delta(A_i)\}$, the space of product policies, which is known as a closed, bounded and convex set in Euclidean space.
\begin{itemize}
\item {\bf Verifying that $\phi(\pi)$ is an upper semi-continuous correspondence.} 
Before starting, we introduce the following two useful lemmas with the proof postponed to Appendix~\ref{proof:lemma:continuous-f} and \ref{proof:lemma:continuous-g}.
\begin{lemma}\label{lemma:continuous-f}
The set of function $\left\{ f_{i,s}(\pi_i'(s), \pi_{-i}(s); V_{i,h+1}), 0 \leq V_{i,h+1} \leq H\right\}$ is equicontinuous with respect to $\pi_i'(s), \pi_{-i}(s)$ for all $(i,s)\in [n]\times \cS$. 
\end{lemma}

\begin{lemma}\label{lemma:continuous-g}
For any $i\in[n]$ and then $x_{-i}: \cS \mapsto \prod_{j\neq i, j\in[n]} \Delta(\cA_j)$, the functions
\begin{align}
\forall s\in\cS: \quad g_{i,s}(x_{-i}(s),V_{i,h+1}) \defn \mathrm{max}_{\pi_i'(s)\in\Delta(\cA_i)} \; f_{i,s}(\pi_i'(s), x_{-i}(s); V_{i,h+1}) \label{defn-g-auxiliary}
\end{align}
are continuous with respect to $x_{-i}(s)$ and the set $\{g_{i,s}(\cdot , V) | V\in\mathbb{R}^S, 0\leq V \leq H\}$ is equicontinuous.
\end{lemma}

Armed with above lemmas, we are in the position to prove this condition.
We suppose there are two sequence $\lim_{n\rightarrow\infty}x^n =x^0, y^n\in \phi(x^n), 
\lim_{n\rightarrow \infty} y^n = y^0$.
Recall the definition of a upper semi-continuous correspondence (cf.~Definition~\ref{def:semi-continous}), we are supposed to show that $y^0 \in \phi(x^0)$, i.e.,
\begin{align}
\forall (i,s)\in [n]\times \cS: \quad f_{i,s}(y^0_i(s), x^0_{-i}(s); V_{i,h+1}) = \mathrm{max}_{\pi_i'(s)\in \Delta(\cA_i)} \;f_{i,s}(\pi_i'(s), x^0_{-i}(s); V_{i,h+1}).
\end{align}

Towards this, we have
\begin{align}
& | f_{i,s}(y^0_i(s), x^0_{-i}(s); V_{i,h+1}) - g_{i,s}(x^0_{-i}(s), V_{i,h+1}) | \notag \\
& \leq | f_{i,s}(y^0_i(s), x^0_{-i}(s); V_{i,h+1}) - f_{i,s}(y^n_i(s), x^n_{-i}(s); V_{i,h+1}) | \notag \\
&\quad + |f_{i,s}(y^n_i(s), x^n_{-i}(s); V_{i,h+1}) - g_{i,s}(x^0_{-i}(s), V_{i,h+1}) | \notag \\
& \overset{\mathrm{(i)}}{=} | f_{i,s}(y^0_i(s), x^0_{-i}(s); V_{i,h+1}) - f_{i,s}(y^n_i(s), x^n_{-i}(s); V_{i,h+1}) | + |g_{i,s}(x^n_{-i}(s), V_{i,h+1}) - g_{i,s}(x^0_{-i}(s), V_{i,h+1}) | \notag \\
& \qquad \rightarrow 0 \qquad \text{as } \quad n\rightarrow \infty,
\end{align}
where the first inequality follows from the triangle inequality, (i) holds by the assumption $y^n \in \phi(x^n)$ so that $f_{i,s}(y^n_i(s), x^n_{-i}(s); V_{i,h+1}) = \mathrm{max}_{\pi_i'\in\Delta(\cS)} \; f_{i,s}(\pi_i'(s), x_{-i}^n(s); V_{i,h+1})$, and the last line can be verified by the continuity verified by Lemma~\ref{lemma:continuous-f} and Lemma~\ref{lemma:continuous-g}.

\item {\bf Verifying $\phi(\pi)$ is convex for any $\pi\in X$.}
Finally, we gonna work on the convexity of $\phi(\pi)$ for any $\pi\in X$. To begin with, by the definition of $\phi(\pi)$ in \eqref{eq:definiton-nash-phi}, we know that $\phi(\pi) \subseteq X$ and the maximum of the continuous function $f_{i,s}(\pi_i(s), \pi_{-i}(s); V_{i,h+1})$ (cf.~Lemma~\ref{lemma:continuous-f}) on a compact set $\Delta(\cA_i)$ exists, i.e., $\phi(x) \neq \emptyset$.

Suppose there exists two Nash equilibrium $z:\cS \mapsto \prod_{i\in[n]} \Delta(\cA_i), v: \cS \mapsto \prod_{i\in[n]} \Delta(\cA_i)$ and  $z,v \in \phi(\pi)$.  Then we have that for  any $(i,s)\in [n] \times \cS$,
\begin{align}
 f_{i,s}(z_i(s), \pi_{-i}(s); V_{i,h+1} ) = f_{i,s}(v_i(s), \pi_{-i}(s); V_{i,h+1} ) = \max_{u_i(s)\in\Delta(\cA_i) }f_{i,s}(u_i(s), \pi_{-i}(s); V_{i,h+1} ).
\end{align}
To continue, for any $0\leq \lambda \leq 1$, one has
\begin{align}
&\max_{u_i(s)\in\Delta(\cA_i) }f_{i,s}(u_i(s), \pi_{-i}(s); V_{i,h+1} ) = \lambda  f_{i,s}(z_i(s), \pi_{-i}(s); V_{i,h+1} ) + (1-\lambda) f_{i,s}(v_i(s), \pi_{-i}(s); V_{i,h+1} ) \nonumber \\
& = \lambda \bigg(\mathbb{E}_{a_i\sim z_{i}(s)} \left[r^{\pi_{-i}}_{i,h}(s, a_i)\right] +  \mathbb{E}_{a_i \sim z_{i}(s)} \bigg[ \inf_{\cU^{\sigma_i}\left(P^{\pi_{-i}}_{h,s,a_i} \right)} P V_{i,h+1}\bigg]\bigg)  \nonumber \\
&\quad + (1-\lambda) \bigg(\mathbb{E}_{a_i\sim v_{i}(s)} \left[r^{\pi_{-i}}_{i,h}(s, a_i)\right] + \mathbb{E}_{a_i \sim v_{i}(s)} \bigg[ \inf_{\cU^{\sigma_i}\left(P^{\pi_{-i}}_{h,s,a_i} \right)} P V_{i,h+1}\bigg] \bigg)  \\
&= \mathbb{E}_{a_i\sim [\lambda z_{i}(s) + (1-\lambda) v_{i}(s)] } \left[r^{\pi_{-i}}_{i,h}(s, a_i)\right] + \mathbb{E}_{a_i\sim [\lambda z_{i}(s) + (1-\lambda) v_{i}(s)] }  \bigg[ \inf_{\cU^{\sigma_i}\left(P^{\pi_{-i}}_{h,s,a_i} \right)} P V_{i,h+1}\bigg] \notag \\
& = f_{i,s}(\lambda z_{i}(s) + (1-\lambda) v_{i}(s), \pi_{-i}(s); V_{i,h+1} ). 
\end{align}
where we denote $r_{i,h}^{\pi_{-i}}(s,a_i) \defn \mathbb{E}_{\ba_{-i}\sim \pi_{-i}(s)}\left[r_{i,h}(s,(a_i,\ba_{-i}))\right]$.
Hence, we show that $\lambda z_{i}(s) + (1-\lambda) v_{i}(s) \in \phi(\pi)$ for all $(i,s)\in [n]\times \cS$ and $0\leq \lambda \leq 1$, thus verify that $\phi(\pi)$ is convex for any $\pi\in X$.

\end{itemize}
\paragraph{Step 3: the existence of robust NE in \ormg.}
Armed with above results, now we consider a general form to show that there exists a policy $\pi:  \cS \times [H] \mapsto \prod_{i\in[n]} \Delta(\cA_i)$ that satisfies
\begin{align}
\forall (i,h,s)\in[n] \times [H] \times \cS:\quad V_{i,h}^{\pi, \ror_i}(s)=V_{i,h}^{\star,\pi_{-i}, \ror_i}(s).
\end{align}

We shall prove this by induction.

\begin{itemize}
	\item {\bf The base case.} Starting with the final step $h=H$, we recall that by definition, for any joint policy $\pi: \cS \times [H]  \mapsto \prod_{i\in[n]} \Delta(\cA_i)$:
	\begin{align}
	\forall (i,s)\in[n]  \times \cS: \quad V^{\pi,\ror_i}_{i,H+1}(s) = 0.
	\end{align}

	To apply the results in the one-step game constructed in Step 2, 
	we consider the one-step game at $h=H$ and using the payoff function (cf.~\eqref{eq:lemma-continuous-2-payoff})
 \begin{align}
     \forall s\in\cS: \quad f_{i,s}(\pi_i(s), \pi_{-i}(s); V^{\pi,\ror_i}_{i,H+1} ) = \mathbb{E}_{\ba \sim \pi(s) }[r_{i,h}(s,\ba)].
 \end{align}
We know that there exists a policy $\pi$ so that 
	\begin{align}
	\forall (i,s)\in[n]  \times \cS: \quad V_{i,H}^{\pi, \ror_i}(s)=V_{i,H}^{\star,\pi_{-i}, \ror_i}(s)
	\end{align}
 by setting $\pi_H$ as the NE of the one-step auxiliary game.
 
	\item {\bf Induction.}
	Assuming that there exists a policy $\pi$ so that for subsequent steps $h+1,\cdots, H$,
	\begin{align}
	\forall (i,h,s)\in[n] \times \{h+1,\cdots, H\} \times \cS:\quad V_{i,h}^{\pi, \ror_i}(s)=V_{i,h}^{\star,\pi_{-i}, \ror_i}(s),
	\end{align}
 which are achieved by determining certain policies for $\{ \pi_{h+1}, \pi_{h+2}, \cdots, \pi_{H}\}$.
We are supposed to prove that at time step $h$, we can ensure our policy $\pi$ satisfying
	\begin{align}
	\forall (i,s)\in[n]  \times \cS: \quad V_{i,h}^{\pi, \ror_i}(s)=V_{i,h}^{\star,\pi_{-i}, \ror_i}(s)
	\end{align}
by choosing a proper policy $\pi_h$ at the time step $h$.

Towards this, it is observed that
\begin{align}
V_{i,h}^{\star,\pi_{-i}, \ror_i}(s) & = \max_{\pi'_i: \cS \times [H] \mapsto \Delta(\mathcal{A}_i)} V_{i,h}^{\pi'_i \times \pi_{-i}, \ror_i }(s) \notag  \\
& = \max_{\pi'_i: \cS \times [H] \mapsto \Delta(\mathcal{A}_i)}   \mathbb{E}_{\ba \sim \pi'_{i,h}(s) \times \pi_{-i,h}(s)}[r_{i,h}(s,\ba)] + \mathbb{E}_{a_i \sim \pi'_{i,h}(s)} \bigg[ \inf_{P \in \cU^{\sigma_i}\left(P^{\pi_{-i}}_{h,s,a_i} \right)} P V^{\pi_i'\times \pi_{-i},\ror_i}_{i,h+1} \bigg] \notag \\
& =  \max_{\pi'_{i,h}(s) \in \Delta(\mathcal{A}_i)}   \mathbb{E}_{\ba \sim \pi'_{i,h}(s) \times \pi_{-i,h}(s)}[r_{i,h}(s,\ba)] \notag \\
&\quad +  \max_{\pi'_{i,h}(s) \in \Delta(\mathcal{A}_i)} \mathbb{E}_{a_i \sim \pi'_{i,h}(s)} \max_{\pi'_{i,h^+}: \cS \times h^+ \mapsto \Delta(\mathcal{A}_i)} \bigg[ \inf_{P \in \cU^{\sigma_i}\left(P^{\pi_{-i}}_{h,s,a_i} \right)} P V^{\pi_i'\times \pi_{-i},\ror_i}_{i,h+1}\bigg] \notag \\
& = \max_{\pi'_{i,h}(s) \in \Delta(\mathcal{A}_i)}  \left[ \mathbb{E}_{\ba \sim \pi'_{i,h}(s) \times \pi_{-i,h}(s)}[r_{i,h}(s,\ba)] + \mathbb{E}_{a_i \sim \pi'_{i,h}(s)} \left[ \inf_{ P \in \cU^{\sigma_i}\left(P^{\pi_{-i}}_{h,s,a_i} \right)} P V^{\star, \pi_{-i},\ror_i}_{i,h+1}\right] \right]. \label{eq:ne-sequential-result1}
\end{align}
where we denote $h^+ = \{ h+1,h+2,\cdots, H\}$ as the set that includes all the time steps after $h$ until the end of the episode, and the last equality follows from the fact
\begin{align}
    \max_{\pi'_{i,h^+}: \cS \times h^+ \mapsto \Delta(\mathcal{A}_i)} \bigg[ \inf_{\cU^{\sigma_i}\left(P^{\pi_{-i}}_{h,s,a_i} \right)} P V^{\pi_i'\times \pi_{-i},\ror_i}_{i,h+1}\bigg]  & = \inf_{\cU^{\sigma_i}\left(P^{\pi_{-i}}_{h,s,a_i} \right)} P \max_{\pi'_{i,h^+}: \cS \times h^+ \mapsto \Delta(\mathcal{A}_i)} V^{\pi_i'\times \pi_{-i},\ror_i}_{i,h+1}\bigg] \notag \\
    & = \inf_{\cU^{\sigma_i}\left(P^{\pi_{-i}}_{h,s,a_i} \right)} P 
 V^{\star, \pi_{-i},\ror_i}_{i,h+1},
\end{align}
which holds by the definition of $V^{\star, \pi_{-i},\ror_i}_{i,h+1}$.
Combining \eqref{eq:ne-sequential-result1} and the results in the auxiliary one-step game with $V_{i,h+1} = V^{\star, \pi_{-i},\ror_i}_{i,h+1}$, one has that there exists a policy with $\pi_{h}$ that satisfies
 \begin{align}
 \forall (i,s)\in[n]  \times \cS: \quad V_{i,h}^{\pi, \ror_i}(s)=V_{i,h}^{\star,\pi_{-i}, \ror_i}(s).
 \end{align}

	\end{itemize}

Combining the results in the base case and induction, we complete the proof by recursively choosing $\pi_h: \cS \mapsto \prod_{i\in[n]} \Delta(\cA_i)$ for $h=H,H-1,\cdots, 1$ as the NE of the corresponding one-step auxiliary game at time step $h$ and arrive at
\begin{align}
\forall (i,s)\in[n]  \times \cS: \quad V_{i,1}^{\pi, \ror_i}(s)=V_{i,1}^{\star,\pi_{-i}, \ror_i}(s).
\end{align}

\subsection{Proof of auxiliary facts}

\subsubsection{Proof of Lemma~\ref{lemma:existence-best-reponse}} \label{proof:lemma:existence-best-reponse}

Without loss of generality, we consider any $i\in[n]$ with other agents' policy as $\pi_{-i}:  \cS \times [H]  \mapsto \Delta(\cA_{-i})$ fixed.
We shall prove this lemma by induction, by recursively showing that for each $(h,s)$, there exist a policy $\widetilde{\pi}_{i,h}(s)$ that satisfies $V_{i,h}^{\widetilde{\pi}_i \times \pi_{-i}, \ror_i}(s)  = V_{i,h}^{\star,\pi_{-i}, \ror_i}(s)$.
\begin{itemize}
	\item {\bf The base case.} Consider the base case $h=H$. Conditioned on other agents' policy $\pi_{-i}: \cS \times [H]  \mapsto \Delta(\cA_{-i})$, the maximum of the robust value function of the $i$-th agent can be expressed by
	\begin{align}
		\forall s\in\cS: \quad V_{i,H}^{\star,\pi_{-i}, \ror_i}(s) & = \max_{\pi'_i: \cS \times [H] \mapsto \Delta(\mathcal{A}_i)} V_{i,H}^{\pi'_i \times \pi_{-i}, \ror_i }(s) \notag \\
		&  = \max_{\pi'_i: \cS \times [H] \mapsto \Delta(\mathcal{A}_i)}\mathbb{E}_{a_i\sim \pi_{i,H}'(s)}\left[\mathbb{E}_{\ba_{-i}\sim \pi_{-i,H}(s)} [r_{i,H}(s,\ba)] \right] \notag \\
		& = \max_{\pi'_{i,H}(s)\sim \Delta(\mathcal{A}_i)}\mathbb{E}_{a_i\sim \pi_{i,H}'(s)}\left[\mathbb{E}_{\ba_{-i}\sim \pi_{-i,H}(s)} [r_{i,H}(s,\ba)] \right].
	\end{align}
	Since the maximum of the continuous function $\mathbb{E}_{a_i\sim \pi_{i,H}'(s)}\left[\mathbb{E}_{\ba_{-i}\sim \pi_{-i,H}(s)} [r_{i,H}(s,\ba)] \right]$ on a compact set $\Delta(\cA_i)$ exists, by setting
	\begin{align}
	\forall s\in\cS: \quad \widetilde{\pi}_{i,H}(s) = \mathrm{argmax}_{\pi'_{i,H}(s)\sim \Delta(\mathcal{A}_i)}\mathbb{E}_{a_i\sim \pi_{i,H}'(s)}\left[\mathbb{E}_{\ba_{-i}\sim \pi_{-i,H}(s)} [r_{i,H}(s,\ba)] \right],
	\end{align}
	we arrive at
	\begin{align}
	\forall s\in\cS: \quad  V_{i,H}^{\widetilde{\pi}_i \times \pi_{-i}, \ror_i }(s) = V_{i,H}^{\star,\pi_{-i}, \ror_i}(s).
	\end{align}
This complete the proof for the base case.

\item {\bf Induction.}
Assuming that for $t=h+1,h+2,\cdots, H$, we have
\begin{align}
\forall s\in\cS: \quad  V_{i,t}^{\widetilde{\pi}_i \times \pi_{-i}, \ror_i }(s) = V_{i,t}^{\star,\pi_{-i}, \ror_i}(s). \label{eq:induction-lemma-1}
\end{align}
Then, we want to prove for the step $h$, where the  maximum of the robust value function of the $i$-th agent can be expressed as: for all $s\in\cS$,
\begin{align}
		 &V_{i,h}^{\star,\pi_{-i}, \ror_i}(s) \notag \\
		& = \max_{\pi'_i: \cS \times [H] \mapsto \Delta(\mathcal{A}_i)} V_{i,h}^{\pi'_i \times \pi_{-i}, \ror_i }(s) \notag \\
		&  = \max_{\pi'_i: \cS \times [H] \mapsto \Delta(\mathcal{A}_i)} \mathbb{E}_{a_i\sim \pi_{i,h}'(s)}\left[\mathbb{E}_{\ba_{-i}\sim \pi_{-i,h}(s)} [r_{i,h}(s,\ba)] \right]  + \mathbb{E}_{a_i \sim \pi_{i,h}'(s)} \left[ \inf_{\cU^{\sigma_i}_{\rho}\left(P^{\pi_{-i}}_{h,s,a_i} \right)} P V^{\star, \pi_{-i},\ror_i}_{i,h+1}\right]\notag \\
		&  \overset{\mathrm{(i)}}{=} \max_{\pi'_i: \cS \times [H] \mapsto \Delta(\mathcal{A}_i)} \mathbb{E}_{a_i\sim \pi_{i,h}'(s)}\left[\mathbb{E}_{\ba_{-i}\sim \pi_{-i,h}(s)} [r_{i,h}(s,\ba)] \right]  + \mathbb{E}_{a_i \sim \pi_{i,h}'(s)} \left[ \inf_{\cU^{\sigma_i}_{\rho}\left(P^{\pi_{-i}}_{h,s,a_i} \right)} P V^{\widetilde{\pi}_i\times \pi_{-i},\ror_i}_{i,h+1}\right]\notag \\
		& = \max_{\pi'_{i,h}(s)\sim \Delta(\mathcal{A}_i)}\mathbb{E}_{a_i\sim \pi_{i,h}'(s)}\left[\mathbb{E}_{\ba_{-i}\sim \pi_{-i,h}(s)} [r_{i,h}(s,\ba)] \right]  + \mathbb{E}_{a_i \sim \pi_{i,h}'(s)} \left[ \inf_{\cU^{\sigma_i}_{\rho}\left(P^{\pi_{-i}}_{h,s,a_i} \right)} P V^{\widetilde{\pi}_i\times \pi_{-i},\ror_i}_{i,h+1}\right].
	\end{align}
where (i) holds by the induction assumption in \eqref{eq:induction-lemma-1}. Similarly to the base case, the maximum of the continuous function $\mathbb{E}_{a_i\sim \pi_{i,h}'(s)}\left[\mathbb{E}_{\ba_{-i}\sim \pi_{-i,h}(s)} [r_{i,h}(s,\ba)] \right]  + \mathbb{E}_{a_i \sim \pi_{i,h}'(s)} \left[ \inf_{\cU^{\sigma_i}_{\rho}\left(P^{\pi_{-i}}_{h,s,a_i} \right)} P V^{\widetilde{\pi}_i\times \pi_{-i},\ror_i}_{i,h+1}\right]$ on a compact set $\Delta(\cA_i)$ exists. So without conflict, for all $s\in\cS$, we can set
\begin{align}
	 &\widetilde{\pi}_{i,h}(s) \notag \\
	 &= \mathrm{argmax}_{\pi'_{i,h}(s)\sim \Delta(\mathcal{A}_i)}\mathbb{E}_{a_i\sim \pi_{i,h}'(s)}\left[\mathbb{E}_{\ba_{-i}\sim \pi_{-i,h}(s)} [r_{i,h}(s,\ba)] \right]  + \mathbb{E}_{a_i \sim \pi_{i,h}'(s)} \left[ \inf_{\cU^{\sigma_i}_{\rho}\left(P^{\pi_{-i}}_{h,s,a_i} \right)} P V^{\widetilde{\pi}_i\times \pi_{-i},\ror_i}_{i,h+1}\right], \label{eq:lemma-1-set-pi-h}
\end{align}
since the function $\inf_{\cU^{\sigma_i}_{\rho}\left(P^{\pi_{-i}}_{h,s,a_i} \right)} P V^{\widetilde{\pi}_i\times \pi_{-i},\ror_i}_{i,h+1}$ and especially $V^{\widetilde{\pi}_i\times \pi_{-i},\ror_i}_{i,h+1}$ are independent from the policy in the first $h$ steps ($\{\widetilde{\pi}_{i,t}(s)\}_{s\in\cS, t\in[h]}$).

Consequently, \eqref{eq:lemma-1-set-pi-h} directly implies that
\begin{align}
\forall s\in\cS: \quad  V_{i,h}^{\widetilde{\pi}_i \times \pi_{-i}, \ror_i }(s) = V_{i,h}^{\star,\pi_{-i}, \ror_i}(s). \label{eq:induction-lemma-1-result}
\end{align}

	\end{itemize}

Combining the results in the base case and the induction, we complete the proof by concluding that there exists a policy $\widetilde{\pi}$ satisfying
\begin{align}
\forall (h,s)\in [H] \times \cS: \quad  V_{i,h}^{\widetilde{\pi}_i \times \pi_{-i}, \ror_i }(s) = V_{i,h}^{\star,\pi_{-i}, \ror_i}(s). \label{eq:induction-lemma-1-result-final}
\end{align}

\subsubsection{Proof of Lemma~\ref{lemma:continuous-f}}\label{proof:lemma:continuous-f}
First, we define the distance metric between any two policy $\pi,\pi' \in X =\{ \pi: \cS \mapsto \prod_{i\in[n]} \Delta(A_i)\}$ as below:
\begin{align}
d(\pi, \pi') \defn \max_{i\in[n]} \max_{(s,a_i)\in \cS\times \cA_i} |\pi_i(a_i \mymid s) - \pi'_i(a_i \mymid s) |.
\end{align} 
To prove the continuity, given any $\epsilon>0$, we want to show that there exists $\delta(\epsilon)>0$ such that if 
\begin{align}
d(\pi, \pi') < \delta(\epsilon), \label{eq:assumption-of-the-d}
\end{align} 
then
\begin{align}
\left| f_{i,s}(\pi_i(s), \pi_{-i}(s); V_{i,h+1}) - f_{i,s}(\pi'_i(s), \pi'_{-i}(s); V_{i,h+1}) \right| < \epsilon
\end{align}
for any fixed $\{V_{i,h+1}\}_{i\in[n]}$ with $0\leq V_{i,h+1} \leq H $ for all $i\in [n]$.
Towards this, we observe that
\begin{align}
&\left| f_{i,s}(\pi_i(s), \pi_{-i}(s); V_{i,h+1}) - f_{i,s}(\pi'_i(s), \pi'_{-i}(s); V_{i,h+1}) \right| \notag \\
& = \bigg| \mathbb{E}_{\ba \sim \pi(s) }[r_{i,h}(s,\ba)] + \mathbb{E}_{a_i \sim \pi_{i}(s)} \bigg[ \inf_{\cU^{\sigma_i}\left(P^{\pi_{-i}}_{h,s,a_i} \right)} P V_{i,h+1}\bigg] \notag \\
& \quad - \mathbb{E}_{\ba \sim \pi'(s) }[r_{i,h}(s,\ba)] + \mathbb{E}_{a_i \sim \pi'_{i}(s)} \bigg[ \inf_{\cU^{\sigma_i}\big(P^{\pi'_{-i}}_{h,s,a_i} \big)} P V_{i,h+1}\bigg] \bigg| \notag \\
& \leq \left| \mathbb{E}_{\ba \sim \pi(s) }[r_{i,h}(s,\ba)] -  \mathbb{E}_{\ba \sim \pi'(s) }[r_{i,h}(s,\ba)] \right| \notag \\
& \quad + \bigg|  \mathbb{E}_{a_i \sim \pi_{i}(s)} \bigg[ \inf_{\cU^{\sigma_i}\left(P^{\pi_{-i}}_{h,s,a_i} \right)} P V_{i,h+1}\bigg]  - \mathbb{E}_{a_i \sim \pi'_{i}(s)} \bigg[ \inf_{\cU^{\sigma_i}\big(P^{\pi'_{-i}}_{h,s,a_i} \big)} P V_{i,h+1}\bigg] \bigg|. \label{eq:continous-decompose}
\end{align}

The first term can be bounded by
\begin{align}
&\left| \mathbb{E}_{\ba \sim \pi(s) }[r_{i,h}(s,\ba)] -  \mathbb{E}_{\ba \sim \pi'(s) }[r_{i,h}(s,\ba)] \right| \notag \\
& \leq \sum_{\ba \in \cA } \bigg|\prod_{i\in[n]}\pi_i(a_i\mymid s)- \prod_{i\in[n]}\pi'_i(a_i\mymid s) \bigg| \max_{(s,\ba)\in\cS\times \cA} r_{i,h}(s,\ba)\notag \\
& \leq \sum_{\ba \in \cA } \bigg|\prod_{i\in[n]}\pi_i(a_i\mymid s)- \prod_{i\in[n]}\pi'_i(a_i\mymid s) \bigg|, \label{eq:continuous-1-middle}
\end{align}
where the last inequality holds by the definition of reward function $\max_{(s,\ba)\in\cS\times \cA} r_{i,h}(s,\ba) \leq 1$ for all $(i,h)\in [n]\times [H]$. To continue, we denote the difference $\delta_i(s,a_i) \defn \pi'_i(a_i \mymid s) - \pi_i(a_i \mymid s)$. Therefore, we have
\begin{align}
    \bigg|\prod_{i\in[n]}\pi_i(a_i\mymid s)- \prod_{i\in[n]}\pi'_i(a_i\mymid s) \bigg| &= \bigg|\prod_{i\in[n]}\pi_i(a_i\mymid s)- \prod_{i\in[n]}(\pi_i(a_i\mymid s) + \delta_i(s,a_i)) \bigg| \notag \\
    & = \bigg| \sum_{|\cY| \geq 1, \cY \subseteq [n]} 
    \left( \prod_{i\in \cY} \delta_i(s,a_i)\right)\cdot \left( \prod_{i\in \cY^c} \pi_i(a_i\mymid s)\right)\bigg| \notag \\
    & \leq \sum_{|\cY| \geq 1, \cY \subseteq [n]} \bigg| \left( \prod_{i\in \cY} \delta_i(s,a_i)\right)\cdot \left( \prod_{i\in \cY^c} \pi_i(a_i\mymid s)\right) \bigg| \leq (2^{n}-1) \delta(\epsilon), \label{eq:product-pi-diff}
\end{align}
where the last inequality holds by \eqref{eq:assumption-of-the-d}. Plugging \eqref{eq:product-pi-diff} back to \eqref{eq:continuous-1-middle} indicates that
\begin{align}
   &\left| \mathbb{E}_{\ba \sim \pi(s) }[r_{i,h}(s,\ba)] -  \mathbb{E}_{\ba \sim \pi'(s) }[r_{i,h}(s,\ba)] \right|  \leq \prod_{i\in[n]} A_i (2^{n}-1) \delta(\epsilon). \label{eq:continous-1}
\end{align}

For the second term in \eqref{eq:continous-decompose}, we observe that
\begin{align}
& \bigg|  \mathbb{E}_{a_i \sim \pi_{i}(s)} \bigg[ \inf_{P \in \cU^{\sigma_i}\left(P^{\pi_{-i}}_{h,s,a_i} \right)} P V_{i,h+1}\bigg]  - \mathbb{E}_{a_i \sim \pi'_{i}(s)} \bigg[ \inf_{P \in\cU^{\sigma_i}\big(P^{\pi'_{-i}}_{h,s,a_i} \big)} P V_{i,h+1}\bigg] \bigg| \notag \\
& \leq \bigg|  \mathbb{E}_{a_i \sim \pi_{i}(s)} \bigg[ \inf_{P \in \cU^{\sigma_i}\left(P^{\pi_{-i}}_{h,s,a_i} \right)} P V_{i,h+1}\bigg]  - \mathbb{E}_{a_i \sim \pi_{i}(s)} \bigg[ \inf_{P \in\cU^{\sigma_i}\left(P^{\pi'_{-i}}_{h,s,a_i} \right)} P V_{i,h+1}\bigg] \bigg| \notag \\
& \quad + \bigg| \mathbb{E}_{a_i \sim \pi_{i}(s)} \bigg[ \inf_{P \in\cU^{\sigma_i}\left(P^{\pi'_{-i}}_{h,s,a_i} \right)} P V_{i,h+1}\bigg]  -\mathbb{E}_{a_i \sim \pi'_{i}(s)} \bigg[ \inf_{P \in\cU^{\sigma_i}\big(P^{\pi'_{-i}}_{h,s,a_i} \big)} P V_{i,h+1}\bigg] \bigg| \notag \\
% &=\Big|  \mathbb{E}_{a_i \sim \pi_{i}(s)} \left[  \max_{\alpha\in [\min_s V_{i,h+1}(s), \max_s  V_{i,h+1}(s)]}  \Big\{ \mathbb{E}_{\pi_{-i}( \ba_{-i} \mymid s)} \left[ P^{0}_{h,s, (a_i,\ba_{-i})}\right]\left[ V_{i,h+1}\right]_{\alpha} - \ror_i \left(\alpha - \min_{s'}\left[ V_{i,h+1}\right]_{\alpha}(s') \right)  \right] \notag \\
% & \quad - \mathbb{E}_{a_i \sim \pi'_{i}(s)} \left[  \max_{\alpha\in [\min_s V_{i,h+1}(s), \max_s  V_{i,h+1}(s)]}  \Big\{ \mathbb{E}_{\pi'_{-i}( \ba_{-i} \mymid s)} \left[ P^{0}_{h,s, (a_i,\ba_{-i})}\right]\left[ V_{i,h+1}\right]_{\alpha} - \ror_i \left(\alpha - \min_{s'}\left[ V_{i,h+1}\right]_{\alpha}(s') \right)  \right] \Big| \notag\\
& \overset{\mathrm{(i)}}{\leq} \mathbb{E}_{a_i \sim \pi_{i}(s)} \bigg[\max_{\alpha\in [\min_s V_{i,h+1}(s), \max_s  V_{i,h+1}(s)]}   \bigg| \mathbb{E}_{\pi_{-i}( \ba_{-i} \mymid s)} \left[ P^{0}_{h,s, (a_i,\ba_{-i})}\right]\left[ V_{i,h+1}\right]_{\alpha} \notag \\
& \quad - \mathbb{E}_{\pi'_{-i}( \ba_{-i} \mymid s)} \left[ P^{0}_{h,s, (a_i,\ba_{-i})}\right]\left[ V_{i,h+1}\right]_{\alpha}\bigg| \bigg] + \sum_{a_i\in\cA_i} \big|\pi_i'(a_i \mymid s) - \pi_i(a_i \mymid s) \big| \inf_{P \in\cU^{\sigma_i}\big(P^{\pi'_{-i}}_{h,s,a_i} \big)} P V_{i,h+1} \notag\\
&  \overset{\mathrm{(ii)}}{\leq} \sum_{\ba_{-i}\in \cA_i} \bigg|\prod_{j \neq i}\pi_j(a_j\mymid s)- \prod_{j \neq i}\pi'_j(a_j\mymid s) \bigg| H + H A_i \delta(\epsilon) \notag \\
&  \overset{\mathrm{(iii)}}{\leq} H\prod_{j\neq i, j\in[n]} A_j (2^{n-1}-1)\delta(\epsilon) + H A_i \delta(\epsilon) \leq 2H\prod_{i\in[n]} A_i (2^n-1) \cdot \delta(\epsilon),\label{eq:continous-2} \end{align}
where the first inequality holds by the triangle inequality, and (i) follows from applying the dual form of the optimization problem associated with TV distance \citep[Lemma 4]{shi2023curious}
% \begin{align}\label{eq:nvi-iteration-dual} 
% \inf_{\cP \in U^{\ror_i}(P)} \cP V=   \max_{\alpha\in [\min_s V(s), \max_s V(s)]}  \Big\{ P\left[V\right]_{\alpha} - \ror_i \left(\alpha - \min_{s'}\left[V\right]_{\alpha}(s') \right) \Big\} , 
% \end{align} 
and the maximum operator is $1$-Lipschitz, (ii) arises from the fact that $\|V_{i,h+1} \|_\infty \leq H$, and (iii) can be verified by following the same pipeline of \eqref{eq:product-pi-diff}.
Combining \eqref{eq:continous-1} and \eqref{eq:continous-2}, one has
\begin{align}
\left| f_{i,s}(\pi_i(s), \pi_{-i}(s); V_{i,h+1}) - f_{i,s}(\pi'_i(s), \pi'_{-i}(s); V_{i,h+1}) \right| \leq 3H\prod_{i\in[n]} A_i  (2^n-1)\cdot \delta(\varepsilon).
\end{align}
Consequently, letting $\delta_1(\epsilon) = \frac{\min\{ \epsilon,1 \}}{3H \prod_{i\in[n]} A_i  (2^n-1)}$, we complete the proof by showing that when $d(\pi,\pi') < \delta_1(\epsilon)$, $$\left| f_{i,s}(\pi_i(s), \pi_{-i}(s); V_{i,h+1}) - f_{i,s}(\pi'_i(s), \pi'_{-i}(s); V_{i,h+1}) \right| < \epsilon.$$

\subsubsection{Proof of Lemma~\ref{lemma:continuous-g}}\label{proof:lemma:continuous-g}

Without loss of generality, we consider any $i\in[n]$, and two policy $x_{-i}: \cS \mapsto \prod_{j\neq i, j\in[n]} \Delta(\cA_j)$ and $y_{-i}: \cS \mapsto \prod_{j\neq i, j\in[n]} \Delta(\cA_j)$. Before continuing, for all $s\in\cS$, we denote
\begin{align}
   u_{i,s}^\star &\defn \mathrm{argmax}_{\pi_i(s)'\in\Delta(\cA_i)} \; f_{i,s}(\pi_i'(s), x_{-i}(s); V_{i,h+1}) , \notag \\
   v_{i,s}^\star &\defn \mathrm{argmax}_{\pi_i(s)'\in\Delta(\cA_i)} \; f_{i,s}(\pi_i'(s), y_{-i}(s); V_{i,h+1}).
\end{align}
Then we have for any $s\in\cS$, recalling the introduced function $g(\cdot)$ in \eqref{defn-g-auxiliary},
\begin{align}
& g_{i,s}(x_{-i}(s),V_{i,h+1}) - g_{i,s}(y_{-i}(s),V_{i,h+1}) \notag \\
&= \mathrm{max}_{\pi_i'(s)\in\Delta(\cA_i)} \; f_{i,s}(\pi_i'(s), x_{-i}(s); V_{i,h+1}) - \mathrm{max}_{\pi_i'(s)\in\Delta(\cA_i)} \; f_{i,s}(\pi_i'(s), y_{-i}(s); V_{i,h+1}) \notag \\
& = f_{i,s}(u_{i,s}^\star, x_{-i}(s); V_{i,h+1}) - f_{i,s}(v_{i,s}^\star, y_{-i}(s); V_{i,h+1}) \notag \\
& \leq f_{i,s}(u_{i,s}^\star, x_{-i}(s); V_{i,h+1}) - f_{i,s}(u_{i,s}^\star, y_{-i}(s); V_{i,h+1}) \quad \rightarrow 0 \quad \text{as} \quad y_{-i}(s) \rightarrow x_{-i}(s),
% & \overset{\mathrm{(i)}}{=} \mathrm{max}_{\pi_i'\in\Delta(\cS)} \mathbb{E}_{a_i\sim \pi_i'(s)} \bigg[ \mathbb{E}_{\ba_i\sim x_{-i}(s)}[r_{i,h}(s,\ba)] + \inf_{\cU^{\sigma_i}\left(P^{x_{-i}}_{h,s,a_i} \right)} P V_{i,h+1} \bigg] \notag \\
% & \quad - \mathrm{max}_{\pi_i'\in\Delta(\cS)} \mathbb{E}_{a_i\sim \pi_i'(s)} \bigg[ \mathbb{E}_{\ba_i\sim y_{-i}(s)}[r_{i,h}(s,\ba)] + \inf_{\cU^{\sigma_i}\left(P^{y_{-i}}_{h,s,a_i} \right)} P V_{i,h+1} \bigg] 
\end{align}
where the last line holds by Lemma~\ref{lemma:continuous-f} which shows that the function $f_{i,s}$ is continuous.
% where (i) follows from the definition of $f_{i,s}$ in  \eqref{eq:lemma-continuous-2-payoff}
Similarly, one has
\begin{align}
& g_{i,s}(x_{-i}(s),V_{i,h+1}) - g_{i,s}(y_{-i}(s),V_{i,h+1}) \notag \\
& \geq f_{i,s}(v_{i,s}^\star, x_{-i}(s); V_{i,h+1}) - f_{i,s}(v_{i,s}^\star, y_{-i}(s); V_{i,h+1}) \quad \rightarrow 0 \quad \text{as} \quad y_{-i}(s) \rightarrow x_{-i}(s).
\end{align}
We complete the proof by showing that 
\begin{align}
    |g_{i,s}(x_{-i}(s),V_{i,h+1}) - g_{i,s}(y_{-i}(s),V_{i,h+1})| \rightarrow 0 \quad \text{as} \quad y_{-i}(s) \rightarrow x_{-i}(s),
\end{align}
which indicates that $g_{i,s}(x_{-i}(s),V_{i,h+1})$ is continuous w.r.t. $x_{-i}(s)$ for all $0 \leq V_{i,h+1} \leq H$.

% \subsubsection{Property of the CCE definition of standard MARL}\label{proof:equivalence-cce-MARL}

% \yc{??}

%% file: sub-files-for-arxiv/appendix_sample_complexity.tex
%!TEX root = ./../DRO-MARL-near-optimal.tex

\section{Proof of Theorem~\ref{thm:robust-mg-upper-bound}}

We will present the proof of Theorem~\ref{thm:robust-mg-upper-bound} by first outlining the key proof structure in a step-by-step manner. More detailed proof are provided at the end of this section.

\subsection{Proof pipeline}

To proof Theorem~\ref{thm:robust-mg-upper-bound}, recall the goal is to show that
\begin{align}
\forall (i,s) \in[n] \times \cS: \quad   \mathbb{E}_{\pi\sim \widehat{\xi}} \left[V_{i,1}^{\star,\pi_{-i}, \ror_i }(s)\right] - \mathbb{E}_{\pi\sim \widehat{\xi}}\left[V_{i,1}^{\pi,\ror_i}(s)\right]  \leq \varepsilon,
\end{align}
where $\widehat{\xi} = \{\widehat{\xi}_h\}_{h\in[H]}$ is the output distribution over the set of policies $\{ \pi_h^k = (\pi_{1,h}^k \times \cdots \times \pi_{n,h}^k) \}_{k\in[K], h\in[H]}$ from Algorithm~\ref{alg:summary}. Namely, $\pi\sim \widehat{\xi}$ means
\begin{align}
    \forall h\in[H]: \quad \pi_h \sim \widehat{\xi}_h, \quad \text{ where } \quad  \widehat{\xi}_h( \pi_h^k) = \alpha_k^K.
\end{align}

Before proceeding, we first introduce the notation of the best-response policy for player \( i \):
\begin{align*}
    \tilde{\pi}_i^\star = [\tilde{\pi}_{i,h}^\star]_{h \in [H]} := \mathrm{argmax}_{\pi_i^\prime:\mathcal{S} \times [H] \to \Delta(\mathcal{A}_i)} \mathbb{E}_{\pi\sim \widehat{\xi}}\left[V_{i,1}^{\pi_i^\prime,\pi_{-i},\ror_i}\right].
\end{align*}
According to the definition and robust Bellman equation in \eqref{eq:robust-bellman-equation}, it is easily verified that $\mathbb{E}_{\pi\sim \widehat{\xi}}\left[V_{i,h}^{\pi,\ror_i}\right]$ satisfies: for all \( (i, s, h) \in [n] \times \mathcal{S} \times [H] \),
\begin{align*}
    \mathbb{E}_{\pi\sim \widehat{\xi}}\left[V_{i,H+1}^{\pi,\ror_i}(s)\right] &= 0
    \end{align*}
    and 
\begin{align*}
    \mathbb{E}_{\pi\sim \widehat{\xi}}\left[V_{i,h}^{\pi,\ror_i}(s)\right]& =\mathbb{E}_{\pi\sim \widehat{\xi}}\Bigg\{\sum_{\ba \in \mathcal{A}}  \pi_h(\ba \mid s) r_{i,h}(s, \ba)  +  \mathbb{E}_{a_i \sim \pi_{i,h}(s)} \Bigg[ \inf_{\mathcal{P} \in \mathcal{U}_i^{\sigma_i}\left(P_{h,s,a_i}^{\pi_{-i}}\right)} \mathcal{P} \mathbb{E}_{\pi\sim \widehat{\xi}}\left[V_{i,h+1}^{\pi,\ror_i}\right] \Bigg]\Bigg\} \notag \\
    &=\sum_{k=1}^K \sum_{\ba \in \mathcal{A}} \alpha_k^K \pi_h^k(\ba \mid s) r_{i,h}(s, \ba)  + \sum_{k=1}^K \alpha_k^K \mathbb{E}_{a_i \sim \pi_{i,h}^k(s)} \Bigg[ \inf_{\mathcal{P} \in \mathcal{U}_i^{\sigma_i}\left(P_{h,s,a_i}^{\pi_{-i}^k}\right)} \mathcal{P} \mathbb{E}_{\pi\sim \widehat{\xi}}\left[V_{i,h+1}^{\pi,\ror_i}\right] \Bigg],
\end{align*}
where \( P_{h,s,a_i}^{\pi_{-i}^k} \) is defined as:
\begin{align*}
    P_{h,s,a_i}^{\pi_{-i}^k} = \mathbb{E}_{\ba_{-i}\sim\pi_{-i,h}^k(s)} \left[ P_{h,s,(a_i,\ba_{-i})}^0 \right] = \sum_{\ba_{-i} \in \mathcal{A}_{-i}} \pi_{-i,h}^k(\ba_{-i} \mid s) \left[ P_{h,s,(a_i,\ba_{-i})}^0 \right].
\end{align*}

Before continuing, we introduce some auxiliary notations for value functions that will be useful throughout the proof. For any $(h,s)\in [H] \times \mathcal{S}$:
\begin{subequations}
\label{eq:auxiliary_value_function_3}
\begin{align}
    &\mathbb{E}_{\pi\sim \widehat{\xi}}\left[\overline{V}^{\pi,\ror_i}_{i,h}(s)\right] = \sum_{k=1}^K \alpha_k^K \mathbb{E}_{a_i \sim \pi_{i,h}^k(s)}\left[r_{i,h}^k(s,a_i)\right]  + \sum_{k=1}^K \alpha_k^K \mathbb{E}_{a_i \sim \pi_{i,h}^k(s)} \left[ \inf_{\mathcal{P} \in \mathcal{U}^{\sigma_i}\left(P_{i,h,s,a_i}^k\right)} \mathcal{P} \mathbb{E}_{\pi\sim \widehat{\xi}}\left[\overline{V}^{\pi,\ror_i}_{i,h+1}\right] \right],\\
    &\mathbb{E}_{\pi\sim \widehat{\xi}}\left[\overline{V}^{\tilde{\pi}_i^\star,\pi_{-i},\ror_i}_{i,h}(s) \right]= \sum_{k=1}^K \alpha_k^K \mathbb{E}_{a_i \sim \tilde{\pi}_{i,h}^\star(s)}\left[r_{i,h}^k(s,a_i)\right] + \sum_{k=1}^K \alpha_k^K \mathbb{E}_{a_i \sim \tilde{\pi}_{i,h}^\star(s)} \left[ \inf_{\mathcal{P} \in \mathcal{U}^{\sigma_i}\left(P_{i,h,s,a_i}^k\right)} \mathcal{P} \mathbb{E}_{\pi\sim \widehat{\xi}}\left[\overline{V}^{\tilde{\pi}_i^\star,\pi_{-i}, \ror_i}_{i,h+1} \right] \right],\\
    &\mathbb{E}_{\pi\sim \widehat{\xi}}\left[\overline{V}^{\star,\pi_{-i}, \ror_i}_{i,h}(s)\right] = \max_{a_i \in \mathcal{A}_i} \sum_{k=1}^K \alpha_k^K \left[ r_{i,h}^k(s,a_i) + \left( \inf_{\mathcal{P} \in \mathcal{U}^{\sigma_i}\left(P_{i,h,s,a_i}^k\right)} \mathcal{P} \mathbb{E}_{\pi\sim \widehat{\xi}}\left[\overline{V}^{\star,\pi_{-i}, \ror_i}_{i,h+1}\right] \right) \right],
\end{align}
where recalling $P_{i,h,s,a_i}^k$ is the $(s,a_i)$-th row of the estimated transition kernel $P_{i,h}^k$ at the k-th iteration of Algorithm~\ref{alg:summary} at time step $h$.
\end{subequations}
In addition, for any policy $\pi: \cS\times [H] \mapsto \Delta(\cA)$, we define
\begin{align}
    \forall s\in\cS: \quad \mathbb{E}_{\pi\sim \widehat{\xi}}\left[ \overline{V}^{\pi,\ror_i}_{i,H+1}(s) \right]= 0. \label{eq:auxiliary_value_function_4}
\end{align}

Armed with above definitions, we now decompose the error to be controlled as follows:
\begin{equation}
\label{eq:error_decomposition}
\begin{aligned}
    &\mathbb{E}_{\pi\sim \widehat{\xi}}\left[V_{i,h}^{\star,\pi_{-i},\ror_i}\right] -\mathbb{E}_{\pi\sim \widehat{\xi}} \left[V_{i,h}^{\pi, \ror_i}\right]\\
     &\leq \underbrace{\mathbb{E}_{\pi\sim \widehat{\xi}}\left[V_{i,h}^{\star,\pi_{-i},\ror_i}\right] - \mathbb{E}_{\pi\sim \widehat{\xi}}\left[\overline{V}_{i,h}^{\tilde{\pi}_i^\star,\pi_{-i},\ror_i}\right]}_{A} + \underbrace{\mathbb{E}_{\pi\sim \widehat{\xi}}\left[ \overline{V}_{i,h}^{\star,\pi_{-i},\ror_i}\right] -\mathbb{E}_{\pi\sim \widehat{\xi}}\left[ \overline{V}_{i,h}^{\pi,\ror_i} \right]}_{B} + \underbrace{\mathbb{E}_{\pi\sim \widehat{\xi}}\left[\overline{V}_{i,h}^{\pi,\ror_i}\right] - \mathbb{E}_{\pi\sim \widehat{\xi}}\left[V_{i,h}^{\pi,\ror_i}\right]}_{C}.
\end{aligned}
\end{equation}
where we use the fact that \( \mathbb{E}_{\pi\sim \widehat{\xi}}\left[\overline{V}^{\star,\pi_{-i},\ror_i}_{i,h}\right] \geq \mathbb{E}_{\pi\sim \widehat{\xi}}\left[\overline{V}^{\tilde{\pi}_i^\star,\pi_{-i},\ror_i}_{i,h}(s)\right] \), which can be directly verified by recursively applying the definitions in \eqref{eq:auxiliary_value_function_3}.
In the next step, we shall control the above three terms, $A$, $B$, and $C$ separately in the following sections.

\subsection{Controlling B: using adversarial online learning}
\label{sec:controlling_B}

\paragraph{Step 1: $\widehat{V}_{i,h}$ serves as an optimistic estimate.} 
Armed with the output from Algorithm~\ref{alg:summary}, we first introduce the following lemma demonstrating that the output $\widehat{V}_{i,h}$  serves as an optimistic estimate of the auxiliary value $\mathbb{E}_{\pi\sim \widehat{\xi}}\left[\overline{V}^{\star,\pi_{-i},\ror_i}_{i,h}\right]$ (see \eqref{eq:auxiliary_value_function_3}).
\begin{lemma}
    \label{lem:UCB}
    It holds that when $K \geq c_9 H^2 \log^4 K$, one has 
    \begin{align*}
   \forall (i,h)\in[n]\times[H]: \quad  \widehat{V}_{i,h}\geq \mathbb{E}_{\pi\sim \widehat{\xi}}\left[\overline{V}^{\star,\pi_{-i},\ror_i}_{i,h}\right].
    \end{align*}
\end{lemma}
\begin{proof}
    See Appendix \ref{sec:proof-lemma:UCB}
\end{proof}
In addition, the following lemma demonstrates that the estimate $\widehat{V}_{i,h}$ is also an optimistic estimate on $\mathbb{E}_{\pi\sim \widehat{\xi}}\left[\overline{V}^{\pi,\ror_i}_{i,h}\right]$ (cf.~\eqref{eq:auxiliary_value_function_3}). 
\begin{lemma}
    \label{lm:term_B_value_vector_comparison}
    It holds that 
    \begin{align*}
       \forall (i,h)\in[n]\times[H]:\quad  \widehat{V}_{i,h}\geq \mathbb{E}_{\pi\sim \widehat{\xi}}\left[\overline{V}^{\pi,\ror_i}_{i,h}\right].
    \end{align*}
\end{lemma}
\begin{proof}
    See Appendix~\ref{proof:lm:term_B_value_vector_comparison}
\end{proof}

Armed with Lemma~\ref{lem:UCB}, we notice that
\begin{align}
B &= \mathbb{E}_{\pi\sim \widehat{\xi}}\left[ \overline{V}_{i,h}^{\star,\pi_{-i},\ror_i}\right] -\mathbb{E}_{\pi\sim \widehat{\xi}}\left[ \overline{V}_{i,h}^{\pi,\ror_i} \right] \widehat{V}_{i,h}(s)-\mathbb{E}_{\pi\sim \widehat{\xi}}\left[\overline{V}_{i,h}^{\pi,\ror_i}(s)\right]. \label{eq:control-B-change-goal}
\end{align}

\paragraph{Step 2: constructing the recursion.}
To further bound \eqref{eq:control-B-change-goal}, according to the definition of $\widehat{V}_{i,h}(s)$ in Algorithm~\ref{alg:summary} and $\mathbb{E}_{\pi\sim \widehat{\xi}}\left[\overline{V}_{i,h}^{\pi,\ror_i}(s)\right]$, we have
\begin{align}
    &\widehat{V}_{i,h}(s)-\mathbb{E}_{\pi\sim \widehat{\xi}}\left[\overline{V}_{i,h}^{\pi,\ror_i}(s)\right]\notag\\
    &=\min\left\{\sum_{k=1}^K\alpha_k^K\mathbb{E}_{a_i\sim \pi_{i,h}^k(s)}\left[r_{i,h}^k(s,a_i)+\inf_{\mathcal{P}\in\mathcal{U}^{\sigma_i}\left(P_{i,h,s,a_i}^k\right)}\mathcal{P}\widehat{V}_{i,h+1} \right]+\beta_{i,h}(s),H-h+1\right\}\notag\\
    &\quad -\sum_{k=1}^K\alpha_k^K\mathbb{E}_{a_i\sim \pi_{i,h}^k(s)}\left[r_{i,h}^k(s,a_i)+\inf_{\mathcal{P}\in\mathcal{U}^{\sigma_i}\left(P_{i,h,s,a_i}^k\right)}\mathcal{P}\mathbb{E}_{\pi\sim \widehat{\xi}}\left[\overline{V}_{i,h+1}^{\pi,\ror_i}\right] \right]\notag\\
   &\leq\sum_{k=1}^K\alpha_k^K\mathbb{E}_{a_i\sim \pi_{i,h}^k(s)}\left[r_{i,h}^k(s,a_i)+\inf_{\mathcal{P}\in\mathcal{U}^{\sigma_i}\left(P_{i,h,s,a_i}^k\right)}\mathcal{P}\widehat{V}_{i,h+1} \right]+\beta_{i,h}(s)\notag\\
   &\quad -\sum_{k=1}^K\alpha_k^K\mathbb{E}_{a_i\sim \pi_{i,h}^k(s)}\left[r_{i,h}^k(s,a_i)+\inf_{\mathcal{P}\in\mathcal{U}^{\sigma_i}\left(P_{i,h,s,a_i}^k\right)}\mathcal{P}\mathbb{E}_{\pi\sim \widehat{\xi}}\left[\overline{V}_{i,h+1}^{\pi,\ror_i}\right] \right]\notag\\
    &=\sum_{k=1}^K\alpha_k^K\mathbb{E}_{a_i\sim \pi_{i,h}^k(s)}\left[\inf_{\mathcal{P}\in\mathcal{U}^{\sigma_i}\left(P_{i,h,s,a_i}^k\right)}\mathcal{P}\widehat{V}_{i,h+1} \right]+\beta_{i,h}(s)-\sum_{k=1}^K\alpha_k^K\mathbb{E}_{a_i\sim \pi_{i,h}^k}\left[\inf_{\mathcal{P}\in\mathcal{U}^{\sigma_i}\left(P_{i,h,s,a_i}^k\right)}\mathcal{P}\mathbb{E}_{\pi\sim \widehat{\xi}}\left[\overline{V}_{i,h+1}^{\pi,\ror_i}\right]\right]\label{eq:term_B_recursion_eq1}
\end{align}
To simplify the notation, we define the transition kernel associated with value function estimates in a manner similar to  \eqref{eq:preliminary-P-hat-k}. For any time step $h$ and iteration $k\in[K]$, we define matrix notations $\widehat{P}_{i,h}^{\pi^k,\widehat{V}}  \in \mathbb{R}^{S A_i \times S}$ and $\widehat{P}_{i,h}^{\widehat{\pi}^k,\overline{V}}  \in \mathbb{R}^{S A_i \times S}$ as: 
\begin{align}
   &P_{i,h}^{k,\widehat{V}} := P_{i,h}^{k,\widehat{V}_{i,h+1}}, \quad \text{where} \quad P_{i,h,s,a_i}^{k, \widehat{V}} := P_{i,h,s,a_i}^{k,\widehat{V}_{i,h+1}} = \mathrm{argmin}_{\mathcal{P} \in \mathcal{U}^{\sigma_i}\left(P_{i,h,s,a_i}^{k}\right)} \mathcal{P} \widehat{V}_{i,h+1}, \notag \\
    &P_{i,h}^{k, \widehat{\xi},\overline{V}} := P_{i,h}^{k, \mathbb{E}_{\pi\sim \widehat{\xi}}\left[\overline{V}_{i,h+1}^{\pi,\ror_i}\right]}, \quad \text{where} \quad P_{i,h,s,a_i}^{k, \widehat{\xi}, \overline{V}} :=  P_{i,h,s,a_i}^{k, \mathbb{E}_{\pi\sim \widehat{\xi}}\left[\overline{V}_{i,h+1}^{\pi,\ror_i}\right]} = \mathrm{argmin}_{\mathcal{P} \in \mathcal{U}^{\sigma_i}\left(P_{i,h,s,a_i}^{k}\right)} \mathcal{P}\mathbb{E}_{\pi\sim \widehat{\xi}}\left[\overline{V}_{i,h+1}^{\pi,\ror_i}\right]. \label{eq:extra-matrix-def}
\end{align}
Additionally, we define square matrices \( \underline{P}_{i,h}^{k,\widehat{V}}  \in \mathbb{R}^{S \times S} \) and $\underline{P}_{i,h}^{k, \widehat{\xi}, \overline{V}} \in \mathbb{R}^{S \times S}$ as
\begin{align}
\underline{P}_{i,h}^{k,\widehat{V}} := \Pi_h^{\pi_i^k} P_{i,h}^{k, \widehat{V}} \quad  \text{and} \quad  \underline{P}_{i,h}^{k, \widehat{\xi}, \overline{V}} := \Pi_h^{\pi_i^k} P_{i,h}^{k, \widehat{\xi},\overline{V}} \label{eq:extra-matrix-def2}
\end{align}
Then we rewrite \eqref{eq:term_B_recursion_eq1} in a vector form as below:
\begin{align}
    &\widehat{V}_{i,h}-\mathbb{E}_{\pi\sim \widehat{\xi}}\left[\overline{V}_{i,h}^{\pi,\ror_i}\right]\\
    &\leq\sum_{k=1}^K\alpha_k^K\Pi_{h}^{\pi_i}\left[\inf_{\mathcal{P}\in\mathcal{U}^{\sigma_i}\left(P_{i,h,s,a_i}^{k}\right)}\mathcal{P}\widehat{V}_{i,h+1}\right]+\beta_{i,h}-\sum_{k=1}^K\alpha_k^K\Pi_{h}^{\pi_i}\left[\inf_{\mathcal{P}\in\mathcal{U}^{\sigma_i}\left(P_{i,h,s,a_i}^{k}\right)}\mathcal{P}\overline{V}_{i,h+1}^{\pi,\ror_i}\right] \notag \\
    &\overset{\mathrm{(i)}}{=} \sum_{k=1}^K\alpha_k^K \underline{P}_{i,h}^{k,\widehat{V}}\widehat{V}_{i,h+1}+\beta_{i,h}-\sum_{k=1}^K\alpha_k^KP_{i,h}^{k, \widehat{\xi},\overline{V}}\mathbb{E}_{\pi\sim \widehat{\xi}}\left[\overline{V}_{i,h+1}^{\pi,\ror_i}\right] \notag\\
   & \overset{\mathrm{(ii)}}{\leq} \sum_{k=1}^K\alpha_k^K\underline{P}_{i,h}^{k, \widehat{\xi}, \overline{V}}\left(\widehat{V}_{i,h+1}-\mathbb{E}_{\pi\sim \widehat{\xi}}\left[\overline{V}_{i,h+1}^{\pi,\ror_i}\right]\right)+\beta_{i,h} \notag\\
   & \leq \left( \sum_{k=1}^K\alpha_k^K\underline{P}_{i,h}^{k, \widehat{\xi}, \overline{V}} \right) \left( \sum_{k=1}^K\alpha_k^K\underline{P}_{i,h+1}^{k,\widehat{\xi},\overline{V}} \right) \left(\widehat{V}_{i,h+2}-\mathbb{E}_{\pi\sim \widehat{\xi}}\left[\overline{V}_{i,h+2}^{\pi,\ror_i}\right]\right)+\beta_{i,h} + \left( \sum_{k=1}^K\alpha_k^K\underline{P}_{i,h}^{k, \widehat{\xi}, \overline{V}} \right)\beta_{i,h+1}\notag\\
   &\leq ... \leq \beta_{i,h} + \sum_{j=h+1}^H\left[\prod_{r=h}^{j-1}\left(\sum_{k=1}^K\alpha_k^K\underline{P}_{i,r}^{k,\widehat{\xi},\overline{V}}\right)\right] \beta_{i,j}, \label{eq:recursively-beta}
\end{align}
where (i) holds by the definitions in \eqref{eq:extra-matrix-def}, and the last inequality holds by recursively applying (ii).

To continue, we first introduce an useful upper bound of the bonus term $\beta_{i,h}$.
\begin{lemma}
\label{lm:upper_bound_bonus}
The bonus vector $\beta_{i,h}$ is bounded by the following inequality:
    \begin{align*}
        \beta_{i,h}\leq 3c_{\mathsf{b}}\sqrt{\frac{\log^3(\frac{KS\sum_{i=1}^nA_i}{\delta})}{KH}}\left(H \cdot 1+\sum_{k=1}^K\alpha_k^K\mathsf{Var}_{\underline{P}_{i,h}^{k,\widehat{V}}}\widehat{V}_{i,h+1}\right)
    \end{align*}
\end{lemma}
\begin{proof}
    See Appendix~\ref{proof:lm:upper_bound_bonus} 
\end{proof}

Then we introduce the following notations for convenience. Let $e_s$ denote the $S$-dimensional standard basis vector, with support on the $s$-th element. Additionally, we denote
\begin{align}
    b_{h}^h = e_s \quad \text{and} \quad \big(b_h^j \big)^\top = e_s^\top \left[\prod_{r=h}^{j-1}\left(\sum_{k=1}^K\alpha_k^K\underline{P}_{i,r}^{k,\widehat{\xi},\overline{V}}\right)\right], \quad \forall j = h+1, \dots, H, \label{eq:defn-of-d}
\end{align}
which obey 
\begin{align}
b_{j}^h \cdot 1 = 1 , \quad \forall j= h, h+ 1, \cdots, H. \label{eq:property-bh}
\end{align}
Combining the above notations and \eqref{eq:recursively-beta}, we have for any $s\in\mathcal{S}$, 
\begin{align}
    &\widehat{V}_{i,h}(s)-\mathbb{E}_{\pi\sim \widehat{\xi}}\left[\overline{V}^{\pi,\ror_i}_{i,h}(s)\right]=\left<e_s,\widehat{V}_{i,h}-\mathbb{E}_{\pi\sim \widehat{\xi}}\left[\overline{V}^{\pi,\ror_i}_{i,h}\right]\right> \leq \sum_{j=h}^H\left<b_h^j,\beta_{i,j}\right>\notag\\
   &\leq\sum_{j=h}^H\left<b_h^j,3c_{\mathsf{b}}H\sqrt{\frac{\log^3(\frac{KS\sum_{i=1}^nA_i}{\delta})}{KH}}1\right>+\sum_{j=h}^H\sum_{k=1}^K\alpha_k^K\left<b_h^j,3c_{\mathsf{b}}\sqrt{\frac{\log^3(\frac{KS\sum_{i=1}^nA_i}{\delta})}{KH}}\mathsf{Var}_{\underline{P}_{i,j}^{k,\widehat{V}}}\widehat{V}_{i,j+1}\right>\notag\\
    &=3c_{\mathsf{b}}\sqrt{\frac{H^3\log^3(\frac{KS\sum_{i=1}^nA_i}{\delta})}{K}}+3c_{\mathsf{b}}\sqrt{\frac{\log^3(\frac{KS\sum_{i=1}^nA_i}{\delta})}{KH}}\sum_{j=h}^H\sum_{k=1}^K\alpha_k^K\left<b_h^j,\mathsf{Var}_{\underline{P}_{i,j}^{k,\widehat{V}}}\widehat{V}_{i,j+1}\right>.\label{eq:term_B_error_decomposion_1}
\end{align}
where the first inequality holds by Lemma~\ref{lm:upper_bound_bonus}.
We further decompose \eqref{eq:term_B_error_decomposion_1} as below:
\begin{align}
    &\widehat{V}_{i,h}(s)-\mathbb{E}_{\pi\sim \widehat{\xi}}\left[\overline{V}^{\pi,\ror_i}_{i,h}(s)\right] \notag \\
   &\leq 3c_{\mathsf{b}}\sqrt{\frac{H^3\log^3(\frac{KS\sum_{i=1}^nA_i}{\delta})}{K}}+3c_{\mathsf{b}}\sqrt{\frac{\log^3(\frac{KS\sum_{i=1}^nA_i}{\delta})}{KH}}\sum_{j=h}^H\sum_{k=1}^K\alpha_k^K \Big<b_h^j,\mathsf{Var}_{\underline{P}_{i,j}^{k,\widehat{V}}}\widehat{V}_{i,j+1}\Big> \notag \\
   &\leq3c_{\mathsf{b}}\sqrt{\frac{H^3\log^3(\frac{KS\sum_{i=1}^nA_i}{\delta})}{K}}+6c_{\mathsf{b}}\sqrt{\frac{\log^3(\frac{KS\sum_{i=1}^nA_i}{\delta})}{KH}}\sum_{k=1}^K\alpha_k^K\sum_{j=h}^H\Big<b_h^j,\mathsf{Var}_{\underline{P}_{i,j}^{k,\widehat{V}}}\left(\widehat{V}_{i,j+1}-\mathbb{E}_{\pi\sim \widehat{\xi}}\left[\overline{V}_{i,j+1}^{\pi,\ror_i}\right]\right)\Big> \notag \\
    &\quad +6c_{\mathsf{b}}\sqrt{\frac{\log^3(\frac{KS\sum_{i=1}^nA_i}{\delta})}{KH}}\sum_{k=1}^K\alpha_k^K\sum_{j=h}^H\Big<b_h^j,\mathsf{Var}_{\underline{P}_{i,j}^{k,\widehat{V}}}\left(\mathbb{E}_{\pi\sim \widehat{\xi}}\left[\overline{V}_{i,j+1}^{\pi,\ror_i}\right]\right)\Big>\\
   & = 3c_{\mathsf{b}}\sqrt{\frac{H^3\log^3(\frac{KS\sum_{i=1}^nA_i}{\delta})}{K}}+\underbrace{6c_{\mathsf{b}}\sqrt{\frac{\log^3(\frac{KS\sum_{i=1}^nA_i}{\delta})}{KH}}\sum_{j=h}^H\sum_{k=1}^K\alpha_k^K\Big<b_h^j,\mathsf{Var}_{\underline{P}_{i,j}^{k,\widehat{V}}}\left(\widehat{V}_{i,j+1}-\mathbb{E}_{\pi\sim \widehat{\xi}}\left[\overline{V}_{i,j+1}^{\pi,\ror_i}\right]\right)\Big>}_{\mathcal{D}_1}\notag \\
    &\quad +\underbrace{6c_{\mathsf{b}}\sqrt{\frac{\log^3(\frac{KS\sum_{i=1}^nA_i}{\delta})}{KH}}\sum_{j=h}^H\sum_{k=1}^K\alpha_k^K\Big<b_h^j,\mathsf{Var}_{\underline{P}_{i,j}^{k,\widehat{V}}}\left(\mathbb{E}_{\pi\sim \widehat{\xi}}\left[\overline{V}_{i,j+1}^{\pi,\ror_i}\right]\right) - \mathsf{Var}_{\underline{P}_{i,j}^{k,\widehat{\xi},\overline{V}}}\left(\mathbb{E}_{\pi\sim \widehat{\xi}}\left[\overline{V}_{i,j+1}^{\pi,\ror_i}\right]\right)\Big>}_{\mathcal{D}_2}\notag \\
    &\quad +\underbrace{6c_{\mathsf{b}}\sqrt{\frac{\log^3(\frac{KS\sum_{i=1}^nA_i}{\delta})}{KH}}\sum_{j=h}^H\sum_{k=1}^K\alpha_k^K\Big<b_h^j,\mathsf{Var}_{\underline{P}_{i,j}^{k,\widehat{\xi},\overline{V}}}\left(\mathbb{E}_{\pi\sim \widehat{\xi}}\left[\overline{V}_{i,j+1}^{\pi,\ror_i}\right]\right)\Big>}_{\mathcal{D}_3}, \label{eq:splitting-up-D123}
\end{align}
where the second inequality is due to the elementary inequality $\mathsf{Var}_P(V+V^\prime) \leq2 \left(\mathsf{Var}_P(V)+ \mathsf{Var}_P(V^\prime) \right)$ for any transition kernel $P\in\mathbb{R}^S$ and vector $V,V^\prime\in\mathbb{R}^S$.

\paragraph{Step 3: controlling $\mathcal{D}_1,\mathcal{D}_2,\mathcal{D}_3$ separately.}

First, we focus on  $\mathcal{D}_1$ and directly obtain the following upper bound:
\begin{align}
    \mathcal{D}_1&=6c_{\mathsf{b}}\sqrt{\frac{\log^3(\frac{KS\sum_{i=1}^nA_i}{\delta})}{KH}}\sum_{j=h}^H\sum_{k=1}^K\alpha_k^K \Big<b_h^j,\mathsf{Var}_{\underline{P}_{i,h}^{k,\widehat{V}}}\left(\widehat{V}_{i,j+1}-\mathbb{E}_{\pi\sim \widehat{\xi}}\left[\overline{V}_{i,j+1}^{\pi,\ror_i}\right]\right)\Big>\notag\\
   &\leq6c_{\mathsf{b}}\sqrt{\frac{\log^3(\frac{KS\sum_{i=1}^nA_i}{\delta})}{KH}}\sum_{j=h}^H\sum_{k=1}^K\alpha_k^K\Big<b_h^j,\Big\lVert\mathsf{Var}_{\underline{P}_{i,h}^{k,\widehat{V}}}\left(\widehat{V}_{i,j+1}-\mathbb{E}_{\pi\sim \widehat{\xi}}\left[\overline{V}_{i,j+1}^{\pi,\ror_i}\right]\right)\Big\rVert_\infty\cdot 1\Big>\notag\\
    & \leq 6c_{\mathsf{b}}\sqrt{\frac{\log^3(\frac{KS\sum_{i=1}^nA_i}{\delta})}{KH}}\sum_{j=h}^H\sum_{k=1}^K\alpha_k^K\Big<b_h^j,\left\lVert\widehat{V}_{i,j+1}-\mathbb{E}_{\pi\sim \widehat{\xi}}\left[\overline{V}_{i,j+1}^{\pi,\ror_i}\right]\right\rVert_\infty^2\cdot 1\Big>\notag\\
    &\overset{\mathsf{(i)}}{\leq}6c_{\mathsf{b}}\sqrt{\frac{H\log^3(\frac{KS\sum_{i=1}^nA_i}{\delta})}{K}}\sum_{j=h}^H\sum_{k=1}^K\alpha_k^K\left<b_h^j,\left\lVert\widehat{V}_{i,j+1}-\mathbb{E}_{\pi\sim \widehat{\xi}}\left[\overline{V}_{i,j+1}^{\pi,\ror_i}\right]\right\rVert_\infty\cdot 1\right>\notag\\
   &\leq 6c_{\mathsf{b}}\sqrt{\frac{H^3\log^3(\frac{KS\sum_{i=1}^nA_i}{\delta})}{K}}\max_{h\leq j\leq H}\left\lVert\widehat{V}_{i,j+1}-\mathbb{E}_{\pi\sim \widehat{\xi}}\left[\overline{V}_{i,j+1}^{\pi,\ror_i}\right]\right\rVert_\infty\label{eq:upper_bound_D_1}
\end{align}
where $\mathsf{(i)}$ follows from the elementary upper bound $\big\lVert\widehat{V}_{i,j+1}\big\rVert_\infty\leq H$, $\big\lVert\mathbb{E}_{\pi\sim \widehat{\xi}}\big[\overline{V}_{i,j+1}^{\pi,\ror_i}\big]\big\rVert_\infty\leq H$ for all $h\leq j\leq H$, and the last line holds by \eqref{eq:property-bh}.

Before deriving the upper bounds for the terms $\mathcal{D}_2$ and $\mathcal{D}_3$, we introduce the following useful lemmas.
\begin{lemma}[Lemma 3, \citep{shi2024sample}]\label{lemma:pnorm-key-value-range}
    For all $(i,h) \in [n] \times [H]$, the auxiliary estimate $ \mathbb{E}_{\pi \sim \widehat{\xi}}\left[\overline{V}_{i,h}^{\pi,\ror_i}\right]$ satisfies
    \begin{align*}
        \max_{s \in \mathcal{S}} \mathbb{E}_{\pi \sim \widehat{\xi}}\left[ \overline{V}_{i,h}^{\pi,\ror_i}(s)\right] - \min_{s \in \mathcal{S}} \mathbb{E}_{\pi \sim \widehat{\xi}}\left[\overline{V}_{i,h}^{\pi,\ror_i}(s)\right] & \leq \min \left\{ \frac{1}{\sigma_i}, H-h+1 \right\}.
    \end{align*}
\end{lemma}

\begin{proof}
    Lemma~\ref{lemma:pnorm-key-value-range} can be directly verified by following the same proof pipeline of \citep[Lemma 3]{shi2024sample}.
\end{proof}
 With Lemma~\ref{lemma:pnorm-key-value-range} in hand, we introduce the following lemma on the variance with the proof postponed to Appendix~\ref{proof:eq:extra-lemma1}.

\begin{lemma}\label{eq:extra-lemma1}
    Consider any fixed vector $V\in\mathbb{R}^{S}$ with $\max_{s \in \mathcal{S}} V(s) - \min_{s \in \mathcal{S}} V(s)  \leq \min \left\{ \frac{1}{\sigma_i}, H-h+1 \right\}$. For a transition kernel $P' \in \mathbb{R}^S$ and any $\widetilde{P} \in \mathbb{R}^S$ such that $\widetilde{P} \in \mathcal{U}^{\sigma_i}(P^\prime)$, the following bound holds for all $(i,h,) \in [n] \times [H]$:
    \begin{subequations}
    \begin{align}
        & \left| \mathsf{Var}_{P'}\left(V\right) - \mathsf{Var}_{\widetilde{P}}\left(V\right)\right| \leq \min \left\{\frac{1}{\sigma_i}, H-h+1 \right\}. \label{eq:lemma-6-line2}
    \end{align}
    \end{subequations}
\end{lemma}

Now we are ready to control $\mathcal{D}_2$ and $\mathcal{D}_3$. 
The key part in $\mathcal{D}_2$ obeys
\begin{align}
\label{eq:term_B_variance_difference}
    & \left|\mathsf{Var}_{\underline{P}_{i,h}^{k,\widehat{V}}}\left(\mathbb{E}_{\pi\sim \widehat{\xi}}\left[\overline{V}_{i,h+1}^{\pi,\ror_i}\right]\right)-\mathsf{Var}_{\underline{P}_{i,h}^{k, \widehat{\xi}, \overline{V}}}\left(\mathbb{E}_{\pi\sim \widehat{\xi}}\left[\overline{V}_{i,h+1}^{\pi,\ror_i}\right]\right)\right| \notag \\
    &\leq\left|\mathsf{Var}_{\underline{P}_{i,h}^{k,\widehat{V}}}\left(\mathbb{E}_{\pi\sim \widehat{\xi}}\left[\overline{V}_{i,h+1}^{\pi,\ror_i}\right]\right)-\mathsf{Var}_{\underline{P}_{i,h}^{k}}\left(\mathbb{E}_{\pi\sim \widehat{\xi}}\left[\overline{V}_{i,h+1}^{\pi,\ror_i}\right]\right)\right| \notag \\
    &\quad  + \left|\mathsf{Var}_{\underline{P}_{i,h}^{k}}\left(\mathbb{E}_{\pi\sim \widehat{\xi}}\left[\overline{V}_{i,h+1}^{\pi,\ror_i}\right]\right)-\mathsf{Var}_{\underline{P}_{i,h}^{k, \widehat{\xi}, \overline{V}}}\left(\mathbb{E}_{\pi\sim \widehat{\xi}}\left[\overline{V}_{i,h+1}^{\pi,\ror_i}\right]\right)\right| \notag\\
     &\leq  2\min\left\{\frac{1}{\sigma_i},H\right\} 1,
\end{align}
which holds by applying Lemma~\ref{eq:extra-lemma1} with $V = \mathbb{E}_{\pi\sim \widehat{\xi}}\left[\overline{V}_{i,h+1}^{\pi,\ror_i}\right]$.

Inserting \eqref{eq:term_B_variance_difference} back to $\mathcal{D}_2$ gives
\begin{align}
\mathcal{D}_2&=6c_{\mathsf{b}}\sqrt{\frac{\log^3(\frac{KS\sum_{i=1}^nA_i}{\delta})}{KH}}\sum_{j=h}^H\sum_{k=1}^K\alpha_k^K\Big<b_h^j,\mathsf{Var}_{\underline{P}_{i,j}^{k,\widehat{V}}}\left(\mathbb{E}_{\pi\sim \widehat{\xi}}\left[\overline{V}_{i,j+1}^{\pi,\ror_i}\right]\right)-\mathsf{Var}_{\underline{P}_{i,j}^{k,\widehat{\xi},\overline{V}}}\left(\mathbb{E}_{\pi\sim \widehat{\xi}}\left[\overline{V}_{i,j+1}^{\pi,\ror_i}\right]\right)\Big>\notag\\
&\leq6c_{\mathsf{b}}\sqrt{\frac{\log^3(\frac{KS\sum_{i=1}^nA_i}{\delta})}{KH}}\sum_{j=h}^H\left<b_h^j,2\min\left\{\frac{1}{\sigma_i},H\right\}1\right>\notag\\
&=12c_{\mathsf{b}}\sqrt{\frac{H\log^3(\frac{KS\sum_{i=1}^nA_i}{\delta})}{K}}\min\left\{\frac{1}{\sigma_i},H\right\},\label{eq:upper_bound_D_2}
\end{align}
where the last line is achieved by applying \eqref{eq:property-bh}.

To control $\mathcal{D}_3$, we apply Lemma~\ref{lm:weight_variance} gives, which gives:
\begin{align}
    \mathcal{D}_3&=6c_{\mathsf{b}}\sqrt{\frac{\log^3(\frac{KS\sum_{i=1}^nA_i}{\delta})}{KH}}\sum_{j=h}^H\sum_{k=1}^K\alpha_k^K\left<b_h^j,\mathsf{Var}_{\underline{P}_{i,j}^{k,\widehat{\xi},\overline{V}}}\left(\mathbb{E}_{\pi\sim \widehat{\xi}}\left[\overline{V}_{i,j+1}^{\pi,\ror_i}\right]\right)\right> \notag \\
   &\leq 6c_{\mathsf{b}}\sqrt{\frac{\log^3(\frac{KS\sum_{i=1}^nA_i}{\delta})}{KH}}\sum_{j=h}^H\left<b_h^j,\mathsf{Var}_{\sum_{k=1}^K\alpha_k^K\underline{P}_{i,j}^{k,\widehat{\xi},\overline{V}}}\left(\mathbb{E}_{\pi\sim \widehat{\xi}}\left[\overline{V}_{i,j+1}^{\pi,\ror_i}\right]\right)\right> \notag \\
   &\overset{(\mathrm{i})}{\leq} 18c_{\mathsf{b}}\sqrt{\frac{H\log^3(\frac{KS\sum_{i=1}^nA_i}{\delta})}{K}}\left(\max_{s\in\mathcal{S}}\mathbb{E}_{\pi\sim \widehat{\xi}}\left[\overline{V}_{i,h}^{\pi,\ror_i}(s)\right] - \min_{s\in\mathcal{S}}\mathbb{E}_{\pi\sim \widehat{\xi}}\left[\overline{V}_{i,h}^{\pi,\ror_i}(s)\right]\right) \notag\\
  &  \overset{(\mathrm{ii})}{\leq} 18c_{\mathsf{b}}\sqrt{\frac{H\log^3(\frac{KS\sum_{i=1}^nA_i}{\delta})}{K}}\min\left\{\frac{1}{\sigma_i},H\right\}\label{eq:upper_bound_D_3}
\end{align}
where (i) can be verified by the following lemma, and (ii) holds by Lemma~\ref{lemma:pnorm-key-value-range}.

\begin{lemma}\label{lem:key-lemma-reduce-H-3}
    Consider any $\delta \in (0,1)$. With probability at least $1-\delta$, it holds for all $(h,i) \in [H] \times [n]$ that
    \begin{align}
    \sum_{j=h}^H \left< b_h^j, \mathsf{Var}_{\sum_{k=1}^K \alpha_k^K \underline{P}_{i,j}^{k, \overline{V}}}\left(\mathbb{E}_{\pi \sim \widehat{\xi}}\left[\overline{V}_{i,j+1}^{\pi,\ror_i}\right]\right) \right>  \leq 3H \left(\max_{s \in \mathcal{S}} \mathbb{E}_{\pi \sim \widehat{\xi}}\left[\overline{V}_{i,h}^{\pi,\ror_i}(s)\right] - \min_{s \in \mathcal{S}} \mathbb{E}_{\pi \sim \widehat{\xi}}\left[\overline{V}_{i,h}^{\pi,\ror_i}(s)\right]\right).
    \end{align}
\end{lemma}

\begin{proof}
    See Appendix~\ref{proof:lem:key-lemma-reduce-H-3}.
\end{proof}

\paragraph{Step 4: summing up the results.}
Plugging in the results in \eqref{eq:upper_bound_D_1}, \eqref{eq:upper_bound_D_2}, and \eqref{eq:upper_bound_D_3} in \eqref{eq:splitting-up-D123} gives
\begin{align*}
    &\widehat{V}_{i,h}-\mathbb{E}_{\pi\sim \widehat{\xi}}\left[\overline{V}^{\pi,\ror_i}_{i,h}\right] \leq 3c_{\mathsf{b}}\sqrt{\frac{H^3\log^3(\frac{KS\sum_{i=1}^nA_i}{\delta})}{K}}1+\mathcal{D}_1+\mathcal{D}_2+\mathcal{D}_3\\
   &\leq c_{\mathsf{b}}\sqrt{\frac{H\log^3(\frac{KS\sum_{i=1}^nA_i}{\delta})}{K}}\left(3H+30\min\left\{\frac{1}{\sigma_i},H\right\}\right)1 \notag \\
   &\quad +6c_{\mathsf{b}}\sqrt{\frac{H^3\log^3(\frac{KS\sum_{i=1}^nA_i}{\delta})}{K}}\max_{h\leq j\leq H}\left\lVert\widehat{V}_{i,j+1}-\mathbb{E}_{\pi\sim \widehat{\xi}}\left[\overline{V}_{i,j+1}^{\pi,\ror_i}\right]\right\rVert_\infty 1.
\end{align*}
In addition, Lemma~\ref{lm:term_B_value_vector_comparison} shows that $\widehat{V}_{i,h}-\mathbb{E}_{\pi\sim \widehat{\xi}}\left[\overline{V}^{\pi,\ror_i}_{i,h}\right]=\left|\widehat{V}_{i,h}-\mathbb{E}_{\pi\sim \widehat{\xi}}\left[\overline{V}^{\pi,\ror_i}_{i,h}\right]\right|$, which indicates that  
\begin{align*}
    &\max_{h\in[H]}\left\lVert \widehat{V}_{i,h} -\mathbb{E}_{\pi\sim \widehat{\xi}}\left[\overline{V}^{\pi,\ror_i}_{i,h}\right]\right\rVert_\infty\\
   &\leq c_{\mathsf{b}}\sqrt{\frac{H\log^3(\frac{KS\sum_{i=1}^nA_i}{\delta})}{K}}\left(3H+30\min\left\{\frac{1}{\sigma_i},H\right\}\right) \notag \\
   &\quad +6c_{\mathsf{b}}\sqrt{\frac{H^3\log^3(\frac{KS\sum_{i=1}^nA_i}{\delta})}{K}}\max_{h\leq j\leq H}\left\lVert\widehat{V}_{i,j+1}-\mathbb{E}_{\pi\sim \widehat{\xi}}\left[\overline{V}_{i,j+1}^{\pi,\ror_i}\right]\right\rVert_\infty \\
  &  \overset{\mathrm{(i)}}{\leq}33c_{\mathsf{b}}\sqrt{\frac{H^3\log^3(\frac{KS\sum_{i=1}^nA_i}{\delta})}{K}}+\frac{1}{2}\max_{h\in[H]}\left\lVert \widehat{V}_{i,h}-\mathbb{E}_{\pi\sim \widehat{\xi}}\left[\overline{V}^{\pi,\ror_i}_{i,h}\right]\right\rVert_\infty\\
    & \leq 66c_{\mathsf{b}}\sqrt{\frac{H^3\log^3(\frac{KS\sum_{i=1}^nA_i}{\delta})}{K}}
\end{align*}
where (i) holds as long as $K\geq 12c_{\mathsf{b}}^2H^3\log^3(\frac{KS\sum_{i=1}^nA_i}{\delta})$. Finally, plugging the above fact back to \eqref{eq:control-B-change-goal}, we can control B as below: 
\begin{align}
    B = \mathbb{E}_{\pi\sim \widehat{\xi}}\left[\overline{V}^{\star,\pi_{-i},\ror_i}_{i,h}\right]-\mathbb{E}_{\pi\sim \widehat{\xi}}\left[\overline{V}_{i,h}^{\pi,\ror_i}\right]\leq66c_{\mathsf{b}}\sqrt{\frac{H^3\log^3(\frac{KS\sum_{i=1}^nA_i}{\delta})}{K}}1.\label{eq:final_upper_bound_term_B}
\end{align}

    \subsection{Controlling terms A and C}
    \label{sec:upper_bound_estimation_error}

   \paragraph{Step 1: considering a general term.}
Recall term A and C in \eqref{eq:error_decomposition}
\begin{align}
A= \mathbb{E}_{\pi\sim \widehat{\xi}}\left[V_{i,h}^{\star,\pi_{-i},\ror_i}\right] - \mathbb{E}_{\pi\sim \widehat{\xi}}\left[\overline{V}_{i,h}^{\tilde{\pi}_i^\star,\pi_{-i},\ror_i}\right], \quad C=\mathbb{E}_{\pi\sim \widehat{\xi}}\left[\overline{V}_{i,h}^{\pi,\ror_i}\right] - \mathbb{E}_{\pi\sim \widehat{\xi}}\left[V_{i,h}^{\pi,\ror_i}\right].
\end{align}
We will control these two terms in a similar manner by considering a general set of policies \( \left\{\widehat{\pi}_h^k: \cS \mapsto \prod_{i \in [n]} \Delta(\cA_i)\right\}_{(h,k)\in[H]\times[K]} \).
Additionally, we introduce a distribution over this set of policies as \( \zeta \coloneqq \{\zeta_h\}_{h \in [H]} \), with \( \zeta_h : \cS \mapsto \Delta(\prod_{i \in [n]} \Delta(\mathcal{A}_i)) \), where \( \zeta_h\left(\widehat{\pi}_h^k\right) = \alpha_k^K \) for all \( (h,k) \in [H] \times [K] \). Moreover, similar to \eqref{eq:auxiliary_value_function_3}, we introduce a set of auxiliary value vectors, defined as 
\begin{align}
    &\mathbb{E}_{\pi\sim\zeta}\left[\overline{V}^{\pi,\ror_i}_{i,H+1}(s)\right] = 0, \notag \\
    &\mathbb{E}_{\pi\sim\zeta}\left[\overline{V}^{\pi,\ror_i}_{i,h}(s)\right] = \sum_{k=1}^K \alpha_{k}^K \mathbb{E}_{a_i \sim \widehat{\pi}_{i,h}^k(s)}[r_{i,h}^k(s,a_i)] + \sum_{k=1}^K \alpha_{k}^K \mathbb{E}_{a_i \sim \widehat{\pi}_{i,h}^k(s)} \left[ \inf_{\mathcal{P} \in \mathcal{U}^{\sigma_i}\left(P_{i,h,s,a_i}^{k}\right)} \mathcal{P} \mathbb{E}_{\pi \sim \zeta}\left[\overline{V}^{\pi,\ror_i}_{i,h+1}\right] \right], \label{eq:auxiliary-V-hat-2}
\end{align}
for all $h\in [H]$ and  $s\in\cS$.

With this general set of policies and its corresponding distribution $\zeta$, it can be observed that $\mathbb{E}_{\pi\sim\zeta}\left[V^{\pi,\ror_i}_{i,h}\right] - \mathbb{E}_{\pi\sim\zeta}\left[\overline{V}^{\pi,\ror_i}_{i,h}\right]$ is equal to A or negative C by setting either \( \widehat{\pi}_h^k = \tilde{\pi}_i^\star \times \pi_{-i,h}^k \) or \( \widehat{\pi}_h^k = \pi_h^k \) for all \( (h,k) \in [H] \times [K] \). Therefore, the rest of the proof will focus on controlling
\begin{align}
\left| \mathbb{E}_{\pi\sim\zeta}\left[V^{\pi,\ror_i}_{i,h}\right] - \mathbb{E}_{\pi\sim\zeta}\left[\overline{V}^{\pi,\ror_i}_{i,h}\right]\right|. \label{eq:goal-of-A-C}
\end{align}

Towards this, we first introduce some useful notations for convenience. We denote $\overline{r}_{i,h}^{\widehat{\pi}^k} \in \mathbb{R}^S$ with $\overline{r}_{i,h}^{\widehat{\pi}^k}(s) \defn \mathbb{E}_{a_i\sim \widehat{\pi}_{i,h}^k(s)}[r_{i,h}^k(s,a_i)]$ for all $s \in \mathcal{S}$. We also define the following extra notations of worst-case transition kernels similar to \eqref{eq:inf-p-special-marl}: for all $(i,h,k)\in[n]\times[H]\times[K]$,
\begin{align}
    &P_{i,h}^{\widehat{\pi}^k,V} := P_{i,h}^{\widehat{\pi}_{-i}^k, \mathbb{E}_{\pi\sim \zeta}\left[V_{i,h+1}^{\pi,\ror_i}\right]}, \quad \text{where} \quad P_{i,h,s,a_i}^{\widehat{\pi}^k,V} := P_{i,h,s,a_i}^{\widehat{\pi}_{-i}^k, \mathbb{E}_{\pi\sim \zeta}\left[V_{i,h+1}^{\pi,\ror_i}\right]}= \mathrm{argmin}_{\mathcal{P} \in \mathcal{U}^{\sigma_i}\left(P_{h,s,a_i}^{\widehat{\pi}^k_{-i}}\right)} \mathcal{P} \mathbb{E}_{\pi\sim \zeta}\left[V_{i,h+1}^{\pi,\ror_i}\right], \notag \\
    &P_{i,h}^{\widehat{\pi}^k,\overline{V}} := P_{i,h}^{\widehat{\pi}_{-i}^k, \mathbb{E}_{\pi\sim \zeta}\left[\overline{V}_{i,h+1}^{\pi,\ror_i}\right]}, \quad \text{where} \quad P_{i,h,s,a_i}^{\widehat{\pi}^k,\overline{V}} := P_{i,h,s,a_i}^{\widehat{\pi}_{-i}^k, \mathbb{E}_{\pi\sim \zeta}\left[\overline{V}_{i,h+1}^{\pi,\ror_i}\right]}= \mathrm{argmin}_{\mathcal{P} \in \mathcal{U}^{\sigma_i}\left(P_{h,s,a_i}^{\widehat{\pi}^k_{-i}}\right)} \mathcal{P} \mathbb{E}_{\pi\sim \zeta}\left[\overline{V}_{i,h+1}^{\pi,\ror_i}\right], \notag \\
    &P_{i,h}^{k, \zeta,\overline{V}} := P_{i,h}^{k, \mathbb{E}_{\pi\sim \zeta}\left[\overline{V}_{i,h+1}^{\pi,\ror_i}\right]} \quad \text{and} \quad P_{i,h,s,a_i}^{k, \zeta,\overline{V}} := P_{i,h,s,a_i}^{k, \mathbb{E}_{\pi\sim \zeta}\left[\overline{V}_{i,h+1}^{\pi,\ror_i}\right]}= \mathrm{argmin}_{\mathcal{P} \in \mathcal{U}^{\sigma_i}\left(P_{i,h,s,a_i}^{k}\right)} \mathcal{P} \mathbb{E}_{\pi\sim \zeta}\left[\overline{V}_{i,h+1}^{\pi,\ror_i}\right]. \label{eq:matrix-notations-A-C-0}
\end{align}
Additionally, we also introduce the square matrices $\underline{P}_{i,h}^{\widehat{\pi}^k,V} \in \mathbb{R}^{S\times S}, \underline{P}_{i,h}^{\widehat{\pi}^k,\overline{V}} \in \mathbb{R}^{S \times S}$ and \( \underline{P}_{i,h}^{k, \zeta, \overline{V}} \in \mathbb{R}^{S \times S} \) as:
\begin{align}
    \underline{P}_{i,h}^{\widehat{\pi}^k,V}: = \Pi_h^{\widehat{\pi}_i^k} P_{i,h}^{\widehat{\pi}^k,V}, \quad 
    \underline{P}_{i,h}^{\widehat{\pi}^k,\overline{V}} := \Pi_h^{\widehat{\pi}_i^k} P_{i,h}^{\widehat{\pi}_{-i}^k, \overline{V}}, \quad \underline{P}_{i,h}^{k, \zeta, \overline{V}} := \Pi_h^{\widehat{\pi}_i^k} P_{i,h}^{k, \zeta,\overline{V}}. \label{eq:matrix-notations-A-C}
\end{align}

With those in mind, we are ready to control \eqref{eq:goal-of-A-C}: at any time step $h\in[H]$,
    \begin{align}
        &\mathbb{E}_{\pi\sim\zeta}\left[V^{\pi,\ror_i}_{i,h}\right]-\mathbb{E}_{\pi\sim\zeta}\left[\overline{V}^{\pi,\ror_i}_{i,h}\right] \notag \\
        &\overset{(\mathrm{i})}{=}\sum_{k=1}^K\alpha_k^Kr_{i,h}^{\widehat{\pi}^k}+\sum_{k=1}^K \alpha_{k}^K \Pi_{h}^{\widehat{\pi}_{i,h}^k}\left[ \inf_{\mathcal{P}\in\mathcal{U}^{\sigma_i}\left(P_{h,s,a_i}^{\widehat{\pi}_{-i}^k}\right)} \mathcal{P} \mathbb{E}_{\pi\sim\zeta}\left[V^{\pi,\ror_i}_{i,h+1}\right]\right] \notag\\
        &\quad -\sum_{k=1}^K\alpha_k^K\overline{r}_{i,h}^{\widehat{\pi}^k}-\sum_{k=1}^K \alpha_{k}^K \Pi_{h}^{\widehat{\pi}_{i,h}^k} \left[ \inf_{\mathcal{P}\in\mathcal{U}^{\sigma_i}\left(P_{h,s,a_i}^{k}\right)} \mathcal{P} \mathbb{E}_{\pi\sim\zeta}\left[\overline{V}^{\pi,\ror_i}_{i,h+1}\right]\right] \notag \\
        &\overset{(\mathrm{ii})}{=} \sum_{k=1}^K\alpha_k^Kr_{i,h}^{\widehat{\pi}^k}+\sum_{k=1}^K \alpha_{k}^K\underline{P}_{i,h}^{\widehat{\pi}^k,V}\mathbb{E}_{\pi\sim\zeta}\left[V_{i,h+1}^{\pi,\ror_i}\right]-\sum_{k=1}^K\alpha_k^K\overline{r}_{i,h}^{\widehat{\pi}^k}-\sum_{k=1}^K \alpha_{k}^K\underline{P}_{i,h}^{k, \zeta, \overline{V}}\mathbb{E}_{\pi\sim\zeta}\left[\overline{V}_{i,h+1}^{\pi,\ror_i}\right] \label{eq:no-name-1} \\
        &=\sum_{k=1}^K\alpha_k^K\bigg[\left(r_{i,h}^{\widehat{\pi}^k}-\overline{r}_{i,h}^{\widehat{\pi}^k}\right)+\left(\underline{P}_{i,h}^{\widehat{\pi}^k,V}\mathbb{E}_{\pi\sim\zeta}[V_{i,h+1}^{\pi,\ror_i}]-\underline{P}_{i,h}^{\widehat{\pi}^k,\overline{V}}\mathbb{E}_{\pi\sim\zeta}\left[\overline{V}_{i,h+1}^{\pi,\ror_i}\right]\right) \notag \\
        &\quad 
        +\left(\underline{P}_{i,h}^{\widehat{\pi}^k,\overline{V}}\mathbb{E}_{\pi\sim\zeta}\left[\overline{V}_{i,h+1}^{\pi,\ror_i}\right]-\underline{P}_{i,h}^{k, \zeta, \overline{V}}\mathbb{E}_{\pi\sim\zeta}\left[\overline{V}_{i,h+1}^{\pi,\ror_i}\right]\right)\bigg]\notag\\
        &\leq\sum_{k=1}^K\alpha_k^K\left(\underline{P}_{i,h}^{\widehat{\pi}^k,\overline{V}}\mathbb{E}_{\pi\sim\zeta}\left[V_{i,h+1}^{\pi,\ror_i}\right]-\underline{P}_{i,h}^{\widehat{\pi}^k,\overline{V}}\mathbb{E}_{\pi\sim\zeta}\left[\overline{V}_{i,h+1}^{\pi,\ror_i}\right]\right)\notag\\
        &\quad +\underbrace{\sum_{k=1}^K\alpha_k^K\left[\left|r_{i,h}^{\widehat{\pi}^k}-\overline{r}_{i,h}^{\widehat{\pi}^k}\right|+\left|\underline{P}_{i,h}^{\widehat{\pi}^k,\overline{V}}\mathbb{E}_{\pi\sim\zeta}\left[\overline{V}_{i,h+1}^{\pi,\ror_i}\right]-\underline{P}_{i,h}^{k, \zeta, \overline{V}}\mathbb{E}_{\pi\sim\zeta}\left[\overline{V}_{i,h+1}^{\pi,\ror_i}\right]\right|\right]}_{:=u_{i,h}^\zeta}\label{eq:recursion_decomposition}.
    \end{align}
    To proceed, applying \eqref{eq:recursion_decomposition} recursively leads to
    \begin{align}
    \label{eq:matrix_multiply_recursion}
        &\mathbb{E}_{\pi\sim\zeta}\left[V^{\pi,\ror_i}_{i,h}\right]-\mathbb{E}_{\pi\sim\zeta}\left[\overline{V}^{\pi,\ror_i}_{i,h}\right]\leq\sum_{j=h}^H\left[\prod_{r=h}^{j-1}\left(\sum_{k=1}^K\alpha_k^K\underline{P}_{i,r}^{\widehat{\pi}^k,\overline{V}}\right)\right]u_{i,j}^{\zeta},
    \end{align}
    where the inequality follows from the abuse of notation described below:
    \begin{align*}
        &\left[\prod_{r=h}^{h-1}\left(\sum_{k=1}^K\alpha_k^K\underline{P}_{i,r}^{\widehat{\pi}^k,\overline{V}}\right)\right]=I, \notag \\
        &\left[\prod_{r=h}^{j-1}\left(\sum_{k=1}^K\alpha_k^K\underline{P}_{i,r}^{\widehat{\pi}^k,\overline{V}}\right)\right]=\left(\sum_{k=1}^K\alpha_k^K\underline{P}_{i,h}^{\widehat{\pi}^k,\overline{V}}\right)\cdot\left(\sum_{k=1}^K\alpha_k^K\underline{P}_{i,h+1}^{\widehat{\pi}^k,\overline{V}}\right)\cdots\left(\sum_{k=1}^K\alpha_k^K\underline{P}_{i,j-1}^{\widehat{\pi}^k,\overline{V}}\right), \quad for j= h+1, \cdots, H.
    \end{align*}
    Conversely, we can achieve the following in a manner analogous to \eqref{eq:recursion_decomposition}, 
    \begin{align}
        &\mathbb{E}_{\pi\sim\zeta}\left[\overline{V}^{\pi,\ror_i}_{i,h}\right]-\mathbb{E}_{\pi\sim\zeta}\left[V^{\pi,\ror_i}_{i,h}\right] \notag\\
        =& \sum_{k=1}^K\alpha_k^K\overline{r}_{i,h}^{\widehat{\pi}^k}+\sum_{k=1}^K\alpha_k^K\underline{P}_{i,h}^{k, \zeta, \overline{V}}\mathbb{E}_{\pi\sim\zeta}\left[\overline{V}_{i,h+1}^{\pi,\ror_i}\right]-\sum_{k=1}^K\alpha_k^Kr_{i,h}^{\widehat{\pi}^k}-\sum_{k=1}^K\alpha_k^K\underline{P}_{i,h}^{\widehat{\pi}^k,V}\mathbb{E}_{\pi\sim\zeta}\left[V_{i,h+1}^{\pi,\ror_i}\right] \notag\\
        =&\sum_{k=1}^K\alpha_k^K\Bigg[\left(\overline{r}_{i,h}^{\widehat{\pi}^k}-r_{i,h}^{\widehat{\pi}^k}\right)+\left(\underline{P}_{i,h}^{k, \zeta, \overline{V}}\mathbb{E}_{\pi\sim\zeta}\left[\overline{V}_{i,h+1}^{\pi,\ror_i}\right]-\underline{P}_{i,h}^{\widehat{\pi}^k,\overline{V}}\mathbb{E}_{\pi\sim\zeta}\left[\overline{V}_{i,h+1}^{\pi,\ror_i}\right]\right)\\
        &\quad +\left(\underline{P}_{i,h}^{\widehat{\pi}^k,\overline{V}}\mathbb{E}_{\pi\sim\zeta}\left[\overline{V}_{i,h+1}^{\pi,\ror_i}\right]-\underline{P}_{i,h}^{\widehat{\pi}^k,V}\mathbb{E}_{\pi\sim\zeta}\left[V_{i,h+1}^{\pi,\ror_i}\right]\right)\Bigg] \notag\\
        \overset{(\mathrm{i})}{\leq}&\sum_{k=1}^K\alpha_k^K\left(\underline{P}_{i,h}^{\widehat{\pi}^k,V}\mathbb{E}_{\pi\sim\zeta}\left[\overline{V}_{i,h+1}^{\pi,\ror_i}\right]-\underline{P}_{i,h}^{\widehat{\pi}^k,V}\mathbb{E}_{\pi\sim\zeta}\left[V_{i,h+1}^{\pi,\ror_i}\right]\right) \notag\\
        &\quad +\sum_{k=1}^K\alpha_k^K\left[\left|r_{i,h}^{\widehat{\pi}^k}-\overline{r}_{i,h}^{\widehat{\pi}^k}\right|+\left|\underline{P}_{i,h}^{\widehat{\pi}^k,\overline{V}}\mathbb{E}_{\pi\sim\zeta}\left[\overline{V}_{i,h+1}^{\pi,\ror_i}\right]-\underline{P}_{i,h}^{k, \zeta, \overline{V}}\mathbb{E}_{\pi\sim\zeta}\left[\overline{V}_{i,h+1}^{\pi,\ror_i}\right]\right|\right] \notag \\
        \leq &\sum_{j=h}^H\left[\prod_{r=h}^{j-1}\left(\sum_{k=1}^K\alpha_k^K\underline{P}_{i,r}^{\widehat{\pi}^k,V}\right)\right]u_{i,j}^{\zeta}, \label{eq:matrix_multiply_recursion_reverse}
    \end{align}
    where (i) holds by the fact $\underline{P}_{i,h}^{\widehat{\pi}^k,\overline{V}}\mathbb{E}_{\pi\sim\zeta}\left[\overline{V}_{i,h+1}^{\pi,\ror_i}\right]\leq \underline{P}_{i,h}^{\widehat{\pi}^k,V}\mathbb{E}_{\pi\sim\zeta}\left[\overline{V}_{i,h+1}^{\pi,\ror_i}\right]$. 
    Summing up the results in \eqref{eq:matrix_multiply_recursion_reverse} and \eqref{eq:matrix_multiply_recursion}, one has
    \begin{align}
        \left|\mathbb{E}_{\pi\sim\zeta}\left[\overline{V}^{\pi,\ror_i}_{i,h}\right]-\mathbb{E}_{\pi\sim\zeta}\left[V^{\pi,\ror_i}_{i,h}\right]\right|&\leq \max\left\{\mathbb{E}_{\pi\sim\zeta}\left[V^{\pi,\ror_i}_{i,h}\right]-\mathbb{E}_{\pi\sim\zeta}\left[\overline{V}^{\pi,\ror_i}_{i,h}\right],\mathbb{E}_{\pi\sim\zeta}\left[\overline{V}^{\pi,\ror_i}_{i,h}\right]-\mathbb{E}_{\pi\sim\zeta}\left[V^{\pi,\ror_i}_{i,h}\right]\right\} \notag\\
       &\leq\max\left\{\sum_{j=h}^H\left[\prod_{r=h}^{j-1}\left(\sum_{k=1}^K\alpha_k^K\underline{P}_{i,r}^{\widehat{\pi}^k,\overline{V}}\right)\right]u_{i,j}^{\zeta},\sum_{j=h}^H\left[\prod_{r=h}^{j-1}\left(\sum_{k=1}^K\alpha_k^K\underline{P}_{i,r}^{\widehat{\pi}^k,V}\right)\right]u_{i,j}^{\zeta}\right\},\label{eq:recursion_decomposition_target}
    \end{align}
    where the max operator is taken entry-wise for vectors.

\paragraph{Step 2: controlling $u_{i,j}^{\zeta}$.}
Towards this, we introduce the following two lemmas in terms of estimation error on transition kernels and rewards, respectively.

\begin{lemma}\label{lemma:tv-dro-b-bound-star-marl}
    Let $\delta \in (0,1)$ and consider any $(h,i,k) \in [H] \times [n] \times [K]$. When $N\geq  \log \left(\frac{2S \sum_{i=1}^n A_i N K n}{\delta}\right)$, with a probability of at least $1 - \delta$, for any fixed value vector $V \in \mathbb{R}^S$ with $0 \leq V(s) \leq H$ for all $s \in \mathcal{S}$, we have
    \begin{align*}
        & \left|\underline{P}_{i,h}^{\widehat{\pi}^k,\overline{V}}\mathbb{E}_{\pi\sim\zeta}\left[\overline{V}_{i,h+1}^{\pi,\ror_i}\right]-\underline{P}_{i,h}^{k, \zeta, \overline{V}}\mathbb{E}_{\pi\sim\zeta}\left[\overline{V}_{i,h+1}^{\pi,\ror_i}\right]\right|  \notag \\
         & \leq 2  \sqrt{\frac{\log\left(\frac{18S\sum_{i=1}^nA_iNHK}{\delta}\right)}{N}} \sqrt{\mathsf{Var}_{\underline{P}_{h}^{\widehat{\pi}^k}}\left(\mathbb{E}_{\pi\sim\zeta}\left[\overline{V}_{i,h+1}^{\pi,\ror_i}\right]\right)}  + \frac{\log\left(\frac{18S\sum_{i=1}^nA_iNHK}{\delta}\right)}{N} 1\\
        & \leq  3 \sqrt{\frac{H^2 \log\left(\frac{18S\sum_{i=1}^nA_iNKH}{\delta}\right)}{N}}1,
    \end{align*}
    where $\mathsf{Var}_{\underline{P}_{h}^{\widehat{\pi}^k}}(\cdot)$ was defined in \eqref{eq:defn-variance-vector-marl}. 
    \end{lemma}
    
    \begin{proof}
        See Appendix~\ref{proof:lemma:tv-dro-b-bound-star-marl}.
    \end{proof}
    \begin{lemma}\label{lm:reward_concentration}
    There exists some constant $c_r$ such that for any $(h,i) \in [H] \times [n]$, with probability at least $1 - \delta$, one has
    \begin{align*}
        \left| \sum_{k=1}^K \alpha_k^K r_{i,h}^{\widehat{\pi}^k} - \sum_{k=1}^K \alpha_k^K \overline{r}_{i,h}^{\widehat{\pi}^k} \right| \leq c_r \sqrt{\frac{\log\left(\frac{KS}{\delta}\right)}{K}}1.
    \end{align*}
    \end{lemma}
    
    \begin{proof}
        See Appendix~\ref{sec:lm_reward_concentration}.
    \end{proof}

Armed with Lemma~\ref{lemma:tv-dro-b-bound-star-marl} and Lemma~\ref{lm:reward_concentration}, to control $u_{i,h}^\zeta$ in \eqref{eq:recursion_decomposition}, one has for all $(i,j)\in[n]\times[H]$,
    \begin{align}
        u_{i,h}^\zeta&=\sum_{k=1}^K\alpha_k^K\left[\left|r_{i,h}^{\widehat{\pi}^k}-\overline{r}_{i,h}^{\widehat{\pi}^k}\right|+\left|\underline{P}_{i,h}^{\widehat{\pi}^k,\overline{V}}\mathbb{E}_{\pi\sim\zeta}\left[\overline{V}_{i,h+1}^{\pi,\ror_i}\right]-\underline{P}_{i,h}^{k, \zeta, \overline{V}}\mathbb{E}_{\pi\sim\zeta}\left[\overline{V}_{i,h+1}^{\pi,\ror_i}\right]\right|\right] \notag \\
       &\leq2\sum_{k=1}^K\alpha_k^K\sqrt{\frac{\log\left(\frac{18S\sum_{i=1}^nA_iKNnH}{\delta}\right)}{N}}\sqrt{\mathsf{Var}_{\underline{P}_{h}^{\widehat{\pi}^k}}\left(\mathbb{E}_{\pi\sim\zeta}\left[\overline{V}^\pi_{i,h+1}\right]\right)} \notag \\
       &\quad +\frac{\log\left(\frac{18S\sum_{i=1}^nA_iKNnH}{\delta}\right)}{N}1+c_r\sqrt{\frac{\log(\frac{KSnH}{\delta})}{K}}1 \label{eq:bound-of-u-ih}
    \end{align} 
holds with probability at least $1-\delta$ when $N\geq  \log \left(\frac{2S \sum_{i=1}^n A_i N K n}{\delta}\right)$

\paragraph{Step 3: Controlling the first term in \eqref{eq:recursion_decomposition_target}.}
To continue, we first focus on controlling the first term in \eqref{eq:recursion_decomposition_target}. To proceed, recalling that $e_s$ represents the standard basis vector in $S$-dimensional space associated with the $s$-th component, and defining
\begin{align}
    d_h^h \defn e_s \quad \text{and} \quad \big(d_h^j\big)^\top \defn e_s^\top \left[\prod_{r=h}^{j-1} \left(\sum_{k=1}^K \alpha_k^K \underline{P}_{i,r}^{\widehat{\pi}^k, \overline{V}}\right)\right], \quad \text{for } j = h+1, \ldots, H, \label{eq:defn-of-d}
\end{align}
which obey 
\begin{align}
d_{j}^h \cdot 1 = 1 , \quad \forall j= h, h+ 1, \cdots, H. \label{eq:property-dh}
\end{align}
With these notations in place, we have for any $s\in\cS$,
\begin{align}
     &\mathbb{E}_{\pi\sim\zeta}\left[V_{i,h}^{\pi,\ror_i}(s)\right] - \mathbb{E}_{\pi\sim\zeta}\left[\overline{V}_{i,h}^{\pi,\ror_i}(s)\right] = \left\langle e_s, \mathbb{E}_{\pi\sim\zeta}\left[V_{i,h}^{\pi,\ror_i}\right] - \mathbb{E}_{\pi\sim\zeta}\left[\overline{V}_{i,h}^{\pi,\ror_i}\right] \right\rangle \notag \\
    & \overset{\mathrm{(i)}}{\leq} \sum_{j=h}^H \left\langle d_h^j, u_{i,j}^\zeta \right\rangle \notag \\
    & \overset{\mathrm{(ii)}}{\leq}  \sum_{j=h}^H \left\langle d_h^j, 2 \sum_{k=1}^K \alpha_k^K \sqrt{\frac{\log\left(\frac{18S \sum_{i=1}^n A_i K N H}{\delta}\right)}{N}} \sqrt{\mathsf{Var}_{\underline{P}_{j}^{\widehat{\pi}^k}}\left( \mathbb{E}_{\pi\sim\zeta}\left[\overline{V}_{i,j+1}^{\pi,\ror_i}\right]\right)} \right\rangle \notag \\
   &\quad + \frac{\log\left(\frac{18S \sum_{i=1}^n A_i K N H}{\delta}\right)}{N}  + c_r \sqrt{\frac{\log\left(\frac{K S n H}{\delta}\right)}{K}} \notag\\
     &\leq  \frac{H \log\left(\frac{18S \sum_{i=1}^n A_i K N H}{\delta}\right)}{N} + c_r \sqrt{\frac{H^2 \log\left(\frac{K S n H}{\delta}\right)}{K}}\notag\\ 
    &\quad + \sum_{j=h}^H \left\langle d_h^j, 2 \sum_{k=1}^K \alpha_k^K \sqrt{\frac{\log\left(\frac{18S \sum_{i=1}^n A_i K N H}{\delta}\right)}{N}} \sqrt{\mathsf{Var}_{\underline{P}_{j}^{\widehat{\pi}^k}}\left( \mathbb{E}_{\pi\sim\zeta}\left[\overline{V}_{i,j+1}^{\pi,\ror_i}\right]\right)} \right\rangle.  \label{eq:first-term-sum-control}
\end{align}
where (i) holds by \eqref{eq:matrix_multiply_recursion}, and (ii) is achieved by \eqref{eq:bound-of-u-ih}

We further decompose the term of interest as follows, along with the triangle inequality:
\begin{align}
    &\mathbb{E}_{\pi\sim\zeta}\left[V_{i,h}^{\pi,\ror_i}(s)\right] - \mathbb{E}_{\pi\sim\zeta}\left[\overline{V}_{i,h}^{\pi,\ror_i}(s)\right] \notag\\
     &\leq  \frac{H \log\left(\frac{18S \sum_{i=1}^n A_i K N H}{\delta}\right)}{N} + c_r \sqrt{\frac{H^2 \log\left(\frac{K S n H}{\delta}\right)}{K}}\notag\\ 
    &\quad + \underbrace{\sum_{j=h}^H \left\langle d_h^j, 2 \sum_{k=1}^K \alpha_k^K \sqrt{\frac{\log\left(\frac{18S \sum_{i=1}^n A_i K N H}{\delta}\right)}{N}} \sqrt{\mathsf{Var}_{\underline{P}_{i,j}^{\widehat{\pi}^k,\overline{V}}}\left(\mathbb{E}_{\pi\sim\zeta}\left[\overline{V}_{i,j+1}^{\pi,\ror_i}\right]\right)} \right\rangle}_{\mathcal{B}_1} \notag\\
    & + \underbrace{\sum_{j=h}^H 2 \sum_{k=1}^K \alpha_k^K \sqrt{\frac{\log\left(\frac{18S \sum_{i=1}^n A_i K N H}{\delta}\right)}{N}} \left\langle d_h^j, \sqrt{\left|\mathsf{Var}_{\underline{P}_{i,j}^{\widehat{\pi}^k,\overline{V}}}\left(\mathbb{E}_{\pi\sim\zeta}\left[\overline{V}_{i,j+1}^{\pi,\ror_i}\right]\right) - \mathsf{Var}_{\underline{P}_{j}^{\widehat{\pi}^k}}\left(\mathbb{E}_{\pi\sim\zeta}\left[\overline{V}_{i,j+1}^{\pi,\ror_i}\right]\right)\right|} \right\rangle}_{\mathcal{B}_2}. \label{eq:key-concentration-bound2}
\end{align}

We then analyze the bounds for the terms $\mathcal{B}_1$ and $\mathcal{B}_2$ separately.

\begin{itemize}

\item Controlling $\mathcal{B}_1$. 

First, applying Lemma~\ref{lm:weight_variance} gives
\begin{align}
    \mathcal{B}_1&=\sum_{j=h}^H\left<d_h^j,2\sum_{k=1}^K\alpha_k^K\sqrt{\frac{\log\left(\frac{18S\sum_{i=1}^nA_iKNnH}{\delta}\right)}{N}}\sqrt{\mathsf{Var}_{\underline{P}_{i,j}^{\widehat{\pi}^k,\overline{V}}}\left(\mathbb{E}_{\pi\sim\zeta}\left[\overline{V}_{i,j+1}^{\pi,\ror_i}\right]\right)}\right> \notag\\
   &\leq2\sqrt{\frac{\log\left(\frac{18S\sum_{i=1}^nA_iKNnH}{\delta}\right)}{N}}\sum_{j=h}^H\left<d_h^j,\sqrt{\mathsf{Var}_{\sum_{k=1}^K\alpha_k^K\underline{P}_{i,j}^{\widehat{\pi}^k,\overline{V}}}\left(\mathbb{E}_{\pi\sim\zeta}\left[\overline{V}_{i,j+1}^{\pi,\ror_i}\right]\right)}\right> \notag\\
   &\leq2\sqrt{\frac{\log\left(\frac{18S\sum_{i=1}^nA_iKNnH}{\delta}\right)}{N}}\sqrt{H\sum_{j=h}^H\left<d_h^j,\mathsf{Var}_{\sum_{k=1}^K\alpha_k^K\underline{P}_{i,j}^{\widehat{\pi}^k,\overline{V}}}\left(\mathbb{E}_{\pi\sim\zeta}\left[\overline{V}_{i,j+1}^{\pi,\ror_i}\right]\right)\right>}, \label{eq:bound_B_1_variance}
\end{align}
where the last inequality arises from the Cauchy-Schwartz inequality. Before proceeding, we introduce the following facts for the key elements in \eqref{eq:bound_B_1_variance}.
\begin{lemma}
    \label{lem:key-lemma-reduce-H}
    Consider any $\delta\in(0,1)$. With probability at least $1-\delta$, one has for all $(h,i)\in [H] \times [n]$:
    \begin{align}
       &\sum_{j=h}^H\left<d_h^j,\mathsf{Var}_{\sum_{k=1}^K\alpha_k^K\underline{P}_{i,j}^{\widehat{\pi}^k,\overline{V}}}\left(\mathbb{E}_{\pi\sim\zeta}\left[\overline{V}_{i,j+1}^{\pi,\ror_i}\right]\right)\right>  \notag\\
      \leq &3H \left(\max_{s\in\mathcal{S}}\mathbb{E}_{\pi\sim\zeta}\left[\overline{V}_{i,h+1}^\pi(s)\right] - \min_{s\in\mathcal{S}}\mathbb{E}_{\pi\sim\zeta}\left[\overline{V}_{i,h+1}^\pi(s)\right]\right) \left(1 + 2H\sqrt{\frac{\log(\frac{18S\sum_{i=1}^nA_inKNH}{\delta})}{N}} \right).
    \end{align}
\end{lemma}

\begin{proof}
    See Appendix \ref{sec:key-lemma-reduce-H}.
\end{proof}
\begin{lemma}[Lemma 3, \citep{shi2024sample}]\label{lemma:pnorm-key-value-range_A_and_C}
    For all $(i,h) \in [n] \times [H]$, the auxiliary estimate $ \mathbb{E}_{\pi \sim \zeta}\left[\overline{V}_{i,h}^{\pi,\ror_i}\right]$ obeys
    \begin{align*}
        \max_{s \in \mathcal{S}} \mathbb{E}_{\pi \sim \zeta}\left[ \overline{V}_{i,h}^{\pi,\ror_i}(s)\right] - \min_{s \in \mathcal{S}} \mathbb{E}_{\pi \sim \zeta}\left[\overline{V}_{i,h}^{\pi,\ror_i}(s)\right] & \leq \min \left\{ \frac{1}{\sigma_i}, H-h+1 \right\}. \notag \\
         \max_{s \in \mathcal{S}} \mathbb{E}_{\pi \sim \zeta}\left[ V_{i,h}^{\pi,\ror_i}(s)\right] - \min_{s \in \mathcal{S}} \mathbb{E}_{\pi \sim \zeta}\left[V_{i,h}^{\pi,\ror_i}(s)\right] & \leq \min \left\{ \frac{1}{\sigma_i}, H-h+1 \right\}.
    \end{align*}
\end{lemma}

\begin{proof}
    Similar to Lemma~\ref{lemma:pnorm-key-value-range}, Lemma~\ref{lemma:pnorm-key-value-range_A_and_C} can be directly verified by following the same proof pipeline of \citep[Lemma 3]{shi2024sample}.
\end{proof}
Apply Lemma~\ref{lem:key-lemma-reduce-H} to \eqref{eq:bound_B_1_variance}, we arrive at
\begin{align}
    \mathcal{B}_1&\leq2\sqrt{\frac{\log\left(\frac{18S\sum_{i=1}^nA_iKNnH}{\delta}\right)}{N}}\sqrt{H\sum_{j=h}^H\left<d_h^j,\mathsf{Var}_{\sum_{k=1}^K\alpha_k^K\underline{P}_{i,j}^{\widehat{\pi}^k,\overline{V}}}\left(\mathbb{E}_{\pi \sim \zeta}\left[\overline{V}_{i,j+1}^{\pi,\ror_i}\right]\right)\right>}\notag\\
   &\leq\sqrt{3H^2 \left(\max_{s\in\mathcal{S}}\mathbb{E}_{\pi \sim \zeta}\left[\overline{V}_{i,h+1}^\pi(s)\right] - \min_{s\in\mathcal{S}}\mathbb{E}_{\pi \sim \zeta}\left[\overline{V}_{i,h+1}^\pi(s)\right]\right)}\notag\\
   &\cdot2\sqrt{\frac{\log\left(\frac{18S\sum_{i=1}^nA_iKNnH}{\delta}\right)}{N}\cdot\left(1 + 2H\sqrt{\frac{\log(\frac{18S\sum_{i=1}^nA_inKNH}{\delta})}{N}} \right)}\notag\\
    \overset{(\mathsf{i})}{\leq}&2\sqrt{\frac{\log\left(\frac{18S\sum_{i=1}^nA_inKNnH}{\delta}\right)}{N}}\sqrt{3H^2 \min\left\{\frac{1}{\sigma_i},H-h+1\right\}\left(1 + 2H\sqrt{\frac{\log(\frac{18S\sum_{i=1}^nA_inKNH}{\delta})}{N}} \right)}\notag\\
   &\leq6\sqrt{\frac{H^2\min\{1/\sigma_i,H\}\log\left(\frac{18S\sum_{i=1}^nA_iKNnH}{\delta}\right)}{N}}\label{eq:key-concentration-bound2-solve1},
\end{align}
where (i) holds by  Lemma~\ref{lemma:pnorm-key-value-range_A_and_C}, and the final inequality follows by taking $N \geq 4H^2\log\left(\frac{18S\sum_{i=1}^nA_iKnNH}{\delta}\right)$.

\item Controlling $\mathcal{B}_2$.
Initially, we introduce the following fact that can be verifeid similar to Lemma~\ref{eq:extra-lemma1}.
\begin{lemma}\label{eq:extra-lemma1_A_and_C}
            For transition kernel $P' \in \mathbb{R}^S$ and any $\widetilde{P} \in \mathbb{R}^S$ such that $\widetilde{P} \in \mathcal{U}^{\sigma_i}\left(P^\prime\right)$, the following bounds are established for all $(i,h)\in[n]\times[H]$:
            \begin{align*}
                \left| \mathsf{Var}_{P'}\left(\mathbb{E}_{\pi \sim \zeta}\left[ \overline{V}_{i,h}^{\pi,\ror_i}\right]\right) - \mathsf{Var}_{\widetilde{P}}\left(\mathbb{E}_{\pi \sim \zeta}\left[ \overline{V}_{i,h}^{\pi,\ror_i}\right]\right)\right| \leq \min \left\{\frac{1}{\sigma_i}, H-h+1 \right\}.
            \end{align*}
            \end{lemma}

It is observed that the main term of $\mathcal{B}_2$ in \eqref{eq:key-concentration-bound2} obeys
\begin{align}
	&\left| \mathsf{Var}_{\underline{P}_{j}^{\widehat{\pi}^k}}\left(\mathbb{E}_{\pi \sim \zeta}\left[\overline{V}_{i,j+1}^{\pi,\ror_i}\right]\right) - \mathsf{Var}_{\underline{P}_{i,j}^{\widehat{\pi}^k, \overline{V}}}\left(\mathbb{E}_{\pi \sim \zeta}\left[\overline{V}_{i,j+1}^{\pi,\ror_i}\right]\right)\right|\notag\\
&	 \overset{\mathrm{(i)}}{\leq} \left\| \mathsf{Var}_{\underline{P}_{j}^{\widehat{\pi}^k}}\left(\mathbb{E}_{\pi \sim \zeta}\left[\overline{V}_{i,j+1}^{\pi,\ror_i}\right]\right) - \mathsf{Var}_{\underline{P}_{i,j}^{\widehat{\pi}^k, \overline{V}}}\left(\mathbb{E}_{\pi \sim \zeta}\left[\overline{V}_{i,j+1}^{\pi,\ror_i}\right]\right) \right\|_\infty 1    \notag \\
&	 \leq \min \left\{\frac{1}{\sigma_i}, H-h+1 \right\} 1, \label{eq:tv-first-C2}
\end{align}
where the last inequality holds by applying Lemma~\ref{eq:extra-lemma1} for each $s\in\cS$ in a entry-wise manner. Specifically,  for each $s\in\cS$, let $V= \overline{V}_{i,j+1}^{\pi,\ror_i}$, $P' = \mathbb{E}_{a \sim \widehat{\pi}^k_{j}(s)}[P^{\widehat{\pi}_{-i}^k}_{j,s,a_i}]$, and $\widetilde{P} = \mathbb{E}_{a_i \sim \widehat{\pi}^k_{i,j}(s)}[P_{i,j,s,a_i}^{\widehat{\pi}^k, \overline{V}}]$, which satisfies
\begin{align}
    \|P' - \widetilde{P}\|_1 & = \left\|\mathbb{E}_{a_i \sim \widehat{\pi}^k_{i,j}(s)}[P^{\widehat{\pi}_{-i}^k}_{j,s,a_i}] - \mathbb{E}_{a_i \sim \widehat{\pi}^k_{i,j}(s)}[P_{i,j,s,a_i}^{\widehat{\pi}^k, \overline{V}}] \right\|_1 \notag \\
    & \leq \mathbb{E}_{a_i \sim \widehat{\pi}^k_{i,j}(s)} \Big\| P^{\widehat{\pi}_{-i}^k}_{j,s,a_i} - \mathrm{argmin}_{\mathcal{P} \in \mathcal{U}^{\sigma_i}\left(P_{j,s,a_i}^{\widehat{\pi}^k_{-i}}\right)} \mathcal{P} \mathbb{E}_{\pi\sim \zeta}\left[\overline{V}_{i,j+1}^{\pi,\ror_i}\right] \Big\|_1 \leq \ror_i.
\end{align}
The above matrix notations can be referred to $\underline{P}_{j}^{\widehat{\pi}^k}$ (cf~\ref{sec:matrix_notation}) and $\underline{P}_{i,j}^{\widehat{\pi}^k, \overline{V}}$ (cf.~\eqref{eq:matrix-notations-A-C-0}).

Plugging \eqref{eq:tv-first-C2} back to \eqref{eq:key-concentration-bound2} yields
\begin{align}
    \mathcal{B}_2&=\sum_{j=h}^H2\sum_{k=1}^K\alpha_k^K\sqrt{\frac{\log\left(\frac{18S\sum_{i=1}^nA_iKNnH}{\delta}\right)}{N}}\left<d_h^j,\sqrt{\left|\mathsf{Var}_{\underline{P}_{i,j}^{\widehat{\pi}^k,\overline{V}}}\left(\mathbb{E}_{\pi \sim \zeta}\left[\overline{V}_{i,j+1}^{\pi,\ror_i}\right]\right)-\mathsf{Var}_{\underline{P}_{j}^{\widehat{\pi}^k}}\left(\mathbb{E}_{\pi \sim \zeta}\left[\overline{V}_{i,j+1}^{\pi,\ror_i}\right]\right)\right|}\right>\notag\\
   &\leq\sum_{j=h}^H2\sum_{k=1}^K\alpha_k^K\sqrt{\frac{\log\left(\frac{18S\sum_{i=1}^nA_iKNnH}{\delta}\right)}{N}}\left<d_h^j,\sqrt{\min\left\{\frac{1}{\sigma_i},H\right\}}1\right>\notag\\
   &\leq2\sqrt{\frac{H^2\min\left\{\frac{1}{\sigma_i},H\right\}\log\left(\frac{18S\sum_{i=1}^nA_iKNnH}{\delta}\right)}{N}}\label{eq:key-concentration-bound2-solve2}
\end{align}

Consequently, combining the results in \eqref{eq:key-concentration-bound2-solve1} and \eqref{eq:key-concentration-bound2-solve2} and inserting back to \eqref{eq:key-concentration-bound2} gives
\begin{align}
    &\mathbb{E}_{\pi \sim \zeta}\left[V_{i,h}^{\pi,\ror_i}(s)\right]-\mathbb{E}_{\pi \sim \zeta}\left[\overline{V}_{i,h}^{\pi,\ror_i}(s)\right]\notag\\
   &\leq\frac{H\log\left(\frac{18S\sum_{i=1}^nA_iKNnH}{\delta}\right)}{N}+c_r\sqrt{\frac{H^2\log(\frac{KSnH}{\delta})}{K}}+\mathcal{B}_1+\mathcal{B}_2\notag\\
   &\leq\frac{H\log\left(\frac{18S\sum_{i=1}^nA_iKNnH}{\delta}\right)}{N}+c_r\sqrt{\frac{H^2\log(\frac{KSnH}{\delta})}{K}}+8\sqrt{\frac{H^2\min\left\{\frac{1}{\sigma_i},H\right\}\log\left(\frac{18S\sum_{i=1}^nA_iKNnH}{\delta}\right)}{N}}\notag\\
   &\leq c_r\sqrt{\frac{H^2\log(\frac{KSnH}{\delta})}{K}}+12\sqrt{\frac{H^2\min\left\{\frac{1}{\sigma_i},H\right\}\log\left(\frac{18S\sum_{i=1}^nA_iKNnH}{\delta}\right)}{N}},\label{eq:upper-final1}
\end{align}
where the last inequality holds by taking $N\geq 4H^2\log\left(\frac{18S\sum_{i=1}^nA_iKNnH}{\delta}\right)$.

\end{itemize}

\paragraph{Step 4: controlling the second term in \eqref{eq:recursion_decomposition_target}.}
To do so, similar to \eqref{eq:defn-of-d}, we define
\begin{align}
    \label{eq:defn-of-d-2}
	w_{h}^h = e_s \quad \text{and} \quad \big( w_h^j \big)^\top = e_s^\top \left[\prod_{r=h}^{j-1}\left(\sum_{k=1}^K\alpha_k^K\underline{P}_{i,r}^{\widehat{\pi}^k,V}\right)\right], \quad \forall j =h+1, \cdots, H, 
\end{align}
and observe that 
\begin{align}
w_{j}^h \cdot 1 = 1 , \quad \forall j= h, h+ 1, \cdots, H. \label{eq:property-wh}
\end{align}
With the above notations in mind, following the routine of \eqref{eq:first-term-sum-control} gives: for any $s\in\mathcal{S}$,
\begin{align}
    &\mathbb{E}_{\pi \sim \zeta}\left[V^{\pi,\ror_i}_{i,h}(s)\right]-\mathbb{E}_{\pi \sim \zeta}\left[\overline{V}^{\pi,\ror_i}_{i,h}(s)\right]\notag\\&\leq\sum_{j=h}^H\bigg<w_h^j,2\sum_{k=1}^K\alpha_k^K\sqrt{\frac{\log\left(\frac{18S\sum_{i=1}^nA_iKNnH}{\delta}\right)}{N}}\sqrt{\mathsf{Var}_{\underline{P}_{j}^{\widehat{\pi}^k}}\left(\mathbb{E}_{\pi \sim \zeta}\left[\overline{V}_{i,j+1}^{\pi,\ror_i}\right]\right)} \notag \\
    &\quad  +\frac{\log\left(\frac{18S\sum_{i=1}^nA_iKNnH}{\delta}\right)}{N}1+c_r\sqrt{\frac{\log(\frac{KSnH}{\delta})}{K}}1\bigg>\notag\\
   &\leq\frac{H\log\left(\frac{18S\sum_{i=1}^nA_iKNnH}{\delta}\right)}{N}+c_r\sqrt{\frac{H^2\log(\frac{KSnH}{\delta})}{K}}\notag\\
    &\quad +\sum_{j=h}^H\left<w_h^j,2\sum_{k=1}^K\alpha_k^K\sqrt{\frac{\log\left(\frac{18S\sum_{i=1}^nA_iKNnH}{\delta}\right)}{N}}\sqrt{\mathsf{Var}_{\underline{P}_{j}^{\widehat{\pi}^k}}\left(\mathbb{E}_{\pi \sim \zeta}\left[\overline{V}_{i,j+1}^{\pi,\ror_i}\right]\right)}\right> \label{eq:variance_bound_term_2},
\end{align}
where the first inequality holds by \eqref{eq:matrix_multiply_recursion_reverse} and  \eqref{eq:bound-of-u-ih}.
Furthermore, we can further decompose the above expression as follows:
\begin{align}
    &\mathbb{E}_{\pi \sim \zeta}\left[V^{\pi,\ror_i}_{i,h}(s)\right]-\mathbb{E}_{\pi \sim \zeta}\left[\overline{V}^{\pi,\ror_i}_{i,h}(s)\right]\notag\\
 &    \overset{(\mathrm{i})}{\leq} \frac{H\log\left(\frac{18S\sum_{i=1}^nA_iKNnH}{\delta}\right)}{N}+ c_r\sqrt{\frac{H^2\log(\frac{KSnH}{\delta})}{K}} \notag \\
 &+ \underbrace{\sum_{j=h}^H\left<w_h^j,2\sum_{k=1}^K\alpha_k^K\sqrt{\frac{\log\left(\frac{18S\sum_{i=1}^nA_iKNnH}{\delta}\right)}{N}}\sqrt{\mathsf{Var}_{\underline{P}_{i,j}^{\widehat{\pi}^k,V}}\left(\mathbb{E}_{\pi \sim \zeta}\left[V_{i,j+1}^{\pi,\ror_i}\right]\right)}\right>}_{\mathcal{B}_3}\notag\\
    &\quad +\underbrace{\sum_{j=h}^H\left<w_h^j,2\sum_{k=1}^K\alpha_k^K\sqrt{\frac{\log\left(\frac{18S\sum_{i=1}^nA_iKNnH}{\delta}\right)}{N}}\sqrt{\left|\mathsf{Var}_{\underline{P}_{j}^{\widehat{\pi}^k}}\left(\mathbb{E}_{\pi \sim \zeta}\left[V_{i,j+1}^{\pi,\ror_i}\right]\right)-\mathsf{Var}_{\underline{P}_{i,j}^{\widehat{\pi}^k,V}}\left(\mathbb{E}_{\pi \sim \zeta}\left[V_{i,j+1}^{\pi,\ror_i}\right]\right)\right|}\right>}_{\mathcal{B}_4}\notag\\
    &\quad +\underbrace{\sum_{j=h}^H\left<w_h^j,2\sum_{k=1}^K\alpha_k^K\sqrt{\frac{\log\left(\frac{18S\sum_{i=1}^nA_iKNnH}{\delta}\right)}{N}}\sqrt{\mathsf{Var}_{\underline{P}_{j}^{\widehat{\pi}^k}}\left(\mathbb{E}_{\pi \sim \zeta}\left[\overline{V}_{i,j+1}^{\pi,\ror_i}\right]-\mathbb{E}_{\pi \sim \zeta}\left[V_{i,j+1}^{\pi,\ror_i}\right]\right)}\right>}_{\mathcal{B}_5}
    \label{eq:key-concentration-bound3}
\end{align}
where (i) holds due to the triangle inequality and the fundamental inequality $\sqrt{\mathsf{Var}_P(V+V')} \leq \sqrt{\mathsf{Var}_P(V)} + \sqrt{\mathsf{Var}_P(V')}$ for any transition kernel \( P \in \mathbb{R}^S \) and vectors \( V, V' \in \mathbb{R}^S \).

In the following, we will control the three main terms \( \mathcal{B}_3, \mathcal{B}_4, \mathcal{B}_5 \) separately.

\begin{itemize}
\item Controlling $\mathcal{B}_3$.
To begin with, we observe that 
\begin{align*}
    \mathcal{B}_3&=\sum_{j=h}^H\left<w_h^j,2\sum_{k=1}^K\alpha_k^K\sqrt{\frac{\log\left(\frac{18S\sum_{i=1}^nA_iKNnH}{\delta}\right)}{N}}\sqrt{\mathsf{Var}_{\underline{P}_{i,j}^{\widehat{\pi}^k,V}}\left(\mathbb{E}_{\pi \sim \zeta}\left[V_{i,j+1}^{\pi,\ror_i}\right]\right)}\right>\\
   &\leq2\sqrt{\frac{\log\left(\frac{18S\sum_{i=1}^nA_iKNnH}{\delta}\right)}{N}}\sum_{j=h}^H\left<w_h^j,\sqrt{\mathsf{Var}_{\sum_{k=1}^K\alpha_k^K\underline{P}_{i,j}^{\widehat{\pi}^k,V}}\left(\mathbb{E}_{\pi \sim \zeta}\left[V_{i,j+1}^{\pi,\ror_i}\right]\right)}\right> \notag \\
   &\leq2\sqrt{\frac{\log\left(\frac{18S\sum_{i=1}^nA_iKNnH}{\delta}\right)}{N}}\sqrt{H\sum_{j=h}^H\left<w_h^j,\mathsf{Var}_{\sum_{k=1}^K\alpha_k^K\underline{P}_{i,j}^{\widehat{\pi}^k,V}}\left(\mathbb{E}_{\pi \sim \zeta}\left[V_{i,j+1}^{\pi,\ror_i}\right]\right)\right>}
\end{align*}
where the first inequality holds by applying the two inequalities in Lemma~\ref{lm:weight_variance} subsequently, and the last line arises from Cauchy-Schwartz inequality. 

Before proceeding, we introduce the following lemma for essential terms in $\mathcal{B}_3$.
\begin{lemma}\label{lem:key-lemma-reduce-H-2}
     For any joint policy $\pi$, we have for all $(h,i)\in [H] \times [n]$:
    \begin{align}
     \sum_{j=h}^H \left< w_h^j, \mathsf{Var}_{\sum_{k=1}^K\alpha_k^K\underline{P}_{i,j}^{\widehat{\pi}^k, V}}\left(\mathbb{E}_{\pi \sim \zeta}\left[V_{i,j+1}^{\pi,\ror_i}\right]\right) \right>  \leq 3H \left(\max_{s\in\mathcal{S}}\mathbb{E}_{\pi \sim \zeta}\left[V_{i,h}^{\pi,\ror_i}(s)\right] - \min_{s\in\mathcal{S}}\mathbb{E}_{\pi \sim \zeta}\left[V_{i,h}^{\pi,\ror_i}(s)\right]\right).
    \end{align}
\end{lemma}
\begin{proof}
    See Appendix~\ref{proof:lem:key-lemma-reduce-H-2}.
    \end{proof}

Then applying Lemma~\ref{lem:key-lemma-reduce-H-2} and Lemma~\ref{lemma:pnorm-key-value-range_A_and_C} to the above expression yields
\begin{align}
    \mathcal{B}_3&\leq2\sqrt{\frac{\log\left(\frac{18S\sum_{i=1}^nA_iKNnH}{\delta}\right)}{N}}\sqrt{H\sum_{j=h}^H\left<w_h^j,\mathsf{Var}_{\sum_{k=1}^K\alpha_k^K\underline{P}_{i,j}^{\widehat{\pi}^k,V}}\left(\mathbb{E}_{\pi \sim \zeta}\left[V_{i,j+1}^{\pi,\ror_i}\right]\right)\right>}\notag\\
   &\leq2\sqrt{\frac{\log\left(\frac{18S\sum_{i=1}^nA_iKNnH}{\delta}\right)}{N}}\sqrt{3H^2 \left(\max_{s\in\mathcal{S}}\mathbb{E}_{\pi \sim \zeta}\left[V_{i,h}^{\pi, \sigma_i}(s)\right] - \min_{s\in\mathcal{S}}\mathbb{E}_{\pi \sim \zeta}\left[V_{i,h}^{\pi, \sigma_i}(s)\right)\right]}\notag\\
   &\leq 4\sqrt{\frac{H^2\min\left\{1/\sigma_i,H\right\}\log\left(\frac{18S\sum_{i=1}^nA_iKNnH}{\delta}\right)}{N}}.\label{eq:key-concentration-bound2-solve2-B3}
\end{align}

\item Controlling $\mathcal{B}_4$ and $\mathcal{B}_5$.
With similar analysis as Lemma~\ref{eq:extra-lemma1}, we have the following lemma:
\begin{lemma}\label{eq:extra-lemma1_A_and_C_true_value}
            For any joint policy $\pi$, transition kernel $P' \in \mathbb{R}^S$, and any $\widetilde{P} \in \mathbb{R}^S$ such that $\widetilde{P} \in \mathcal{U}^{\sigma_i}\left(P^\prime\right)$, the following bounds are established for all $(i,h)\in[n]\times[H]$:
            \begin{align*}
                \left| \mathsf{Var}_{P'}\left(\mathbb{E}_{\pi \sim \zeta}\left[ V_{i,h}^{\pi,\ror_i}\right]\right) - \mathsf{Var}_{\widetilde{P}}\left(\mathbb{E}_{\pi \sim \zeta}\left[ V_{i,h}^{\pi,\ror_i}\right]\right)\right| \leq \min \left\{\frac{1}{\sigma_i}, H-h+1 \right\}.
            \end{align*}
            \end{lemma}

First, following the same pipeline of \eqref{eq:tv-first-C2}, we have
\begin{align}
\left|\mathsf{Var}_{\underline{P}_{j}^{\widehat{\pi}^k}}\left(\mathbb{E}_{\pi \sim \zeta}\left[V_{i,j+1}^{\pi,\ror_i}\right]\right)-\mathsf{Var}_{\underline{P}_{i,j}^{\widehat{\pi}^k,V}}\left(\mathbb{E}_{\pi \sim \zeta}\left[V_{i,j+1}^{\pi,\ror_i}\right]\right)\right| \leq  \min \left\{\frac{1}{\sigma_i}, H-h+1 \right\} 1, \label{eq:tv-first-C3}
\end{align}
which directly leads to 
\begin{align}
    \mathcal{B}_4&=\sum_{j=h}^H\left<w_h^j,2\sum_{k=1}^K\alpha_k^K\sqrt{\frac{\log\left(\frac{18S\sum_{i=1}^nA_iKNnH}{\delta}\right)}{N}}\sqrt{\left|\mathsf{Var}_{\underline{P}_{j}^{\widehat{\pi}^k}}\left(\mathbb{E}_{\pi \sim \zeta}\left[V_{i,j+1}^{\pi,\ror_i}\right]\right)-\mathsf{Var}_{\underline{P}_{i,j}^{\widehat{\pi}^k,V}}\left(\mathbb{E}_{\pi \sim \zeta}\left[V_{i,j+1}^{\pi,\ror_i}\right]\right)\right|}\right>\notag\\
   &\leq\sum_{j=h}^H\left<w_h^j,2\sum_{k=1}^K\alpha_k^K\sqrt{\frac{\log\left(\frac{18S\sum_{i=1}^nA_iKNnH}{\delta}\right)}{N}}\sqrt{\min\left\{\frac{1}{\sigma_i},H\right\}}1\right>\notag\\
    &=2\sqrt{\frac{H^2\min\{1/\sigma_i,H\}\log\left(\frac{18S\sum_{i=1}^nA_iKNnH}{\delta}\right)}{N}}. \label{eq:key-concentration-bound2-solve2-B4}
\end{align}
Then the remainder of the proof shall focus on $\mathcal{B}_5$. Recalling the definition in \eqref{eq:key-concentration-bound3}, one has
\begin{align}
    \mathcal{B}_5&=\sum_{j=h}^H\left<w_h^j,2\sum_{k=1}^K\alpha_k^K\sqrt{\frac{\log\left(\frac{18S\sum_{i=1}^nA_iKNnH}{\delta}\right)}{N}}\sqrt{\mathsf{Var}_{\underline{P}_{j}^{\widehat{\pi}^k}}\left(\mathbb{E}_{\pi \sim \zeta}\left[\overline{V}_{i,j+1}^{\pi,\ror_i}\right]-\mathbb{E}_{\pi \sim \zeta}\left[V_{i,j+1}^{\pi,\ror_i}\right]\right)}\right>\notag\\
   &\leq\sum_{j=h}^H\left<w_h^j,2\sum_{k=1}^K\alpha_k^K\sqrt{\frac{\log\left(\frac{18S\sum_{i=1}^nA_iKNnH}{\delta}\right)}{N}}\sqrt{\left\lVert\mathsf{Var}_{\underline{P}_{j}^{\widehat{\pi}^k}}\left(\mathbb{E}_{\pi \sim \zeta}\left[\overline{V}_{i,j+1}^{\pi,\ror_i}\right]-\mathbb{E}_{\pi \sim \zeta}\left[V_{i,j+1}^{\pi,\ror_i}\right]\right)\right\rVert_\infty}1\right>\notag\\
   &\leq2\sqrt{\frac{\log\left(\frac{18S\sum_{i=1}^nA_iKNnH}{\delta}\right)}{N}}\sum_{j=h}^H\left<w_h^j,\left\lVert\mathbb{E}_{\pi \sim \zeta}\left[\overline{V}_{i,j+1}^{\pi,\ror_i}\right]-\mathbb{E}_{\pi \sim \zeta}\left[V_{i,j+1}^{\pi,\ror_i}\right]\right\rVert_\infty 1\right>\notag\\
    & \leq 2\sqrt{\frac{H^2\log\left(\frac{18S\sum_{i=1}^nA_iKNnH}{\delta}\right)}{N}}\max_{h\leq j\leq H}\left\lVert\mathbb{E}_{\pi \sim \zeta}\left[\overline{V}_{i,j+1}^{\pi,\ror_i}\right]-\mathbb{E}_{\pi \sim \zeta}\left[V_{i,j+1}^{\pi,\ror_i}\right]\right\rVert_\infty.\label{eq:key-concentration-bound2-solve2-B5}
\end{align}
Summing up the results in \eqref{eq:key-concentration-bound2-solve2-B3}, \eqref{eq:key-concentration-bound2-solve2-B4}, and \eqref{eq:key-concentration-bound2-solve2-B5} and inserting back to \eqref{eq:key-concentration-bound3}, we conclude that
\begin{align}
    &\mathbb{E}_{\pi \sim \zeta}\left[V^{\pi,\ror_i}_{i,h}(s)\right]-\mathbb{E}_{\pi \sim \zeta}\left[\overline{V}^{\pi,\ror_i}_{i,h}(s)\right]\notag\\
   &\leq \frac{H\log\left(\frac{18S\sum_{i=1}^nA_iKNnH}{\delta}\right)}{N}+c_r\sqrt{\frac{H^2\log\left(\frac{KSnH}{\delta}\right)}{K}}+\mathcal{B}_3+\mathcal{B}_4+\mathcal{B}_5\notag\\
   &\leq \frac{H\log\left(\frac{18S\sum_{i=1}^nA_iKNnH}{\delta}\right)}{N}+c_r\sqrt{\frac{H^2\log\left(\frac{KSnH}{\delta}\right)}{K}}+6\sqrt{\frac{H^2\min\{1/\sigma_i,H\}\log\left(\frac{18S\sum_{i=1}^nA_iKNnH}{\delta}\right)}{N}}\notag \\
    &\quad +2\sqrt{\frac{H^2\log\left(\frac{18S\sum_{i=1}^nA_iKNnH}{\delta}\right)}{N}}\max_{h\leq j\leq H}\left\lVert\mathbb{E}_{\pi \sim \zeta}\left[\overline{V}_{i,j+1}^{\pi,\ror_i}\right]-\mathbb{E}_{\pi \sim \zeta}\left[V_{i,j+1}^{\pi,\ror_i}\right]\right\rVert_\infty. \label{eq:upper-final2}
\end{align}

\end{itemize}
\paragraph{Step 5: summing up the results to bound term A and C.}
Inserting \eqref{eq:upper-final1} and \eqref{eq:upper-final2} back into \eqref{eq:recursion_decomposition_target}, we observe that
\begin{align*}
    &\left|\mathbb{E}_{\pi \sim \zeta}\left[V^{\pi,\ror_i}_{i,h}\right]-\mathbb{E}_{\pi \sim \zeta}\left[\overline{V}^{\pi,\ror_i}_{i,h}\right]\right|\\
   &\leq\max\left\{\mathbb{E}_{\pi \sim \zeta}\left[V^{\pi,\ror_i}_{i,h}\right]-\mathbb{E}_{\pi \sim \zeta}\left[\overline{V}^{\pi,\ror_i}_{i,h}\right],\mathbb{E}_{\pi \sim \zeta}\left[\overline{V}^{\pi,\ror_i}_{i,h}\right]-\mathbb{E}_{\pi \sim \zeta}\left[V^{\pi,\ror_i}_{i,h}\right]\right\}\\
   &\leq\max\Bigg\{c_r\sqrt{\frac{H^2\log\left(\frac{KSnH}{\delta}\right)}{K}}1+12\sqrt{\frac{H^2\min\left\{\frac{1}{\sigma_i},H\right\}\log\left(\frac{18S\sum_{i=1}^nA_iKNnH}{\delta}\right)}{N}}1,\frac{H\log\left(\frac{18S\sum_{i=1}^nA_iKNnH}{\delta}\right)}{N}1\\
    &\quad +2\sqrt{\frac{H^2\log\left(\frac{18S\sum_{i=1}^nA_iKNnH}{\delta}\right)}{N}}\max_{h\leq j\leq H}\left\lVert\mathbb{E}_{\pi \sim \zeta}\left[\overline{V}_{i,j+1}^{\pi,\ror_i}\right]-\mathbb{E}_{\pi \sim \zeta}\left[V_{i,j+1}^{\pi,\ror_i}\right]\right\rVert_\infty 1+c_r\sqrt{\frac{H^2\log\left(\frac{KSnH}{\delta}\right)}{K}}1\\
    &\quad +6\sqrt{\frac{H^2\min\{\frac{1}{\sigma_i},H\}\log\left(\frac{18S\sum_{i=1}^nA_iKNnH}{\delta}\right)}{N}}1\Bigg\},
\end{align*}
which indicates that
\begin{align}
    &\max_{h\in[H]}\left\lVert\mathbb{E}_{\pi \sim \zeta}\left[ V^{\pi,\ror_i}_{i,h}\right]-\mathbb{E}_{\pi \sim \zeta}\left[\overline{V}^{\pi,\ror_i}_{i,h}\right]\right\rVert_\infty\\
   &\leq c_r\sqrt{\frac{H^2\log(\frac{KSnH}{\delta})}{K}}+12\sqrt{\frac{H^2\min\left\{\frac{1}{\sigma_i},H\right\}\log\left(\frac{18S\sum_{i=1}^nA_iKNnH}{\delta}\right)}{N}}+\frac{H\log\left(\frac{18S\sum_{i=1}^nA_iKNnH}{\delta}\right)}{N}\notag\\
    &\quad +2\sqrt{\frac{H^2\log\left(\frac{18S\sum_{i=1}^nA_iKNnH}{\delta}\right)}{N}}\max_{h\in [H]}\left\lVert\mathbb{E}_{\pi \sim \zeta}\left[\overline{V}_{i,h+1}^\pi\right]-\mathbb{E}_{\pi \sim \zeta}\left[V_{i,H+1}^{\pi,\ror_i}\right]\right\rVert_\infty\notag\\
  &  \overset{\mathrm{(i)}}{\leq} c_r\sqrt{\frac{H^2\log(\frac{KSnH}{\delta})}{K}}+12\sqrt{\frac{H^2\min\left\{\frac{1}{\sigma_i},H\right\}\log\left(\frac{18S\sum_{i=1}^nA_iKNnH}{\delta}\right)}{N}}\notag\\
    &\quad +\frac{H\log\left(\frac{18S\sum_{i=1}^nA_iKNnH}{\delta}\right)}{N}+\frac{1}{2}\max_{h\in [H]}\left\lVert\mathbb{E}_{\pi \sim \zeta}\left[\overline{V}_{i,h}^{\pi,\ror_i}\right]-\mathbb{E}_{\pi \sim \zeta}\left[V_{i,h}^{\pi,\ror_i}\right]\right\rVert_\infty\notag\\
   &\leq2c_r\sqrt{\frac{H^2\log(\frac{KSnH}{\delta})}{K}}+24\sqrt{\frac{H^2\min\left\{\frac{1}{\sigma_i},H\right\}\log\left(\frac{18S\sum_{i=1}^nA_iKNnH}{\delta}\right)}{N}}+\frac{2H\log\left(\frac{18S\sum_{i=1}^nA_iKNnH}{\delta}\right)}{N} \label{eq:upper_bound_A_and_C_vanilla}
\end{align}
where (i) holds when $N\geq 4H^2\log\left(\frac{18S\sum_{i=1}^nA_iNnH}{\delta}\right)$, and invoking the basic fact that $\mathbb{E}_{\pi \sim \zeta}\left[\overline{V}_{i,H+1}^\pi\right]=\mathbb{E}_{\pi \sim \zeta}\left[V_{i,H+1}^{\pi,\ror_i}\right]=0.$ Finally, applying \eqref{eq:upper_bound_A_and_C_vanilla} and setting either \( \widehat{\pi}_h^k = \tilde{\pi}_i^\star \times \pi_{-i,h}^k \) or \( \widehat{\pi}_h^k = \pi_h^k \), we arrive at the following upper bounds on term $A$ and term $C$:
\begin{align}
    &A = \mathbb{E}_{\pi \sim \widehat{\xi}}\left[V^{\star,\pi_{-i},\ror_i}_{i,h}\right]-\mathbb{E}_{\pi \sim \widehat{\xi}}\left[\overline{V}_{i,h}^{\widetilde{\pi}_i^\star,\pi_{-i}}\right]\notag\\
    &\leq 2c_r\sqrt{\frac{H^2\log(\frac{KSnH}{\delta})}{K}}1+24\sqrt{\frac{H^2\min\left\{\frac{1}{\sigma_i},H\right\}\log\left(\frac{18S\sum_{i=1}^nA_iKNnH}{\delta}\right)}{N}}1+\frac{2H\log\left(\frac{18S\sum_{i=1}^nA_iKNnH}{\delta}\right)}{N}1, \label{eq:final_upper_bound_term_A}\\
    &C = \mathbb{E}_{\pi \sim \widehat{\xi}}\left[\overline{V}^{\pi,\ror_i}_{i,h}\right]-\mathbb{E}_{\pi \sim \widehat{\xi}}\left[V_{i,h}^{\pi,\ror_i}\right]\notag\\
    &\leq 2c_r\sqrt{\frac{H^2\log(\frac{KSnH}{\delta})}{K}}1+24\sqrt{\frac{H^2\min\left\{\frac{1}{\sigma_i},H\right\}\log\left(\frac{18S\sum_{i=1}^nA_iKNnH}{\delta}\right)}{N}}1+\frac{2H\log\left(\frac{18S\sum_{i=1}^nA_iKNnH}{\delta}\right)}{N}1. \label{eq:final_upper_bound_term_C}
\end{align}

\paragraph{Step 6: summing up the results to complete the proof.}
Combining the bounds for A in \eqref{eq:final_upper_bound_term_A}, B in \eqref{eq:final_upper_bound_term_B}, and C in \eqref{eq:final_upper_bound_term_C}, the term of interest in \eqref{eq:error_decomposition} can be controlled by
\begin{align*}
    \mathbb{E}_{\pi \sim \widehat{\xi}}\left[V_{i,1}^{\star,\pi_{-i},\ror_i}\right]-\mathbb{E}_{\pi \sim \widehat{\xi}}\left[V_{i,1}^{\pi}\right]
   &\leq 36c_{\mathsf{b}}\sqrt{\frac{H^3\log^3(\frac{KS\sum_{i=1}^nA_i}{\delta})}{K}}1+4c_r\sqrt{\frac{H^2\log(\frac{KSnH}{\delta})}{K}}1 \notag \\
   &\quad +\frac{4H\log\left(\frac{18S\sum_{i=1}^nA_iKNnH}{\delta}\right)}{N}1  +48\sqrt{\frac{H^2\min\left\{\frac{1}{\sigma_i},H\right\}\log\left(\frac{18S\sum_{i=1}^nA_iKNnH}{\delta}\right)}{N}}1.
\end{align*}
Therefore, there exists a constant $C$, such that when $N$ and $K$ satisfies:
\begin{align*}
    N\geq C H^2\min\left\{\frac{1}{\min_{1\leq i\leq n}\sigma_i},H\right\}\log\left(\frac{18S\sum_{i=1}^nA_iKNnH}{\delta}\right)\frac{1}{\epsilon^2},\quad K\geq C H^3\log^3\left(\frac{KS\sum_{i=1}^nA_iH}{\delta}\right)\frac{1}{\epsilon^2}
\end{align*}
we can achieve $\max_{i\in[n]}\mathbb{E}_{\pi \sim \widehat{\xi}}\left[V_{i,1}^{\star,\pi_{-i},\ror_i}\right]-\mathbb{E}_{\pi \sim \widehat{\xi}}\left[V_{i,1}^{\pi}\right]\leq \epsilon\cdot 1$ with probability at least $1-\delta$.
Consequently, it holds as long as the total number of samples satisfies
\begin{align*}
    N_{\mathsf{all}}=HS\sum_{i=1}^n A_i KN=\tilde{\mathcal{O}}\left(\frac{S\sum_{i=1}^n A_i H^6}{\epsilon^4}\min\left\{\frac{1}{\min_{1\leq i\leq n}\sigma_i},H\right\}\right).
\end{align*}

 \subsection{Proof of auxiliary lemmas}
\subsubsection{Proof of Lemma~\ref{lem:UCB}}
\label{sec:proof-lemma:UCB}
We will prove Lemma~\ref{lem:UCB} by the following induction argument. To begin with, for the base case, we consider the final time step $H+1$, where we have
\begin{align*}
   \forall i\in [n]: \quad \widehat{V}_{i,H+1}=\mathbb{E}_{\pi\sim \widehat{\xi}}\left[\overline{V}_{i,H+1}^{\star,\pi_{-i},\ror_i}\right]=0
\end{align*}
by definition. To continue, we assume that for subsequent steps, the following holds
\begin{align*}
    \forall (j,i)\in [h+1,\cdots,H] \times [n]: \quad \widehat{V}_{i,j}\geq\mathbb{E}_{\pi\sim \widehat{\xi}}\left[\overline{V}_{i,j}^{\star,\pi_{-i},\ror_i}\right].
\end{align*}
So the rest of the proof is to verify for the time step $h$ to complete the induction argument.

Towards this, for any $s\in\cS$ at time step $h$, let $l_k = -q_{i,h}^k(s,\cdot) \in \mathbb{R}^{A_i}$ for all $k\geq 1$, then the update rule of Algorithm~\ref{alg:summary} can be viewed as the FTRL algorithm applied to the loss vectors $\{l_k\}_{k\in[K]}$. 
Specifically, we invoke Theorem~\ref{thm:FTRL-refined} with $l_k = -q_{i,h}^k(s,\cdot)$ and arrive at
\begin{align}
    &\max_{a_i\in\mathcal{A}_i}\sum_{k=1}^K\alpha_k^Kq_{i,h}^k(s,a_i)-\sum_{k=1}^K\alpha_k^K\left<\pi_{i,h}^k(s),q_{i,h}^k(s,\cdot)\right>\notag\\
     &= \sum_{k=1}^K \alpha_k^K \left< \pi_{i,h}^k(s), l_k \right> - \min_{a \in \cA} \left[ \sum_{k=1}^K \alpha_k^K l_k(a) \right] \notag\\
    &\overset{(\mathrm{i})}{\leq}\frac{5}{3}\sum_{k=2}^{K}\alpha_{k}^{K}\frac{\eta_{k}\alpha_{k}}{1-\alpha_{k}}\mathsf{Var}_{\pi_{i,h}^{k}(s)}\Big(q_{i,h}^{k}(s,\cdot)\Big)+\frac{\log A_i}{\eta_{K+1}}+\tau_{i,h} \notag \\
    & \leq \underbrace{\frac{5}{3}\sum_{k=2}^{K/2}\frac{\big(2c_{\alpha}\big)^{1.5}\log^{2}K}{\sqrt{kH}}\alpha_{k}^{K}\mathsf{Var}_{\pi_{i,h}^{k}(s)}\Big(q_{i,h}^{k}(s,\cdot)\Big)}_{\mathcal{C}_1} \notag \\
 & \quad +\frac{20}{3}\sum_{k=K/2+1}^{K}\alpha_{k}^{K}\sqrt{\frac{c_{\alpha}\log^{2}K}{KH}}\,\mathsf{Var}_{\pi_{i,h}^{k}(s)}\Big(q_{i,h}^{k}(s,\cdot)\Big)+\underbrace{\frac{\log A_i}{\eta_{K+1}}}_{\mathcal{C}_2}+\tau_{i,h,s},\label{eq:Q-upper-135}
\end{align}
where (i) holds by applying Lemma~\ref{lem:weight-2} and Theorem~\ref{thm:FTRL-refined} with $\tau_{i,h,s}$ defined as 
% % 
\begin{align}
	\tau_{i,h,s}&:= \frac{5}{3}\alpha_{1}^{K}\eta_{2}\big\|q_{i,h}^{1}(s,\cdot)\big\|_{\infty}^{2} + \left\{ 3\sum_{k=2}^{K}\alpha_{k}^{K}\frac{\eta_{k}^{2}\alpha_{k}^{2}}{(1-\alpha_{k})^{2}}\big\|q_{i,h}^{k}(s,\cdot)\big\|_{\infty}^{3}\mathbf{1}\bigg(\frac{\eta_{k}\alpha_{k}}{1-\alpha_{k}}\big\|q_{i,h}^{k}(s,\cdot)\big\|_{\infty}>\frac{1}{3}\bigg) \right\} \notag \\
	& \quad +3\alpha_{1}^{K}\eta_{2}^{2}\big\|q_{i,h}^{1}(s,\cdot)\big\|_{\infty}^{3}. 
\end{align}
Here, the last inequality arises from applying $\eta_{k}\alpha_{k} \leq\sqrt{\frac{2c_{\alpha}\log^{2}K}{kH}}$ from Lemma~\ref{lem:weight-2} and \eqref{eq:alpha-properties-2} from Lemma~\ref{lem:weight} to the first term.
To further control the above expression, we first derive upper bound for $\mathcal{C}_1,\mathcal{C}_2$, and $\tau_{i,h,s}$ in \eqref{eq:Q-upper-135} separately.
    \begin{itemize}[topsep=0pt,leftmargin=17pt]
    \setlength{\itemsep}{0pt}
        \item For the term $\mathcal{C}_1$, we have    
        \begin{align}
            &\frac{5}{3}\sum_{k=2}^{K/2}\frac{\alpha_{k}^{K}\big(2c_{\alpha}\big)^{1.5} \log^{2}K}{\sqrt{kH}}\mathsf{Var}_{\pi_{i,h}^{k}(s)}\Big(q_{i,h}^{k}(s,\cdot)\Big)  \leq \frac{5}{3}\cdot (2c_{\alpha}\big)^{1.5} \log^{2}K\cdot  \sum_{k=2}^{K/2}\frac{1}{K^{6}\sqrt{kH}}\mathsf{Var}_{\pi_{i,h}^{k}(s)}\Big(q_{i,h}^{k}(s,\cdot)\Big) \notag\\
             & \leq \frac{5}{3}\cdot (2c_{\alpha}\big)^{1.5} \log^{2}K\sum_{k=2}^{K/2}\frac{1}{K^{6}\sqrt{kH}}\big\|q_{i,h}^{k}(s,\cdot)\big\|_{\infty}^{2}
             \leq \frac{5}{3}\cdot (2c_{\alpha}\big)^{1.5} \log^{2}K \frac{H^{3/2}} {K^{6}}\sum_{k=2}^{K/2}\frac{1}{\sqrt{k}}\notag\\
             & \leq\frac{2 (2c_{\alpha}\big)^{1.5} H^{3/2}\log^{2}K}{K^{6}}\cdot\sqrt{K/2}\le \frac{2 (2c_{\alpha}\big)^{1.5} H^{3/2}\log^{2}K}{K^{5}},
                \label{eq:sum-var-q-plus-beta-135}
            \end{align}	
            where the first inequality follows from the result of Lemma~\ref{lem:weight} that $\alpha_k^K\leq1/K^6$ for all $k\leq K/2$, and the third inequality holds due to the fact $\big\|q_{i,h}^{k}(s,\cdot)\big\|_{\infty}\leq H$.
        \item For the term $\mathcal{C}_2$, we have 
        \begin{align}
            \frac{\log A_i}{\eta_{K+1}} 
             & =  \log A_i \sqrt{\frac{\alpha_{K}H}{\log K}} \le \sqrt{\frac{2c_{\alpha}H \log^{2}A_i}{K}},\label{eq:logA-eta-K-UB-135}
            \end{align}
            where the first equality holds by the definition of $\eta_{K+1}$ in \eqref{eq:policy-update-exponential-alg}, and the last inequality is obtained by \eqref{eq:alpha-properties} in Lemma~\ref{lem:weight}.

        \item For the term $\tau_{i,h,s}$, we first observe that
        \begin{align}
            \frac{\eta_{k}\alpha_{k}}{1-\alpha_{k}}\big\|q_{i,h}^{k}(s,\cdot)\big\|_{\infty} 
             & \leq\frac{\sqrt{\frac{2c_{\alpha}\log^{2}K}{kH}}}{\frac{1}{2c_{\alpha}\log K}}\cdot H  =\sqrt{\frac{8c_{\alpha}^{3}H\log^{4}K}{k}} \leq \frac{1}{3}
                \label{eq:norm-eta-q-UB-123}
            \end{align}
            when $k \geq c_9 H \log^4 K$ for some sufficiently large constant $c_9 > 0$, 
            where the first inequality holds by applying $1 - \alpha_k \geq \frac{1}{2c_\alpha \log K}$, $\eta_{k}\alpha_{k} = \sqrt{\frac{2c_{\alpha}\log^{2}K}{kH}}$ of Lemma~\ref{lem:weight-2} and $\big\|q_{i,h}^{k}(s,\cdot)\big\|_{\infty}\leq H$.
    
            Consequently, recalling the fact that $\eta_2 = \frac{\log K}{\alpha_1 H} = \frac{\log K}{H}$, one has
            \begin{align}
               &\frac{5}{3}\alpha_{1}^{K}\eta_{2}\big\|q_{i,h}^{1}(s,\cdot)\big\|_{\infty}^{2} + \left\{ 3\sum_{k=2}^{K}\alpha_{k}^{K}\frac{\eta_{k}^{2}\alpha_{k}^{2}}{(1-\alpha_{k})^{2}}\big\|q_{i,h}^{k}(s,\cdot)\big\|_{\infty}^{3}\mathbf{1}\bigg(\frac{\eta_{k}\alpha_{k}}{1-\alpha_{k}}\big\|q_{i,h}^{k}(s,\cdot)\big\|_{\infty}>\frac{1}{3}\bigg) \right\} \notag \\
            & \quad +3\alpha_{1}^{K}\eta_{2}^{2}\big\|q_{i,h}^{1}(s,\cdot)\big\|_{\infty}^{3} \notag \\
                 & \leq \frac{5}{3K^{6}}\sqrt{\frac{\log K}{H}}\big\|q_{i,h}^{1}(s,\cdot)\big\|_{\infty}^{2} + \frac{\big(2c_\alpha\log K\big)^{2}}{K^{6}}\left\{ 3\sum_{k=2}^{c_{9}H^2\log^{4}\frac{K}{\delta}}\eta_{k}^{2}\alpha_{k}^{2}\big\|q_{i,h}^{k}(s,\cdot)\big\|_{\infty}^{3}\right\} +\frac{3}{K^{6}}\frac{\log K}{H}\big\|q_{i,h}^{1}(s,\cdot)\big\|_{\infty}^{3} \notag \\
                 & \leq \frac{6 \log K}{\sqrt{H} K^6} \big\|q_{i,h}^{k}(s,\cdot)\big\|_{\infty}^{3} + 3 \frac{\big(2c_\alpha\log K\big)^{2}}{K^{6}}\left\{ \sum_{k=2}^{c_{9}H^2\log^{4}\frac{K}{\delta}}\eta_{k}^{2}\alpha_{k}^{2}\big\|q_{i,h}^{k}(s,\cdot)\big\|_{\infty}^{3}\right\}. \label{eq:xihk-UB-145-initial}
                \end{align}
Here, the penultimate inequality follows from applying the fact $\alpha_k^K \leq \frac{1}{K^6}$ for all $k \leq \frac{K}{2}$ in Lemma~\ref{lem:weight} yields
\begin{align*}
&\frac{5}{3} \alpha_1^K \eta_2 \|q_{i,h}^1(s, \cdot)\|_\infty^2 \leq \frac{5}{3 K^6} \sqrt{\frac{\log K}{H}} \|q_{i,h}^1(s, \cdot)\|_\infty^2,\\
&3 \alpha_1^K \eta_2^2 \|q_{i,h}^1(s, \cdot)\|_\infty^3 \leq \frac{3}{K^6} \frac{\log K}{H} \|q_{i,h}^1(s, \cdot)\|_\infty^3
\end{align*}
and
\begin{align*}
&\left\{ 3 \sum_{k=2}^{K} \alpha_k^K \frac{\eta_k^2 \alpha_k^2}{(1-\alpha_k)^2} \|q_{i,h}^k(s, \cdot)\|_\infty^3 \mathbf{1}\left(\frac{\eta_k \alpha_k}{1-\alpha_k} \|q_{i,h}^k(s, \cdot)\|_\infty > \frac{1}{3}\right) \right\}\\
&\leq 3 \cdot (2c_\alpha \log K)^2 \left\{  \sum_{k=2}^{K} \alpha_k^K \eta_k^2 \alpha_k^2 \|q_{i,h}^k(s, \cdot)\|_\infty^3   \mathbf{1}\left(\frac{\eta_k \alpha_k}{1-\alpha_k} \|q_{i,h}^k(s, \cdot)\|_\infty > \frac{1}{3}\right) \right\}  \\
&\leq 3 \cdot (2c_\alpha \log K)^2\left\{  \sum_{k=2}^{c_9 H \log^4 \frac{K}{\delta}} \alpha_k^K  \eta_k^2 \alpha_k^2 \|q_{i,h}^k(s, \cdot)\|_\infty^3 \right\} \notag \\
&\leq 3 \frac{(2c_\alpha \log K)^2}{K^6} \left\{  \sum_{k=2}^{c_9 H \log^4 \frac{K}{\delta}}   \eta_k^2 \alpha_k^2 \|q_{i,h}^k(s, \cdot)\|_\infty^3 \right\},
\end{align*}
where the first inequality arises from applying $1 - \alpha_k \geq \frac{1}{2c_\alpha \log K}$, the penultimate inequality holds by \eqref{eq:norm-eta-q-UB-123}, and  the last inequality holds by the facts $\alpha_k^K \leq \frac{1}{K^6}$ when $k  \leq c_9 H \log^4 K \leq K$ (see Lemma~\ref{lem:weight}). 

To continue, \eqref{eq:xihk-UB-145-initial} can be further controlled as
 \begin{align}
              &\frac{6 \log K}{\sqrt{H} K^6} \big\|q_{i,h}^{k}(s,\cdot)\big\|_{\infty}^{3} +  3\frac{\big(2c_\alpha\log K\big)^{2}}{K^{6}}\left\{ \sum_{k=2}^{c_{9}H \log^{4}\frac{K}{\delta}}\eta_{k}^{2}\alpha_{k}^{2}\big\|q_{i,h}^{k}(s,\cdot)\big\|_{\infty}^{3}\right\}  \notag \\
                 & \leq\frac{24c_{\alpha}^{3}\log^{4}K}{\sqrt{H} K^{6} }\left\{ \sum_{k=1}^{K}\frac{1}{k}H^{3}\right\} \nonumber\\
                 & \leq\frac{24c_{\alpha}^{3}H^{3}\log^{5}K}{K^{6}}\leq\frac{1}{K^{4}}, \label{eq:xihk-UB-145}
                \end{align}
where the first  inequality follows from the fact $\eta_k \alpha_k = \sqrt{\frac{2c_\alpha \log^2 K}{kH}}$ (see Lemma~\ref{lem:weight-2}) and $\|q_{i,h}^k(s, \cdot)\|_\infty \leq H$, and the last inequality holds by letting $K \geq \sqrt{24c_{\alpha}^{3}H^{3}\log^{5}K}$.
\end{itemize}

Inserting the results of the three terms back to \eqref{eq:Q-upper-135}, we arrive at 
            \begin{align}
                &\max_{a_i\in\mathcal{A}_i}\sum_{k=1}^K\alpha_k^Kq_{i,h}^k(s,a_i)-\sum_{k=1}^K\alpha_k^K\left<\pi_{i,h}^k,q_{i,h}^k(s,\cdot)\right>\notag\\
                & \leq \mathcal{C}_1 + \frac{20}{3} \sum_{k=K/2+1}^{K} \alpha_k^K \sqrt{\frac{c_\alpha \log^2 K}{KH}} \, \mathsf{Var}_{\pi_{i,h}^k(s)}\left( q_{i,h}^k(s, \cdot) \right) + \mathcal{C}_2 + \tau_{i,h,s} \notag \\
               &\leq \frac{2 (2c_{\alpha}\big)^{1.5} H^{3/2}\log^{2}K}{K^{5}}  +\frac{20}{3}\sqrt{\frac{c_{\alpha}\log^{2}K}{KH}}\sum_{k=K/2+1}^{K}\alpha_{k}^{K}\mathsf{Var}_{\pi_{i,h}^{k}(s)}\Big(q_{i,h}^{k}(s,\cdot)\Big)+\sqrt{\frac{2c_{\alpha}H \log^{2}A_i }{K}}+\frac{1}{K^{4}}\notag\nonumber \\
               &\leq10\sqrt{\frac{c_{\alpha}\log^{3}(KA_i)}{KH}}\sum_{k=1}^{K}\alpha_{k}^{K}\mathsf{Var}_{\pi_{i,h}^{k}(s)}\Big(q_{i,h}^{k}(s,\cdot)\Big)+2\sqrt{\frac{c_{\alpha}H\log^{3}(KA_i)}{K}} \label{eq:Q-sa-sum-q-UB-135} \\
                &= \beta_{i,h}(s),
            \end{align}
            where the second inequality follows from inserting the results in \eqref{eq:sum-var-q-plus-beta-135}, \eqref{eq:logA-eta-K-UB-135}, and \eqref{eq:norm-eta-q-UB-123} subsequently, the penultimate inequality holds by letting $K \geq c_9 H \log^4K $, and the last equality is induced by  the definition of $\beta_{i,h}(s)$ in \eqref{eq:bonus-beta} by setting the constant $c_{\mathsf{b}}$ properly.

           Finally, recalling the definition of $\mathbb{E}_{\pi\sim \widehat{\xi}}\left[\overline{V}^{\star,\pi_{-i},\ror_i}_{i,h}(s)\right]$ in \eqref{eq:auxiliary_value_function_3}, we have for all $s\in\mathcal{S}$,
            \begin{align*}
                \mathbb{E}_{\pi\sim \widehat{\xi}}\left[\overline{V}^{\star,\pi_{-i},\ror_i}_{i,h}(s)\right]&=\max_{a_i\in\mathcal{A}_i}\sum_{k=1}^K\alpha_k^K\left[r_{i,h}^k(s,a_i)+\inf_{\mathcal{P}\in\mathcal{U}^{\sigma_i}(P_{i,h,s,a_i}^k)}\mathcal{P\mathbb{E}_{\pi\sim \widehat{\xi}}}\left[\overline{V}_{i,h+1}^{\star,\pi_{-i},\ror_i}\right]\right]\\
               &\leq\max_{a_i\in\mathcal{A}_i}\sum_{k=1}^K\alpha_k^K\left[r_{i,h}^k(s,a_i)+\inf_{\mathcal{P}\in\mathcal{U}^{\sigma_i}(P_{i,h,s,a_i}^k)}\mathcal{P}\widehat{V}_{i,h+1}\right]\\
               &\leq\sum_{k=1}^K\alpha_k^K\mathbb{E}_{a_i\sim\pi_{i,h}^k(s)}\left[r_{i,h}^k(s,a_i)+\inf_{\mathcal{P}\in\mathcal{U}^{\sigma_i}(P_{i,h,s,a_i}^k)}\mathcal{P}\widehat{V}_{i,h+1}\right]+\beta_{i,h}(s)\\
               &\leq\widehat{V}_{i,h}(s),
            \end{align*}
            where the first inequality holds by the induction hypothesis of the step $h+1$, and the last inequality follows from the update rule in line~\ref{eq:line-number-policy-update} in Algorithm~\ref{alg:summary}.
\subsubsection{Proof of Lemma~\ref{lm:term_B_value_vector_comparison}}
\label{proof:lm:term_B_value_vector_comparison}
We will prove the lemma by induction argument. To proceed, for the base case $H+1$, the following fact holds immediately
\begin{align*}
    \widehat{V}_{i,H+1}=\mathbb{E}_{\pi\sim \widehat{\xi}}\left[\overline{V}^{\pi,\ror_i}_{i,H+1}\right]=0
\end{align*}
due to the definition of $\widehat{V}_{i,H+1}$ (see Algorithm~\ref{alg:summary}) and $\mathbb{E}_{\pi\sim \widehat{\xi}}\left[\overline{V}^{\pi,\ror_i}_{i,H+1}\right]$ (see \eqref{eq:auxiliary_value_function_4}).
Then assuming the following induction assumption holds
\begin{align*}
    \widehat{V}_{i,j}\geq \mathbb{E}_{\pi\sim \widehat{\xi}}\left[\overline{V}^{\pi,\ror_i}_{i,j}\right], \quad \forall j=h+1, \cdots, H,
\end{align*}
we will focus on showing that it also holds for time step $h$.
According to the definition of $\mathbb{E}_{\pi\sim \widehat{\xi}}\left[\overline{V}^{\pi,\ror_i}_{i,h}\right]$ in \eqref{eq:auxiliary_value_function_3}, we have for any $s\in\cS$,
\begin{align*}
    \mathbb{E}_{\pi\sim \widehat{\xi}}\left[\overline{V}^{\pi,\ror_i}_{i,h}(s)\right]&=\sum_{k=1}^K\alpha_k^K\mathbb{E}_{a_i\sim\pi_{i,h}^k(s)}\left[r_{i,h}^k(s,a_i)+\inf_{\mathcal{P}\in\mathcal{U}^{\sigma_i}(P_{i,h,s,a_i}^k)}\mathcal{P}\mathbb{E}_{\pi\sim \widehat{\xi}}\left[\overline{V}_{i,h+1}^{\pi,\ror_i}\right]\right]\\
&\leq\sum_{k=1}^K\alpha_k^K\mathbb{E}_{a_i\sim\pi_{i,h}^k(s)}\left[r_{i,h}^k(s,a_i)+\inf_{\mathcal{P}\in\mathcal{U}^{\sigma_i}(P_{i,h,s,a_i}^k)}\mathcal{P}\widehat{V}_{i,h+1}\right]\\
&\leq\sum_{k=1}^K\alpha_k^K\mathbb{E}_{a_i\sim\pi_{i,h}^k(s)}\left[r_{i,h}^k(s,a_i)+\inf_{\mathcal{P}\in\mathcal{U}^{\sigma_i}(P_{i,h,s,a_i}^k)}\mathcal{P}\widehat{V}_{i,h+1}\right]+\beta_{i,h}(s) \notag \\
    &\leq\min\left\{\sum_{k=1}^K\alpha_k^K\mathbb{E}_{a_i\sim\pi_{i,h}^k(s)}\left[r_{i,h}^k(s,a_i)+\inf_{\mathcal{P}\in\mathcal{U}^{\sigma_i}(P_{i,h,s,a_i}^k)}\mathcal{P}\widehat{V}_{i,h+1}\right]+\beta_{i,h}(s),H-h+1\right\}=\widehat{V}_{i,h}(s).
\end{align*}
where the fist inequality is achieved by the induction assumption, the penultimate line holds by $\beta_{i,h}(s) \geq 0$, the last line arises from the fundamental fact $ \mathbb{E}_{\pi\sim \widehat{\xi}}\left[\overline{V}^{\pi,\ror_i}_{i,h}(s)\right]\leq H-h+1$, and then the definition of $\widehat{V}_{i,h}$ in line~\ref{eq:line-number-policy-update} of Algorithm~\ref{alg:summary}.

\subsubsection{Proof of Lemma~\ref{lm:upper_bound_bonus}} 
\label{proof:lm:upper_bound_bonus}
Recall that for all $s\in\mathcal{S}$, bonus term $\beta_{i,h}(s)$ is defined as
\begin{align}
    \beta_{i,h}(s)=c_{\mathsf{b}}\sqrt{\frac{\log^3(\frac{KS\sum_{i=1}^nA_i}{\delta})}{KH}}\sum_{k=1}^K\alpha_k^K\left\{\mathsf{Var}_{\pi_{i,h}^k(\cdot\mid s)}\left(q_{i,h}^k(s,\cdot)\right)+H\right\}.\label{eq:bonus_definition}
\end{align}
For any $k\in[K]$, recalling the definition of $q_{i,h}^k(s,\cdot)$ in \eqref{eq:nvi-iteration} gives
\begin{align}
    &\mathsf{Var}_{\pi_{i,h}^k(\cdot\mid s)}\left(q_{i,h}^k(s,\cdot)\right)\notag\\
   &\leq 2\mathsf{Var}_{\pi_{i,h}^k(\cdot\mid s)}\left(r_{i,h}^k(s,\cdot)\right)+2\mathsf{Var}_{\pi_{i,h}^k(\cdot\mid s)}\left(\widehat{P}_{i,h,s,\cdot}^{\pi_{-i}^k,\widehat{V}}\widehat{V}_{i,h+1}\right)\notag\\
    &\overset{\mathrm{(i)}}{\leq} 2+2\left[\sum_{a_i\in\mathcal{A}_i}\pi_{i,h}^k(a_i\mid s)\widehat{P}_{i,h,s,a_i}^{\pi_{-i}^k,\widehat{V}}\left(\widehat{V}_{i,h+1}\circ\widehat{V}_{i,h+1}\right)-\left(\sum_{a_i\in\mathcal{A}_i}\pi_{i,h}^k(a_i\mid s)\widehat{P}_{i,h,s,a_i}^{\pi_{-i}^k,\widehat{V}}\widehat{V}_{i,h+1}\right)^2\right]\notag\\
    &=2+2\left<e_s,\mathsf{Var}_{\underline{P}_{i,h}^{k,\widehat{V}}}\widehat{V}_{i,h+1}\right>.\label{eq:bonus_variance_bound}
\end{align}
where $e_s$ denotes an $S$-dimensional standard basis supported on the $s$-th element, the first inequality holds by the the elementary inequality $\mathsf{Var}_P(V+V^\prime) \leq2 \left(\mathsf{Var}_P(V)+ \mathsf{Var}_P(V^\prime) \right)$ for any transition kernel $P\in\mathbb{R}^S$ and vector $V,V^\prime\in\mathbb{R}^S$, and (i) is satisfied due to $\left|r_{i,h}^k(s,a_i)\right|\leq 1$ and $\left|\widehat{P}_{i,h}^{\pi_{-i}^k,\widehat{V}}\left(s^\prime\mid s,a_i\right)\right|\leq 1$ for all $s,s^\prime\in\mathcal{S},a_i\in\mathcal{A}_i$.
Inserting  \eqref{eq:bonus_variance_bound} back into \eqref{eq:bonus_definition} and rewriting in the vector form lead to
\begin{align*}
    \beta_{i,h}\leq 3c_{\mathsf{b}}\sqrt{\frac{\log^3(\frac{KS\sum_{i=1}^nA_i}{\delta})}{KH}}\left(H \cdot 1+ \sum_{k=1}^K \alpha_k^K \mathsf{Var}_{\underline{P}_{i,h}^{k,\widehat{V}}}\widehat{V}_{i,h+1}\right)
\end{align*}

\subsubsection{Proof for Lemma~\ref{eq:extra-lemma1}} 
\label{proof:eq:extra-lemma1}
Firstly, we introduce a set of robust value function variants:
\begin{align}
\forall (i,h) \in n\times  [H], \quad \overline{V}_{i,h}^{\mathsf{span}} := \mathbb{E}_{\pi \sim \widehat{\xi}}\left[ \overline{V}_{i,h}^{\pi,\ror_i}\right] - \min_{s' \in \mathcal{S}} \mathbb{E}_{\pi \sim \widehat{\xi}}\left[\overline{V}_{i,h}^{\pi,\ror_i}(s')\right],
\end{align}
which normalizes the robust value function $\overline{V}_{i,h}^{\pi,\ror_i}$. Applying Lemma~\ref{lemma:pnorm-key-value-range} leads to 
\begin{align}
\left\| \overline{V}_{i,h}^{\mathsf{span}} \right\|_\infty \leq \min \left\{ \frac{1}{\sigma_i}, H - h + 1 \right\}.
\end{align}

To continue, consider any transition kernel $P' \in \mathbb{R}^S$ and any $\widetilde{P} \in \mathbb{R}^S$ satisfying $\widetilde{P} \in \mathcal{U}^{\sigma_i}(P')$. For any $(i,h) \in [n] \times [H]$, the difference between the variances obeys
\begin{align}
\left| \mathsf{Var}_{P'}\left(\mathbb{E}_{\pi \sim \widehat{\xi}}\left[ \overline{V}_{i,h}^{\pi,\ror_i}\right]\right) - \mathsf{Var}_{\widetilde{P}}\left(\mathbb{E}_{\pi \sim \widehat{\xi}}\left[ \overline{V}_{i,h}^{\pi,\ror_i}\right]\right)\right| &= \big|\mathsf{Var}_{P'}\left(\overline{V}_{i,h}^{\mathsf{span}}\right) - \mathsf{Var}_{\widetilde{P}}\left(\overline{V}_{i,h}^{\mathsf{span}}\right)\big| \notag \\
&\leq \big\|\widetilde{P} - P'\big\|_1 \left\|\overline{V}_{i,h}^{\mathsf{span}}\right\|_\infty \notag \\
&\leq \sigma_i \left(\min \left\{ \frac{1}{\sigma_i}, H - h + 1 \right\} \right)^2 \leq \min \left\{ \frac{1}{\sigma_i}, H - h + 1 \right\}.
\end{align}

\subsubsection{Proof of Lemma~\ref{lem:key-lemma-reduce-H-3}}
\label{proof:lem:key-lemma-reduce-H-3}

To control the term of interest
$$\sum_{j=h}^H \left< b_h^j, \mathsf{Var}_{\sum_{k=1}^K\alpha_k^K\underline{P}_{i,j}^{k, \overline{V}}}\left(\mathbb{E}_{\pi\sim \widehat{\xi}}\left[\overline{V}_{i,j+1}^{\pi,\ror_i}\right]\right)  \right>,$$ this section follows the pipeline in Appendix~\ref{sec:key-lemma-reduce-H}, which we show here for completeness. We first introduce some auxiliary notations for convenience.
\begin{definition}\label{def:lemma-12}
    For any $(h,i)\in [H] \times [n]$, we denote 
    \begin{itemize}
        \item $\overline{V}_h^{\min} := \min_{s\in\mathcal{S}} \mathbb{E}_{\pi\sim \widehat{\xi}}\left[\overline{V}_{i,h}^{\pi,\ror_i} (s)\right]$ as the minimum value of all the entries.
        \item $\overline{V}_h':=   \mathbb{E}_{\pi\sim \widehat{\xi}}\left[\overline{V}_{i,h}^{\pi,\ror_i} \right]- \overline{V}_h^{\min} 1 $: the expected truncated value function following the distribution over all online policies.
        \item $\overline{r}_{i,h}^{\min} = \sum_{k=1}^K\alpha_k^K \mathbb{E}_{a_i\sim\pi_{i,h}^k}[r_{i,h}^k(\cdot,a_i)] + \left( \overline{V}_{h+1}^{\min} - \overline{V}_{h}^{\min}\right) 1 $: the expected truncated reward function.
\end{itemize}
\end{definition}

Applying the definition of $\mathbb{E}_{\pi\sim \widehat{\xi}} \left[\overline{V}_{i,h}^{\pi,\ror_i}\right]$ in \eqref{eq:auxiliary_value_function_3} and $\underline{P}_{i,h}^{k, \widehat{\xi},\overline{V}}$ (see \eqref{eq:extra-matrix-def} and \eqref{eq:extra-matrix-def2}) leads to
\begin{align}
    \overline{V}_h' = \mathbb{E}_{\pi\sim \widehat{\xi}} \left[\overline{V}_{i,h}^{\pi,\ror_i}\right] - \overline{V}_h^{\min} 1 &=\sum_{k=1}^K\alpha_k^K \mathbb{E}_{a_i\sim \pi_{i,h}^k}[r_{i,h}^k(\cdot,a_i)] + \sum_{k=1}^K\alpha_k^K \underline{P}_{i,h}^{k, \widehat{\xi},\overline{V}} \mathbb{E}_{\pi\sim \widehat{\xi}}\left[ \overline{V}_{i,h+1}^{\pi,\ror_i}\right] -  \overline{V}_h^{\min} 1 \nonumber \\
    &= \sum_{k=1}^K\alpha_k^K\mathbb{E}_{a_i\sim \pi_{i,h}^k}[r_{i,h}^k(\cdot,a_i)] + \left( \overline{V}_{h+1}^{\min} -\overline{V}_{h}^{\min}\right) 1 + \sum_{k=1}^K\alpha_k^K\underline{P}_{i,h}^{k, \widehat{\xi},\overline{V}}\overline{ V}_{h+1}'\\
    & = \overline{r}_{i,h}^{\min} + \sum_{k=1}^K\alpha_k^K\underline{P}_{i,h}^{k, \widehat{\xi},\overline{V}} \overline{V}_{h+1}'. \label{eq:bellman-minus-vmin-vstar2_empirical}
    \end{align}

With above fact in mind, one has
    \begin{align}
        &\mathsf{Var}_{\sum_{k=1}^K\alpha_k^K\underline{P}_{i,h}^{k, \widehat{\xi},\overline{V}}}\left(\mathbb{E}_{\pi\sim \widehat{\xi}}\left[\overline{V}_{i,h+1}^{\pi,\ror_i}\right]\right) \notag \\
         &\overset{\mathrm{(i)}}{=} \mathrm{Var}_{\sum_{k=1}^K\alpha_k^K\underline{P}_{i,h}^{k, \widehat{\xi},\overline{V}}}\left(\overline{V}_{h+1}'\right)\notag \\
        & =\sum_{k=1}^K\alpha_k^K\underline{P}_{i,h}^{k, \widehat{\xi},\overline{V}} \left(\overline{V}_{h+1}' \circ \overline{V}_{h+1}'\right) - \big(\sum_{k=1}^K\alpha_k^K\underline{P}_{i,h}^{k, \widehat{\xi},\overline{V}}\overline{ V}_{h+1}'\big) \circ  \big(\sum_{k=1}^K\alpha_k^K\underline{P}_{i,h}^{k, \widehat{\xi},\overline{V}} \overline{V}_{h+1}' \big) \notag \\
        & \overset{\mathrm{(ii)}}{=}  \sum_{k=1}^K\alpha_k^K\underline{P}_{i,h}^{k, \widehat{\xi},\overline{V}} \left(\overline{V}_{h+1}' \circ \overline{V}_{h+1}'\right)  - \Big(\overline{V}_h' - \overline{r}_{i,h}^{\min} \Big)^{\circ 2} \nonumber \\
        & = \sum_{k=1}^K\alpha_k^K\underline{P}_{i,h}^{k, \widehat{\xi},\overline{V}} \left(\overline{V}_{h+1}' \circ \overline{V}_{h+1}'\right) -  \overline{V}_h' \circ \overline{V}_h' + 2 \overline{V}_h' \circ \overline{r}_{i,h}^{\min} -  \overline{r}_{i,h}^{\min} \circ\overline{ r}_{i,h}^{\min} \nonumber \\
       & \leq \sum_{k=1}^K\alpha_k^K\underline{P}_{i,h}^{k, \widehat{\xi},\overline{V}} \left(\overline{V}_{h+1}' \circ \overline{V}_{h+1}'\right) -  \overline{V}_h' \circ \overline{V}_h' + 2 \|\overline{V}_h'\|_\infty 1,   \label{eq:variance-tight-bound-vstar2_empirical}
    \end{align}
    where (i) follows from the fact that $\mathrm{Var}_{\sum_{k=1}^K\alpha_k^K\underline{\widehat{P}}_{i,h}^{\pi^k, V}}(V - b 1) = \mathrm{Var}_{\sum_{k=1}^K\alpha_k^K\underline{\widehat{P}}_{i,h}^{\pi^k, V}}(V)$ for any value vector $V\in\mathbb{R}^S$ and scalar $b$, (ii) holds by \eqref{eq:bellman-minus-vmin-vstar2_empirical}, and the last inequality arises from $\overline{r}_{i,h}^{\min}  \leq\sum_{k=1}^K\alpha_k^K\mathbb{E}_{a_i\sim \pi^k}r_{i,h}^{k}(\cdot,a_i) \leq 1 $ since $ \overline{V}_{h+1}^{\min} - \overline{V}_{h}^{\min} \leq 0$ (see Definition~\ref{def:lemma-12}).

Consequently, invoking the definition of $b_h^j$ in \eqref{eq:defn-of-d}, combined with \eqref{eq:variance-tight-bound-vstar2_empirical}, we arrive at
\begin{align}
&\sum_{j=h}^H \Big< b_h^j, \mathsf{Var}_{\sum_{k=1}^K\alpha_k^K\underline{\widehat{P}}_{i,h}^{\pi^k, V}} \left(\mathbb{E}_{\pi\sim \widehat{\xi}}\left[\overline{V}_{i,j+1}^{\pi,\ror_i}\right] \right)  \Big> \notag \\
 &= \sum_{j=h}^H \left(b_h^j\right)^\top \left( \sum_{k=1}^K\alpha_k^K\underline{P}_{i,h}^{k, \widehat{\xi},\overline{V}}\left(\overline{V}_{j+1}' \circ \overline{V}_{j+1}'\right) -  \overline{V}_j' \circ \overline{V}_j' + 2 \|\overline{V}_h'\|_\infty 1 \right) \notag \\
& \overset{\mathrm{(i)}}{\leq} \sum_{j=h}^H \left[\left(b_h^j\right)^\top \left(\sum_{k=1}^K\alpha_k^K\underline{P}_{i,h}^{k, \widehat{\xi},\overline{V}} \left(\overline{V}_{j+1}' \circ \overline{V}_{j+1}'\right) -  \overline{V}_j' \circ \overline{V}_j'\right) \right]+ 2H\left\|\overline{V}_h'\right\|_\infty  \notag
\end{align}
which follows by the fact $\left\|\overline{V}_h'\right\|_\infty \geq \left\|\overline{V}_{h+1}'\right\|_\infty \geq \cdots \geq \left\|\overline{V}_{H}'\right\|_\infty $ and the property of $b_h^j$ in \eqref{eq:property-bh}. Finally, applying the definition of $b_h^j$ gives
\begin{align}
    &\sum_{j=h}^H \Big< b_h^j, \mathsf{Var}_{\sum_{k=1}^K\alpha_k^K\underline{\widehat{P}}_{i,h}^{\pi^k, V}} \left(\mathbb{E}_{\pi\sim \widehat{\xi}}\left[\overline{V}_{i,j+1}^{\pi,\ror_i}\right] \right)  \Big> \notag \\
    & = \sum_{j=h}^H \left[ \left(b_h^{j+1}\right)^\top  \left(\overline{V}_{j+1}' \circ \overline{V}_{j+1}'\right) -  (b_h^{j})^\top \left(\overline{V}_j' \circ \overline{V}_j' \right) \right]+ 2H\left\|\overline{V}_h'\right\|_\infty \notag \\
    & \leq \left\|b_h^{H+1}\right\|_1  \left\|\overline{V}_{H+1}' \circ \overline{V}_{H+1}'\right\|_\infty + 2H\left\|\overline{V}_h'\right\|_\infty  \notag \\
  &\leq 3H\left\|\overline{V}_h'\right\|_\infty,
\end{align}
where the last line also holds by the fact $\left\|\overline{V}_h'\right\|_\infty \geq \left\|\overline{V}_{h+1}'\right\|_\infty \geq \cdots \geq \left\|\overline{V}_{H}'\right\|_\infty$.

\subsubsection{Proof of Lemma~\ref{lemma:tv-dro-b-bound-star-marl}}
\label{proof:lemma:tv-dro-b-bound-star-marl}

To begin with, we start by analyzing the term of interest for any $(s,a_i)\in \cS\times \cA_i$,
\begin{align}
    \left|P_{i,h,s,a_i}^{\widehat{\pi}_{-i}^k,V}V -   P_{i,h,s,a_i}^{k, V}  V\right| 
    &= \left|\inf_{\mathcal{P} \in \mathcal{U}^{\sigma_i}\left(P_{h,s,a_i}^{\widehat{\pi}_{-i}^k}\right)} \mathcal{P} V - \inf_{\mathcal{P} \in \mathcal{U}^{\sigma_i}\left(P_{i,h,s,a_i}^{k}\right)} \mathcal{P} V\right| \notag \\
    &\overset{(\mathrm{i})}{=} \left|\max_{\alpha \in [\min_s V(s), \max_s V(s)]} \left[P_{h,s,a_i}^{\widehat{\pi}_{-i}^k}[V]_\alpha - \sigma_i \left(\alpha - \min_{s'} [V]_\alpha(s')\right)\right] \right. \notag \\
    & \quad  - \left.  \max_{\alpha \in [\min_s V(s), \max_s V(s)]} \left[P_{i,h,s,a_i}^{k}[V]_\alpha - \sigma_i \left(\alpha - \min_{s'} [V]_\alpha(s')\right)\right]\right| \notag \\
    &\leq \max_{\alpha \in [\min_s V(s), \max_s V(s)]} \left|P_{h,s,a_i}^{\widehat{\pi}_{-i}^k}[V]_\alpha - P_{i,h,s,a_i}^{k}[V]_\alpha\right|, \label{eq:middle-key-lemma}
\end{align}
where the last inequality uses the fact that the maximum operator is 1-Lipschitz.

For any fixed \( (\alpha, k,s,a_i) \in [0, H] \times [K] \times \cS \times \cA_i \), applying Bernstein's inequality yields: with probability at least \( 1 - \delta \),
\begin{align}
    \left|P_{h,s,a_i}^{\widehat{\pi}_{-i}^k}[V]_\alpha - P_{i,h,s,a_i}^{k}[V]_\alpha\right|
    &\leq \sqrt{\frac{2 \log \left(\frac{2}{\delta}\right)}{N}} \sqrt{\mathsf{Var}_{P_{h,s,a_i}^{\widehat{\pi}_{-i}^k}} \left([V]_\alpha\right)} + \frac{2H \log \left(\frac{2}{\delta}\right)}{3N}, \label{eq:V-p-phat-gap-one-alpha-bernstein}
\end{align}
since $P_{i,h,s,a_i}^{k}$ is the average of $N$ i.i.d sample following the distribution of $P_{h,s,a_i}^{\widehat{\pi}_{-i}^k}$.
To extend the pointwise bound to a union bound, we first construct an $\epsilon_1$-net $N_{\epsilon_1}$ over $[0, H]$ for $\alpha$, whose size satisfies $|N_{\epsilon_1}| \leq \frac{3H}{\epsilon_1}$ \citep{vershynin2018high}. Next, for any $(k,s,a_i) \in [K] \times \cS\times \cA_i$, we have
\begin{align}
    \left|P_{h,s,a_i}^{\widehat{\pi}_{-i}^k}[V]_\alpha - P_{i,h,s,a_i}^{k}[V]_\alpha\right|
    &\overset{\mathrm{(i)}}{\leq} \max_{\alpha \in N_{\varepsilon_1}} \left|P_{h,s,a_i}^{\widehat{\pi}_{-i}^k}[V]_\alpha - P_{i,h,s,a_i}^{k}[V]_\alpha\right| + \varepsilon_1 \notag \\
    &\overset{\mathrm{(ii)}}{\leq} \sqrt{\frac{2 \log \left(\frac{2S \sum_{i=1}^n A_i |N_{\varepsilon_1}| K n}{\delta}\right)}{N}} \sqrt{\mathsf{Var}_{P_{h,s,a_i}^{\widehat{\pi}_{-i}^k}} ([V]_\alpha)} + \frac{2H \log \left(\frac{2S \sum_{i=1}^n A_i |N_{\varepsilon_1}| K n}{\delta}\right)}{3N} + \varepsilon_1 \notag \\
    & \leq \sqrt{\frac{2 \log \left(\frac{2S \sum_{i=1}^n A_i |N_{\varepsilon_1}| K n}{\delta}\right)}{N}} \sqrt{\mathsf{Var}_{P_{h,s,a_i}^{\widehat{\pi}_{-i}^k}} ([V]_\alpha)} + \frac{2H \log \left(\frac{2S \sum_{i=1}^n A_i |N_{\varepsilon_1}| K n}{\delta}\right)}{3N} + \varepsilon_1 \notag \\
    &\overset{\mathrm{(iii)}}{\leq} \sqrt{\frac{2 \log \left(\frac{2S \sum_{i=1}^n A_i N K n}{\delta}\right)}{N}} \sqrt{\mathsf{Var}_{P_{h,s,a_i}^{\widehat{\pi}_{-i}^k}} (V)} + \frac{H \log \left(\frac{2S \sum_{i=1}^n A_i N K n}{\delta}\right)}{N}, \label{eq:donkey}
\end{align}
Here, (i) follows from the fact that the function $\left|P_{h,s,a_i}^{\widehat{\pi}_{-i}^k}[V]_\alpha - P_{i,h,s,a_i}^{k}[V]_\alpha\right|$ is 1-Lipschitz with respect to $\alpha$, and that any $\alpha \in [0,H]$ lies within an $\epsilon_1$-net centered around some point inside $N_{\epsilon_1}$; (ii) holds by the virtue of \eqref{eq:V-p-phat-gap-one-alpha-bernstein}, and (iii) holds when we set $ \varepsilon_1 = \frac{H \log \left(\frac{2S \sum_{i=1}^n A_i N K n}{\delta}\right)}{3N} $, based on the fact that $ |N_{\varepsilon_1}| \leq 9N $.

Inserting this back into \eqref{eq:middle-key-lemma} gives
\begin{align}
    \left|P_{i,h}^{\widehat{\pi}_{-i}^k,V}V -   P_{i,h}^{k, V}  V\right| 
    &\leq \sqrt{\frac{2 \log \left(\frac{2S \sum_{i=1}^n A_i N K n}{\delta}\right)}{N}} \sqrt{\mathsf{Var}_{P_{h}^{\widehat{\pi}_{-i}^k}} (V)} + \frac{H \log \left(\frac{2S \sum_{i=1}^n A_i N K n}{\delta}\right)}{N} 1\notag \\
    &\leq 3 \sqrt{\frac{H^2 \log \left(\frac{2S \sum_{i=1}^n A_i N K n}{\delta}\right)}{N}}1, \label{eq:lemma-13-scalar-result}
\end{align}
where the last inequality holds when $N\geq  \log \left(\frac{2S \sum_{i=1}^n A_i N K n}{\delta}\right)$, and base on the elementary fact that $\mathsf{Var}_{P_{h}^{\widehat{\pi}_{-i}^k}} (V)\leq H^2$.

Finally, we transfer the above results to vector form as below:
\begin{align}
    &\left|\underline{P}_{i,h}^{\widehat{\pi}^k,\overline{V}}\mathbb{E}_{\pi\sim\zeta}\left[\overline{V}_{i,h+1}^{\pi,\ror_i}\right]-\underline{P}_{i,h}^{k, \zeta, \overline{V}}\mathbb{E}_{\pi\sim\zeta}\left[\overline{V}_{i,h+1}^{\pi,\ror_i}\right]\right| \notag \\
   & =\left|\Pi_h^{\widehat{\pi}_i^k}\left(P_{i,h}^{\widehat{\pi}^k,\overline{V}}\mathbb{E}_{\pi\sim\zeta}\left[\overline{V}_{i,h+1}^{\pi,\ror_i}\right]-P_{i,h}^{k, \zeta, \overline{V}}\mathbb{E}_{\pi\sim\zeta}\left[\overline{V}_{i,h+1}^{\pi,\ror_i}\right]\right)\right| \notag  \\
   & \leq\Pi_h^{\widehat{\pi}_i^k}\left|P_{i,h}^{\widehat{\pi}^k,\overline{V}}\mathbb{E}_{\pi\sim\zeta}\left[\overline{V}_{i,h+1}^{\pi,\ror_i}\right]-P_{i,h}^{k, \zeta, \overline{V}}\mathbb{E}_{\pi\sim\zeta}\left[\overline{V}_{i,h+1}^{\pi,\ror_i}\right]\right| \notag \\
  &  \overset{(\mathrm{i})}{\leq} 2  \sqrt{\frac{\log\left(\frac{18S\sum_{i=1}^nA_iNHK}{\delta}\right)}{N}} \Pi_h^{\widehat{\pi}_i^k}\sqrt{\mathsf{Var}_{P_{h}^{\widehat{\pi}_{-i}^k}}\left(\mathbb{E}_{\pi\sim\zeta}\left[\overline{V}_{i,h+1}^{\pi,\ror_i}\right]\right)}  + \frac{\log\left(\frac{18S\sum_{i=1}^nA_iNHK}{\delta}\right)}{N} 1 \notag \\
  &  \overset{(\mathrm{ii})}{\leq} 2  \sqrt{\frac{\log\left(\frac{18S\sum_{i=1}^nA_iNHK}{\delta}\right)}{N}} \sqrt{\mathsf{Var}_{\underline{P}_{h}^{\widehat{\pi}^k}}\left(\mathbb{E}_{\pi\sim\zeta}\left[\overline{V}_{i,h+1}^{\pi,\ror_i}\right]\right)}  + \frac{\log\left(\frac{18S\sum_{i=1}^nA_iNHK}{\delta}\right)}{N} 1 \notag \\
    & \leq 3 \sqrt{\frac{H^2 \log \left(\frac{2S \sum_{i=1}^n A_i N K n}{\delta}\right)}{N}}1,
\end{align}
where (i) follows from applying \eqref{eq:lemma-13-scalar-result} with $V = \mathbb{E}_{\pi\sim \zeta}\left[\overline{V}_{i,h+1}^{\pi,\ror_i}\right]$, (ii) follows from Lemma~\ref{lm:weight_variance} and the definition of $\underline{P}_{h}^{\widehat{\pi}^k} = \prod_h^{\widehat{\pi}^k}P_h^0 = \Pi_h^{\widehat{\pi}_i^k} P_{h}^{\widehat{\pi}_{-i}^k}$ in Appendix~\ref{sec:matrix_notation}, and the last inequality holds by $\|V\|_\infty \leq H$ and letting $N \geq \log\left(\frac{18S\sum_{i=1}^nA_iNHK}{\delta}\right)$.

\subsubsection{Proof of Lemma~\ref{lm:reward_concentration}}
\label{sec:lm_reward_concentration}
We first introduce a modified version of the Freedman inequality for martingales \citep{freedman1975tail}, which is introduced in \citet[Theorem~5]{li2024q} and crucial for our analysis. 
\begin{theorem}[Theorem~5, \citet{li2024q}]\label{thm:Freedman}
    Suppose $Y_{n} = \sum_{k=1}^{n} X_{k} \in \mathbb{R}$, where $\{X_{k}\}$ is a real-valued scalar sequence such that 
    \[
    \left|X_{k}\right| \leq R \qquad \text{and} \qquad \mathbb{E}\left[X_{k} \mid \{X_{j}\}_{j < k}\right] = 0 \quad \text{for all } k \geq 1 
    \]
    for some constant $R > 0$. Define
    \[
    W_{n} := \sum_{k=1}^{n} \mathbb{E}_{k-1}[X_{k}^{2}],
    \]
    where $\mathbb{E}_{k-1}$ denotes the conditional expectation given $\{X_{j}\}_{j < k}$. For any $\kappa > 0$, with probability at least $1 - \delta$, the following holds:
    \begin{align}
    \left|Y_{n}\right| &\leq \sqrt{8W_{n} \log\frac{3n}{\delta}} + 5R \log\frac{3n}{\delta} \leq \kappa W_{n} + \left(\frac{2}{\kappa} + 5R\right) \log\frac{3n}{\delta}. \label{eq:Freedman-random}
    \end{align}
\end{theorem}

In order to apply Theorem~\ref{thm:Freedman}, we let
\begin{align*}
    R = \max_{k \in [K]} \left|\alpha_k^K \left<\widehat{\pi}_{i,h}^k(\cdot \mymid s), r_{i,h}^k(s, \cdot)\right>\right|\leq \Big\{\max_{k \in [K]} \alpha_k^K\Big\} \Big\{\max_{k \in [K]} \|\widehat{\pi}_{i,h}^k(s)\|_1 \|r_{i,h}^k\|_\infty\Big\} \leq \frac{2c_\alpha \log K}{K},
\end{align*}
which holds by $\alpha_k^K \leq \frac{2c_\alpha \log K}{K}$ in Lemma~\ref{lem:weight}, $\|\widehat{\pi}_{i,h}^k(s)\|_1 =1$ and $\|r_{i,h}^k\|_\infty \leq 1$. In addition, we define and observe that 
\begin{align*}
    W_K &= \sum_{k=1}^K (\alpha_k^K)^2 \mathsf{Var}_{h,k-1} \left(\left<\widehat{\pi}_{i,h}^k(s), r_{i,h}^k(s, \cdot)\right>\right) \\
    &\leq \left\{\max_{k \in [K]} \alpha_k^K\right\} \left\{\sum_{k=1}^K \alpha_k^K \mathsf{Var}_{h,k-1} \left(\left<\widehat{\pi}_{i,h}^k(s), r_{i,h}^k(s, \cdot)\right>\right)\right\} \\
    &= \left\{\max_{k \in [K]} \alpha_k^K\right\} \left\{\sum_{k=1}^K \alpha_k^K \mathsf{Var}_{h,k-1} \left(\mathbb{E}_{a\sim \widehat{\pi}_{i,h}^k(s)} \left[\left<e_a, r_{i,h}^k(s, \cdot) \right> \mymid \mymid r_{i,h}^k(s, \cdot) \right]\right)\right\} \\
    &\leq \frac{2c_\alpha \log K}{K} \sum_{k=1}^K \alpha_k^K \mathsf{Var}_{h,k-1} (r_{i,h}^k(s, a_{k,s})),
\end{align*}
where we use $\mathsf{Var}_{h,k-1}[\cdot]$ to denote the variance operator conditioned on the information available before the start of the $k$-th round of data collection for time step $h$. Here, the last inequality arises from defining a set of auxiliary random actions
\begin{align}
a_{k,s} \sim \widehat{\pi}_{i,h}^k(s),
\end{align}
which are independent from $r_{i,h}^k$ conditioned on the previous history before the start of the $k$-th round of data collection for time step $h$, and then applying Lemma~\ref{lem:weight} and the law of total variation.
With above facts in mind, applying Freedman's inequality (Lemma~\ref{thm:Freedman}) with $\kappa_1 = \sqrt{K \log(K/\delta)}$ gives
\begin{align*}
    &\left|\sum_{k=1}^K \alpha_k^K \mathbb{E}_{a_i \sim \widehat{\pi}_{i,h}^k} [r_{i,h}^k(s, a_i)] - \sum_{k=1}^K \alpha_k^K \mathbb{E}_{a_i \sim \widehat{\pi}_{i,h}^k} [r_{i,h}^{\widehat{\pi}_{-i}^k}(s, a_i)]\right| \\
   &\leq \kappa_1 W_K + \left(\frac{2}{\kappa_1} + 5R_1\right) \log\left(\frac{3K}{\delta}\right) \\
   &\leq 2c_\alpha \sqrt{\frac{\log^3\left(\frac{K}{\delta}\right)}{K}} \sum_{k=1}^K \alpha_k^K \mathsf{Var}_{h,k-1}\left(r_{i,h}^k(s,a_{k,s})\right)  + \left(2 \sqrt{\frac{1}{K \log\left(\frac{K}{\delta}\right)}} + \frac{10c_\alpha \log K}{K}\right) \log\left(\frac{3K}{\delta}\right) \\
   &\leq 2c_\alpha \sqrt{\frac{\log^3\left(\frac{K}{\delta}\right)}{K}} \sum_{k=1}^K \alpha_k^K \mathsf{Var}_{h,k-1}\left(r_{i,h}^k(s, a_{k,s})\right)  + 4 \sqrt{\frac{\log\left(\frac{3K}{\delta}\right)}{K}},
\end{align*}
with probability at least $1 - \delta$. Taking a union bound over $s \in \mathcal{S}$ gives
\begin{align*}
    \sum_{k=1}^K \alpha_k^K \left|r_{i,h}^{\widehat{\pi}^k} - \overline{r}_{i,h}^{\widehat{\pi}^k}\right| \leq c_r \sqrt{\frac{\log(KS/\delta)}{K}}1,
\end{align*}
where $c_r$ is some absolute constant. 
\subsubsection{Proof of Lemma~\ref{lem:key-lemma-reduce-H}} 
The proof follows the framework of \citep[Lemma 5]{shi2024sample}, with differences in the details. We provide the full proof here for completeness.

\label{sec:key-lemma-reduce-H}
One key aspect of the proof is truncating the terms of interest to consider their true ranges during controlling $\sum_{j=h}^H \left< d_h^j, \mathsf{Var}_{\sum_{k=1}^K\alpha_k^K\underline{P}_{i,j}^{\widehat{\pi}^k, \overline{V}}}\left(\mathbb{E}_{\pi\sim\zeta}\left[\overline{V}_{i,j+1}^{\pi,\ror_i}\right]\right)  \right>$. Towards this, we first introduce some auxiliary notations as below. 
\begin{definition}\label{defn:lemma-15}
    For any $(h,i)\in [H] \times [n]$, we denote
    \begin{itemize}
     \item $\overline{V}_h^{\min} := \min_{s\in\mathcal{S}} \mathbb{E}_{\pi\sim\zeta}\left[\overline{V}_{i,h}^{\pi,\ror_i} (s)\right]$: the minimum value of all the entries in vector $\mathbb{E}_{\pi\sim\zeta}\left[\overline{V}_{i,h}^{\pi,\ror_i}\right]$.
     \item $\overline{V}_h':=   \mathbb{E}_{\pi\sim\zeta}\left[\overline{V}_{i,h}^{\pi,\ror_i}\right] - \overline{V}_h^{\min} 1 $: the truncated expected value function associated with the distribution $\zeta$.
    \item $\overline{r}_{i,h}^{\min} = \sum_{k=1}^K\alpha_k^K\mathbb{E}_{a_i\sim\widehat{\pi}_{i,h}^k}[r_{i,h}^k(\cdot,a_i)] + \left( \overline{V}_{h+1}^{\min} - \overline{V}_{h}^{\min}\right) 1 $: the truncated expected reward function.
\end{itemize}
\end{definition}
With them  in mind, we introduce the following fact of $\overline{V}_h'$:
\begin{align}
    \overline{V}_h' & =  \mathbb{E}_{\pi\sim\zeta}\left[\overline{V}_{i,h}^{\pi,\ror_i}\right] - \overline{V}_h^{\min} 1\notag \\
    &= \sum_{k=1}^K\alpha_k^K\mathbb{E}_{a_i\sim\widehat{\pi}_{i,h}^k}[r_{i,h}^k(\cdot,a_i)]+\sum_{k=1}^K\alpha_k^K\underline{P}_{i,h}^{k, \zeta,\overline{V}} \mathbb{E}_{\pi\sim\zeta}\left[\overline{V}_{i,h+1}^{\pi,\ror_i}\right] -  \overline{V}_h^{\min} 1,
    \label{eq:bellman-minus-vmin-vstar_1}
\end{align}
which  holds by the definition of $\mathbb{E}_{\pi\sim\zeta}\left[\overline{V}_{i,h+1}^{\pi,\ror_i}\right]$ in \eqref{eq:auxiliary-V-hat-2} and transferring it to a vector form with the definition of $\underline{P}_{i,h}^{k, \zeta,\overline{V}}$ in \eqref{eq:matrix-notations-A-C}. To continue, one has
\begin{align}
    \overline{V}_h' & = \sum_{k=1}^K\alpha_k^K\mathbb{E}_{a_i\sim\widehat{\pi}_{i,h}^k}[r_{i,h}^k(\cdot,a_i)] + \sum_{k=1}^K\alpha_k^K \underline{P}_{i,h}^{\widehat{\pi}^k, \overline{V}} \mathbb{E}_{\pi\sim\zeta}\left[\overline{V}_{i,h+1}^{\pi,\ror_i}\right] \notag \\
&\quad  +  \Big( \sum_{k=1}^K\alpha_k^K\underline{P}_{i,h}^{k, \zeta,\overline{V}} - \sum_{k=1}^K\alpha_k^K\underline{P}_{i,h}^{\widehat{\pi}^k, \overline{V}}\Big) \mathbb{E}_{\pi\sim\zeta}\left[\overline{V}_{i,h+1}^{\pi,\ror_i}\right] - \overline{V}_h^{\min} 1\nonumber \\
&= \sum_{k=1}^K\alpha_k^K\mathbb{E}_{a_i\sim\widehat{\pi}_{i,h}^k}[r_{i,h}^k(\cdot,a_i)]  + \left( \overline{V}_{h+1}^{\min} - \overline{V}_{h}^{\min}\right) 1 \notag \\
&\quad  + \sum_{k=1}^K\alpha_k^K\underline{P}_{i,h}^{\widehat{\pi}^k, \overline{V}}  \overline{V}_{h+1}' + \Big( \sum_{k=1}^K\alpha_k^K\underline{P}_{i,h}^{k, \zeta,\overline{V}}  - \sum_{k=1}^K\alpha_k^K\underline{P}_{i,h}^{\widehat{\pi}^k, \overline{V}}\Big) \mathbb{E}_{\pi\sim\zeta}\left[\overline{V}_{i,h+1}^{\pi,\ror_i} \right] \nonumber \\
 &= \overline{r}_{i,h}^{\min} + \sum_{k=1}^K\alpha_k^K\underline{P}_{i,h}^{\widehat{\pi}^k, \overline{V}}  \overline{V}_{h+1}' + \Big( \sum_{k=1}^K\alpha_k^K\underline{P}_{i,h}^{k, \zeta,\overline{V}}  - \sum_{k=1}^K\alpha_k^K\underline{P}_{i,h}^{\widehat{\pi}^k, \overline{V}}\Big) \mathbb{E}_{\pi\sim\zeta}\left[\overline{V}_{i,h+1}^{\pi,\ror_i}\right],
 \label{eq:bellman-minus-vmin-vstar}
\end{align}
which follows from Definition~\ref{defn:lemma-15}. With above facts in mind, the main term in Lemma~\ref{lem:key-lemma-reduce-H} can be controlled as

\begin{align*}
    &\mathsf{Var}_{\sum_{k=1}^K\alpha_k^K\underline{P}_{i,h}^{\widehat{\pi}^k, \overline{V}}}\left(\mathbb{E}_{\pi\sim\zeta}\left[\overline{V}_{i,h+1}^{\pi,\ror_i}\right]\right) \notag\\ 
&    \overset{\mathrm{(i)}}{=} \mathsf{Var}_{\sum_{k=1}^K\alpha_k^K\underline{P}_{i,h}^{\widehat{\pi}^k, \overline{V}}}\left(\overline{V}_{h+1}'\right)\\
     & =  \sum_{k=1}^K\alpha_k^K\underline{P}_{i,h}^{\widehat{\pi}^k, \overline{V}} \left(\overline{V}_{h+1}' \circ \overline{V}_{h+1}'\right) - \big(\sum_{k=1}^K\alpha_k^K\underline{P}_{i,h}^{\widehat{\pi}^k, \overline{V}} \overline{V}_{h+1}'\big) \circ  \big(\sum_{k=1}^K\alpha_k^K\underline{P}_{i,h}^{\widehat{\pi}^k, \overline{V}} \overline{V}_{h+1}' \big), \nonumber 
\end{align*}
where (i) follows from the fact that $\mathrm{Var}_{\sum_{k=1}^K\alpha_k^K\underline{P}_{i,h}^{\widehat{\pi}^k, \overline{V}}}(V - b 1) = \mathrm{Var}_{\sum_{k=1}^K\alpha_k^K\underline{P}_{i,h}^{\widehat{\pi}^k, \overline{V}}}(V)$ for any value vector $V\in\mathbb{R}^S$ and scalar $b$. Then applying ~\eqref{eq:bellman-minus-vmin-vstar} leads to 
\begin{align}
     &\mathsf{Var}_{\sum_{k=1}^K\alpha_k^K\underline{P}_{i,h}^{\widehat{\pi}^k, \overline{V}}}\left(\mathbb{E}_{\pi\sim\zeta}\left[\overline{V}_{i,h+1}^{\pi,\ror_i}\right]\right) \notag\\ 
     & =\sum_{k=1}^K\alpha_k^K\underline{P}_{i,h}^{\widehat{\pi}^k, \overline{V}} \left( \overline{V}_{h+1}' \circ  \overline{V}_{h+1}'\right)  - \Big( \overline{V}_h' - \overline{r}_{i,h}^{\min} - \Big(\sum_{k=1}^K\alpha_k^K\underline{P}_{i,h}^{k, \zeta,\overline{V}}  - \sum_{k=1}^K\alpha_k^K\underline{P}_{i,h}^{\widehat{\pi}^k, \overline{V}} \Big) \mathbb{E}_{\pi\sim\zeta}\left[\overline{V}_{i,h+1}^{\pi,\ror_i}\right] \Big)^{\circ 2} \nonumber \\
    &= \sum_{k=1}^K\alpha_k^K\underline{P}_{i,h}^{\widehat{\pi}^k, \overline{V}} \left(\overline{V}_{h+1}' \circ \overline{V}_{h+1}'\right) -  \overline{V}_h' \circ \overline{V}_h' + 2 \overline{V}_h' \circ \Big(\overline{r}_{i,h}^{\min} + \Big(\sum_{k=1}^K\alpha_k^K\underline{P}_{i,h}^{k, \zeta,\overline{V}}  - \sum_{k=1}^K\alpha_k^K\underline{P}_{i,h}^{\widehat{\pi}^k, \overline{V}}\Big) \mathbb{E}_{\pi\sim\zeta}\left[\overline{V}_{i,h+1}^{\pi,\ror_i}\right]\Big) \nonumber \\
   & -   \Big(\overline{r}_{i,h}^{\min} + \Big( \sum_{k=1}^K\alpha_k^K\underline{P}_{i,h}^{k, \zeta,\overline{V}}  - \sum_{k=1}^K\alpha_k^K\underline{P}_{i,h}^{\widehat{\pi}^k, \overline{V}}\Big) \mathbb{E}_{\pi\sim\zeta}\left[\overline{V}_{i,h+1}^{\pi,\ror_i}\right] \Big)^{\circ 2}. \notag \\
     & \leq \sum_{k=1}^K\alpha_k^K\underline{P}_{i,h}^{\widehat{\pi}^k, \overline{V}}\left(\overline{V}_{h+1}' \circ \overline{V}_{h+1}'\right) -  \overline{V}_h' \circ \overline{V}_h' + 2 \big\|\overline{V}_h'\big\|_\infty \Big( 1 + \Big|\Big( \sum_{k=1}^K\alpha_k^K\underline{P}_{i,h}^{k, \zeta,\overline{V}}  - \sum_{k=1}^K\alpha_k^K\underline{P}_{i,h}^{\widehat{\pi}^k, \overline{V}}\Big) \mathbb{E}_{\pi\sim\zeta}\left[\overline{V}_{i,h+1}^{\pi,\ror_i}\right] \Big| \Big) \notag \\
     &\leq  \sum_{k=1}^K\alpha_k^K\underline{P}_{i,h}^{\widehat{\pi}^k, \overline{V}} \left(\overline{V}_{h+1}' \circ \overline{V}_{h+1}'\right) -  \overline{V}_h' \circ \overline{V}_h' + 2 \big\|\overline{V}_h' \big\|_\infty 1 + 6 \|V_h'\|_\infty   \sqrt{\frac{H^2\log \left(\frac{18S\sum_{i=1}^nA_i nHNK}{\delta} \right)}{N}}1, \label{eq:variance-tight-bound-vstar}
\end{align}
where the penultimate inequality follows from the fact that $\overline{r}_{i,h}^{\min}  \leq \sum_{k=1}^K\alpha_k^K\mathbb{E}_{a_i\sim\widehat{\pi}_{i,h}^k}[r_{i,h}^k(\cdot,a_i)] \leq 1 $ due to $ V_{h+1}^{\min} - V_{h}^{\min} \leq 0$ (see Definition~\ref{defn:lemma-15}), and the last inequality arises from Lemma~\ref{lemma:tv-dro-b-bound-star-marl}.

Consequently, along with the definition of $d_h^j$ in \eqref{eq:defn-of-d}, applying \eqref{eq:variance-tight-bound-vstar} yields
\begin{align}
&\sum_{j=h}^H \left< d_h^j, \mathsf{Var}_{\sum_{k=1}^K\alpha_k^K\underline{P}_{i,j}^{\widehat{\pi}^k, \overline{V}}}\left(\mathbb{E}_{\pi\sim\zeta}\left[\overline{V}_{i,j+1}^{\pi,\ror_i}\right]\right)  \right> \notag \\
& =\sum_{j=h}^H (d_h^j)^\top \Bigg( \sum_{k=1}^K\alpha_k^K\underline{P}_{i,j}^{\widehat{\pi}^k, \overline{V}}\left(\overline{V}_{j+1}' \circ \overline{V}_{j+1}'\right) -  \overline{V}_j' \circ \overline{V}_j' + 2 \|\overline{V}_j'\|_\infty 1 + 6 \|\overline{V}_j'\|_\infty   \sqrt{\frac{H^2\log \left(\frac{18S\sum_{i=1}^nA_i nHNK}{\delta} \right)}{N}}1\Bigg) \notag \\
& \overset{\mathrm{(i)}}{\leq} \sum_{j=h}^H \left[(d_h^j)^\top \left( \sum_{k=1}^K\alpha_k^K\underline{P}_{i,j}^{\widehat{\pi}^k, \overline{V}}\left(\overline{V}_{j+1}' \circ \overline{V}_{j+1}'\right) -  \overline{V}_j' \circ \overline{V}_j'\right) \right]+ 2H\|\overline{V}_h'\|_\infty  \notag \\
&\quad + 6H^2\|\overline{V}_h'\|_\infty\sqrt{\frac{\log \left(\frac{18S\sum_{i=1}^nA_i nHNK}{\delta} \right)}{N}} \notag 
\end{align}
where (i) holds by the basic fact $\|\overline{V}_h'\|_\infty \geq \|\overline{V}_{h+1}'\|_\infty \geq \cdots \geq \|\overline{V}_{H}'\|_\infty $, and the property of $d_j^h$ in \eqref{eq:property-dh}. Finally, applying definition of $d_h^j$ in \eqref{eq:defn-of-d} further leads to
\begin{align*}
    &\sum_{j=h}^H \left< d_h^j, \mathsf{Var}_{\sum_{k=1}^K\alpha_k^K\underline{P}_{i,j}^{\widehat{\pi}^k, \overline{V}}}\left(\mathbb{E}_{\pi\sim\zeta}\left[\overline{V}_{i,j+1}^{\pi,\ror_i}\right]\right)  \right> \notag \\
    &= \sum_{j=h}^H \left[ (d_h^{j+1})^\top  \left(\overline{V}_{j+1}' \circ \overline{V}_{j+1}'\right) -  (d_h^{j})^\top \left(\overline{V}_j' \circ \overline{V}_j' \right) \right]+2H\|\overline{V}_h'\|_\infty + 6H^2\|\overline{V}_h'\|_\infty\sqrt{\frac{\log \left(\frac{18S\sum_{i=1}^nA_i nHNK}{\delta} \right)}{N}} \notag \\
 &\leq \left\|d_h^{H+1} \right\|_1  \left\|\overline{V}_{H+1}' \circ \overline{V}_{H+1}'\right\|_\infty + 2H\|\overline{V}_h'\|_\infty + 6H^2\|\overline{V}_h'\|_\infty\sqrt{\frac{\log \left(\frac{18S\sum_{i=1}^nA_i nHNK}{\delta} \right)}{N}} \notag \\
&\leq 3H\|\overline{V}_h'\|_\infty + 6H^2\|\overline{V}_h'\|_\infty\sqrt{\frac{\log \left(\frac{18S\sum_{i=1}^nA_i nHNK}{\delta} \right)}{N}} \notag \\
  &= 3H\|\overline{V}_h'\|_\infty \left(1 + 2H\sqrt{\frac{\log \left(\frac{18S\sum_{i=1}^nA_i nHNK}{\delta} \right)}{N}} \right),
\end{align*}
where the last inequality holds by the basic fact $\|\overline{V}_h'\|_\infty \geq \|\overline{V}_{h+1}'\|_\infty \geq \cdots \geq \|\overline{V}_{H}'\|_\infty $.

\subsubsection{Proof of Lemma~\ref{lem:key-lemma-reduce-H-2}}
\label{proof:lem:key-lemma-reduce-H-2}

To control
$$\sum_{j=h}^H \left< w_h^j, \mathsf{Var}_{\sum_{k=1}^K\alpha_k^K\underline{P}_{i,j}^{\widehat{\pi}^k, V}}\left(\mathbb{E}_{\pi \sim \zeta}\left[V_{i,j+1}^{\pi,\ror_i}\right]\right)  \right>,$$ this section also follows the pipeline in Appendix~\ref{sec:key-lemma-reduce-H}, which we show here for completeness. We first introduce some auxiliary notations for convenience.
\begin{definition}\label{def:lemma-18}
    For any $(h,i)\in [H] \times [n]$, we denote, we denote 
    \begin{itemize}
        \item $V_h^{\min} := \min_{s\in\mathcal{S}} \mathbb{E}_{\pi \sim \zeta}\left[V_{i,h}^{\pi,\ror_i} (s)\right]$: the minimum value of all the entries in the vector $\mathbb{E}_{\pi \sim \zeta}\left[V_{i,h}^{\pi,\ror_i}\right]$.
        \item  $V_h':=   \mathbb{E}_{\pi \sim \zeta}\left[V_{i,h}^{\pi,\ror_i}\right]- V_h^{\min} 1 $: the truncated expected robust value function associated with $\zeta$.
        \item $r_{i,h}^{\min} = \sum_{k=1}^K\alpha_k^K \mathbb{E}_{a_i\sim\pi_{i,h}^k} [r_{i,h}^{\widehat{\pi}^k_{-i}}(\cdot,a_i)] + \left( V_{h+1}^{\min} - V_{h}^{\min}\right) 1 $: the truncated expected reward function associated with distribution $\zeta$.
    \end{itemize}
\end{definition}

Then applying the robust Bellman's consistency equation in \eqref{eq:robust-bellman-equation} with the definition of $\underline{P}_{i,h}^{\widehat{\pi}^k, V}$ in \eqref{eq:preliminary-V-defn} and its matrix version gives
\begin{align}
V_h' = \mathbb{E}_{\pi \sim \zeta}\left[V_{i,h}^{\pi,\sigma_i}\right] - V_h^{\min} 1 &=\sum_{k=1}^K\alpha_k^Kr_{i,h}^{\widehat{\pi}^k} +\sum_{k=1}^K\alpha_k^K \underline{P}_{i,h}^{\widehat{\pi}^k, V} \mathbb{E}_{\pi \sim \zeta}\left[V_{i,h+1}^{\pi,\sigma_i}\right] -  V_h^{\min} 1 \nonumber \\
&= \sum_{k=1}^K\alpha_k^Kr_{i,h}^{\widehat{\pi}^k} + \left( V_{h+1}^{\min} - V_{h}^{\min}\right) 1 + \sum_{k=1}^K\alpha_k^K\underline{P}_{i,h}^{\widehat{\pi}^k, V} V_{h+1}'\\
& = r_{i,h}^{\min} + \sum_{k=1}^K\alpha_k^K\underline{P}_{i,h}^{\widehat{\pi}^k, V} V_{h+1}'. \label{eq:bellman-minus-vmin-vstar2}
\end{align}
With the above fact in hand, the term of interest can be controlled by
\begin{align}
	&\mathsf{Var}_{\sum_{k=1}^K\alpha_k^K\underline{P}_{i,h}^{\widehat{\pi}^k, V}}\left(\mathbb{E}_{\pi \sim \zeta}\left[V_{i,h+1}^{\pi,\sigma_i}\right]\right)  \notag \\
    &\overset{\mathrm{(i)}}{=} \mathrm{Var}_{\sum_{k=1}^K\alpha_k^K\underline{P}_{i,h}^{\widehat{\pi}^k, V}}(V_{h+1}')\notag \\
    & =\sum_{k=1}^K\alpha_k^K\underline{P}_{i,h}^{\widehat{\pi}^k, V} \left(V_{h+1}' \circ V_{h+1}'\right) - \big(\sum_{k=1}^K\alpha_k^K\underline{P}_{i,h}^{\widehat{\pi}^k, V} V_{h+1}'\big) \circ  \big(\sum_{k=1}^K\alpha_k^K\underline{P}_{i,h}^{\widehat{\pi}^k, V} V_{h+1}' \big) \nonumber \\
	 & \overset{\mathrm{(ii)}}{=} \sum_{k=1}^K\alpha_k^K\underline{P}_{i,h}^{\widehat{\pi}^k, V} \left(V_{h+1}' \circ V_{h+1}'\right)  - \Big(V_h' - r_{i,h}^{\min} \Big)^{\circ 2} \nonumber \\
	 & = \sum_{k=1}^K\alpha_k^K\underline{P}_{i,h}^{\widehat{\pi}^k, V} \left(V_{h+1}' \circ V_{h+1}'\right) -  V_h' \circ V_h' + 2 V_h' \circ r_{i,h}^{\min} -  r_{i,h}^{\min} \circ r_{i,h}^{\min} \nonumber \\
	& \leq \sum_{k=1}^K\alpha_k^K\underline{P}_{i,h}^{\widehat{\pi}^k, V} \left(V_{h+1}' \circ V_{h+1}'\right) -  V_h' \circ V_h' + 2 \|V_h'\|_\infty 1,   \label{eq:variance-tight-bound-vstar2}
\end{align}
where (i) follows from the fact that $\mathrm{Var}_{\sum_{k=1}^K\alpha_k^K\underline{P}_{i,h}^{\widehat{\pi}^k, V}}(V - b 1) = \mathrm{Var}_{\sum_{k=1}^K\alpha_k^K\underline{P}_{i,h}^{\widehat{\pi}^k, V}}(V)$ for any value vector $V\in\mathbb{R}^S$ and scalar $b$, (ii) holds by \eqref{eq:bellman-minus-vmin-vstar2}, and the last inequality arises from $r_{i,h}^{\min}  \leq r_{i,h}^{\pi} \leq 1 $ since $ V_{h+1}^{\min} - V_{h}^{\min} \leq 0$ by Definition~\ref{def:lemma-18}.

Consequently, applying \eqref{eq:variance-tight-bound-vstar2} with the definition of $w_h^j$ in \eqref{eq:defn-of-d-2}, we arrive at
\begin{align}
&\sum_{j=h}^H \Big< w_h^j, \mathsf{Var}_{\sum_{k=1}^K\alpha_k^K\underline{P}_{i,h}^{\widehat{\pi}^k, V}} \left(\mathbb{E}_{\pi \sim \zeta}\left[V_{i,j+1}^{\pi, \sigma_i} \right]\right)  \Big> \notag \\
& \leq \sum_{j=h}^H (w_h^j)^\top \left( \sum_{k=1}^K\alpha_k^K\underline{P}_{i,j}^{\widehat{\pi}^k, V}\left(V_{j+1}' \circ V_{j+1}'\right) -  V_j' \circ V_j' + 2 \|V_h'\|_\infty 1 \right) \notag \\
& \overset{\mathrm{(i)}}{\leq}\sum_{j=h}^H \left[(w_h^j)^\top \left(\sum_{k=1}^K\alpha_k^K\underline{P}_{i,j}^{\widehat{\pi}^k, V} \left(V_{j+1}' \circ V_{j+1}'\right) -  V_j' \circ V_j'\right) \right]+ 2H\|V_h'\|_\infty  \notag \\
& = \sum_{j=h}^H \left[ (w_h^{j+1})^\top  \left(V_{j+1}' \circ V_{j+1}'\right) -  (w_h^{j})^\top \left(V_j' \circ V_j' \right) \right]+ 2H\|V_h'\|_\infty \notag \\
& \leq \|w_h^{H+1}\|_1  \left\|V_{H+1}' \circ V_{H+1}'\right\|_\infty + 2H\|V_h'\|_\infty  \notag \\
&\leq 3H\|V_h'\|_\infty,
\end{align}
where (i) and the last inequality hold by the fact $\|V_h'\|_\infty \geq \|V_{h+1}'\|_\infty \geq \cdots \geq \|V_{H}'\|_\infty $ and the property of $w_h^j$ in \eqref{eq:property-wh}, and the last equality follows from the definition of $w_h^j$ in \eqref{eq:defn-of-d-2}.